\documentclass[twoside,11pt]{article}

\usepackage{jmlr2e}
\usepackage{natbib}
\usepackage{amsmath}
\usepackage{ifthen}

\usepackage{lastpage}
\ShortHeadings{A Novel Theoretical Analysis for Clustering Heteroscedastic Gaussian Data}{Pastor, Dupraz, Hbilou and Ansel}
\firstpageno{1}

\usepackage{amsmath,amsfonts}
\usepackage{algorithmic}
\usepackage{algorithm}
\usepackage{array}
\usepackage[caption=false,font=small]{subfig}
\usepackage{textcomp}
\usepackage{stfloats}
\usepackage{url}
\usepackage{verbatim}

\usepackage{bbm}
\usepackage{bm}
\usepackage[normalem]{ulem}
\usepackage{stmaryrd}
\usepackage{mathrsfs}
\usepackage{cases}
\usepackage{empheq}
\usepackage{color}
\usepackage{subfig}
\usepackage{comment}
\usepackage{ifthen}
\usepackage[T1]{fontenc}
\usepackage{hyperref}
\usepackage{ragged2e}
\usepackage{subeqnarray}
\usepackage[dvipsnames]{xcolor}
\usepackage{xcolor}
\usepackage[capitalise, sort]{cleveref}

\usepackage{cancel}

\usepackage[dvipsnames]{xcolor}

\newcommand{\myscale}{.5}
\newcommand{\mylimsup}[1]{\underset{#1}{\overline{\lim}}}
\newcommand{\myiff}{\Leftrightarrow}

\newcommand{\myclubsuit}{(\textbf{A})}
\newcommand{\myspadesuit}{(\textbf{B})}
\newcommand{\Diagmat}{\Deltamat}
\newcommand{\estDiagmat}{{\Deltamat'}}
\newcommand{\estDiagmatn}{\estDiagmat_{\! \! n}}
\newcommand{\estDiagmatone}{\estDiagmat_{\! \! 1}}
\newcommand{\estDiagmatN}{\estDiagmat_{\! \! N}}

\newcommand{\Diagmatn}{\Diagmat_{n}}

\newcommand{\investDiagmatn}{\estDiagmatn^{\, -1}}

\newcommand{\new}[1]{#1}

\makeatletter
\newcommand{\manlabel}[2]{(#2)\def\@currentlabel{(#2)}\label{#1}}
\makeatother
\newcommand{\ignore}[1]{}
\newcommand{\unmarked}{\Ucal}
\newcommand{\initobs}{{\obs_{n^\star}}}
\newcommand{\initObs}{{\Obs_{n^\star}}}
\newcommand{\initseq}{{\seq_{n^\star}}}

\newcommand{\eigval}{\sqrteigval^{\, 2}}
\newcommand{\sqrteigval}{\sigma}
\newcommand{\sqrtesteigval}{{\sigma'}}
\newcommand{\sqrtesteigvalndimi}{\sqrtesteigval_{\! \! n, \dimi}}
\newcommand{\sqrtesteigvalndimib}{\sqrtesteigval_{\! \! n, \dimib}}
\newcommand{\sqrtesteigvalnmdimib}{\sqrtesteigval_{\! \! n_m, \dimib}}
\newcommand{\sqrtesteigvalnmdimi}{\sqrtesteigval_{\! \! n_m, \dimi}}
\newcommand{\esteigval}{\sqrtesteigval^{\, 2}}
\newcommand{\esteigvalndimi}{\sqrtesteigval_{\! \! n,\dimi}^{\, 2}}
\newcommand{\esteigvalnone}{\esteigval_{\! \! \! n,1}}
\newcommand{\esteigvalndim}{\esteigval_{\! \! \! n,\mydim}}
\newcommand{\eigvalnone}{\eigval_{\! n,1}}
\newcommand{\eigvalndim}{\eigval_{\! n,\mydim}}
\newcommand{\eigvalndimi}{\eigval_{\! n,\dimi}}
\newcommand{\eigvalmin}{\eigval_{\mymin}}
\newcommand{\eigvalmax}{\eigval_{\mymax}}
\newcommand{\sqrteigvalmin}{\sqrteigval_{\mymin}}
\newcommand{\sqrteigvalmax}{\sqrteigval_{\mymax}}

\newcommand{\esteigvalmin}{\esteigval_{\! \! \mymin}}
\newcommand{\esteigvalmax}{\esteigval_{\! \! \mymax}}
\newcommand{\sqrtesteigvalmin}{\sqrtesteigval_{\! \! \mymin}}
\newcommand{\sqrtesteigvalmax}{\sqrtesteigval_{\! \! \mymax}}

\newcommand{\sqrteigvalndimi}{\sqrteigval_{n,\dimi}}

\newcommand{\sqrteigvalnmdimi}{\sqrteigval_{n_m,\dimi}}
\newcommand{\sqrteigvalnmdimib}{\sqrteigval_{n_m,\dimib}}
\newcommand{\invsqrteigvalnmdimi}{\sqrteigval^{-1}_{n_m,\dimi}}

\newcommand{\investeigvalndimi}{\sqrtesteigvalndimi^{\, -2}}

\newcommand{\investeigvalmin}{\sqrtesteigvalmin^{\, -2}}

\newcommand{\invsqrtesteigvalndimi}{\sqrtesteigvalndimi^{\, -1}}

\newcommand{\invsqrtesteigvalmin}{\sqrtesteigvalmin^{\, -1}}
\newcommand{\invsqrtesteigvalmax}{\sqrtesteigvalmax^{\, -1}}
\newcommand{\invsqrteigvalmin}{\sqrteigvalmin^{\, -1}}
\newcommand{\invsqrteigvalmax}{\sqrteigvalmax^{\, -1}}

\newcommand{\investeigvalmax}{\sqrtesteigvalmax^{\, -2}}

\newcommand{\theexpect}{\Psi}

\newcommand{\constmin}{\const_{\mymin}}
\newcommand{\constmax}{\const_{\mymax}}
\newcommand{\const}{{\lambda}}
\newcommand{\theconst}{c}
\newcommand{\myf}{g}
\newcommand{\Cmat}{\Cbm}

\newcommand{\Umat}{\Ubm}

\newcommand{\Smat}{\Sbm}

\newcommand{\sigmaMLE}{{\est{\sigma}_{\text{MLE}}}}
\newcommand{\sigmaMLEsq}{{\est{\sigma}^2_{\text{MLE}}}}
\newcommand{\sigmaMEAN}{\sigma_{\text{mean}}}
\newcommand{\themax}{\mathrm{M}}

\newcommand{\lambdamin}{\lambda_{\mymin}}
\newcommand{\lambdamax}{\lambda_{\mymax}}
\newcommand{\ctrind}{k_0}

\newcommand{\fnGdiag}[2]{{\operatorname{GDiag}}{[#1,#2]}}

\newtheorem{Theorem}{Theorem}
\newtheorem{Lemma}{Lemma}
\newtheorem{Corollary}{Corollary}
\newtheorem{Remark}{Remark}

\newtheorem{Definition}{Definition}
\newtheorem{Proposition}{Proposition}

\newtheorem{Assumption}{Assumption}
\crefname{Assumption}{Assumption}{Assumptions}

\newtheorem{Predicate}{Predicate}
\crefname{Predicate}{Predicate}{Predicates}

\newenvironment{assumptioncomment}{\par\noindent{\bf Comment\ }}{\\[2mm]}

\newcommand{\matset}{\operatorname{M}_\mydim(\Rset)}
\newcommand{\diag}{\operatorname{diag}}

\newcommand{\eig}{\sqrteig^2}
\newcommand{\sqrteig}{\lambda}
\newcommand{\FreeMap}{\Thetabm}
\newcommand{\Maha}[1]{\nu_{#1}}
\newcommand{\Mahasq}[1]{\nu^2_{#1}}
\newcommand{\GL}{\operatorname{GL}_{\mydim}(\Rset)}

\newcommand{\transform}{\gamma}

\newcommand{\Models}{\operatorname{Models}}
\newcommand{\MVM}{{\Models(\Set,\family)}}
\newcommand{\MVMbis}{{\Models(\aSetbis,\family)}}
\newcommand{\MVMter}{{\Models(\bSetbis,\familyfn{\Diagmat})}}

\newcommand{\CtrSeqfn}[1]{\operatorname{CtrSeq}_{#1}}
\newcommand{\CtrSeq}{\CtrSeqfn{\seq,\MCov}}
\newcommand{\pval}[2]{\text{pval}_{#1}(#2)}
\newcommand{\Coeff}{\Lambdamat}
\newcommand{\coeff}{\lambda}

\newcommand{\algo}{CENTREx}

\def\radius{r}
\def\radiusb{\rho}
\newcommand{\theradius}{\rho}
\newcommand{\gooddist}[3]
{
	\ifthenelse{\equal{#1}{}}{\Delta^{\text{\tiny $#3$}, \text{\tiny $#2$}}_{\ctrind}}{\Delta^{\text{\tiny $#3$}, \, \text{\tiny $#2$}}_{\ctrind} ( #1 )}
}

\newcommand{\mydist}[1]
{
	\ifthenelse{\equal{#1}{}}{D^{\Ctr_0}_{\ctrind}}{D^{\Ctr_0}_{\ctrind}( #1 )}
}

\newcommand{\mydistfn}[2]
{
	\ifthenelse{\equal{#1}{}}{D^{#2}_{\ctrind}}{D^{#2}_{\ctrind}( #1 )}
}

\newcommand{\est}[1]{\widehat{#1}}

\newcommand{\ball}{\Bcal}
\newcommand{\Bcal}{\mathcal B}
\newcommand{\Ccal}{\mathcal U}
\newcommand{\Ucal}{\mathcal U}

\newcommand{\intmin}{N'}

\newcommand{\fn}[1]{\overline{#1}}
\newcommand{\TheThreshold}[1]{\mu(#1)}
\newcommand{\Topt}{\test_{\matcov}^{\alpha}}

\newcommand{\thenorm}{\nu_{\matcov}}

\newcommand{\TheFunc}{\Gbm_N^{\estMCov}}
\newcommand{\TheFuncb}{\Gbm_N^{\estDiagmat}}

\newcommand{\Thefuncbasic}{\gbm_{N}}
\newcommand{\TheFuncbasic}{\Gbm_{N}}

\newcommand{\Thefuncinit}{{\Thefuncbasic^{(\text{init})}}}

\newcommand{\MyWTer}{\MyWTerSigma}
\newcommand{\MyWTerfn}[1]{\Ubm_{N,\dimi}}
\newcommand{\MyWTerSigma}{\MyWTerfn{\est{\MCov}}}

\newcommand{\MyVTer}{\MyVTerSigma}
\newcommand{\MyVTerfn}[1]{V_{N,\dimi}}
\newcommand{\MyVTerSigma}{\MyVTerfn{\est{\MCov}}}

\newcommand{\MySBis}{S_{k,N,\dimi}}
\newcommand{\MyTBis}{T_{k,N,\dimi}}

\newcommand{\Zvec}{\Zbm}

\def\myvspace{\vspace{0.1cm}}

\newcommand{\dd}{d}

\newcommand{\dx}{\dd x}
\newcommand{\dy}{\dd y}

\newcommand{\familyfn}[1]{{\operatorname{Seq}}_{\, \seq,#1}}
\newcommand{\family}{\familyfn{\MCov}}

\newcommand{\sigmamax}{\sigma_{\mymax}}
\newcommand{\sigmamin}{\sigma_{\mymin}}

\newcommand{\dxvec}{\dd \xvec}
\newcommand{\dyvec}{\dd \yvec}

\def\Amat{\Abm}

\def\xivec{\xibm}
\def\Xvec{\Xbm}

\def\xvec{\xbm}
\def\yvec{\ybm}
\def\Wvec{\Wbm}

\def\Ibm{\bm I}

\newcommand{\Qmat}{\Qbm}
\newcommand{\Rmat}{\Rbm}

\newcommand{\Imat}{\Ibm}
\newcommand{\Deltamat}{\Deltabm}

\newcommand{\Lambdamat}{\Lambdabm}

\newcommand{\Abm}{\bm A}

\newcommand{\Cbm}{\bm C}

\newcommand{\Dbm}{\bm D}

\newcommand{\Qbm}{\bm Q}
\newcommand{\Rbm}{\bm R}

\newcommand{\varphibm}{\bm \varphi}

\newcommand{\Phibm}{\bm \Phi}

\newcommand{\Deltabm}{\bm \Delta}
\newcommand{\Lambdabm}{\bm \Lambda}

\newcommand{\thetabm}{\bm \theta}
\newcommand{\Thetabm}{\bm \Theta}

\newcommand{\Xibm}{\bm \Xi}

\newcommand{\xibm}{\bm \xi}
\newcommand{\Fbm}{\bm F}
\newcommand{\Gbm}{\bm G}
\newcommand{\fbm}{\bm f}
\newcommand{\gbm}{\bm g}
\newcommand{\Xbm}{\bm X}

\newcommand{\xbm}{\bm x}
\newcommand{\Ybm}{\bm Y}
\newcommand{\ybm}{\bm y}
\newcommand{\zbm}{\bm z}
\newcommand{\Sbm}{\bm S}
\newcommand{\Ubm}{\bm U}

\newcommand{\Wbm}{\bm W}

\newcommand{\Zbm}{\bm Z}

\def\Ebb{\mathbb{E}}
\def\Fbb{\mathbb{F}}

\def\Nbb{\mathbb{N}}
\def\Pbb{\mathbb{P}}

\def\Rbb{\mathbb{R}}
\def\Sbb{\mathbb{S}}
\def\Xbb{\mathbb{X}}
\def\Ybb{\mathbb{Y}}

\def\Acal{\mathcal{A}}
\def\Ccal{\mathcal{C}}

\def\Hcal{\mathcal{H}}

\def\Mcal{\mathcal{M}}
\def\Ncal{\mathcal{N}}

\def\Rcal{\mathcal R}
\def\Scal{\mathcal{S}}

\newcommand{\Bscr}{\mathscr B}

\newcommand{\Tfrak}{\mathfrak T}

\def\Nset{\Nbb}
\def\Rset{\Rbb}
\def\Sset{\Sbb}
\def\Xset{\Xbb}
\def\Yset{\Ybb}

\newcommand\fnrv[2]{\Mcal(#1,#2)}
\newcommand\rv{\fnrv{\Omega}{\Rset^{\mydim}}}
\newcommand\Lone{{L^1}(\Omega,\Rset^{\mydim})}
\newcommand{\mydim}{\mathrm{d}}
\newcommand{\dimi}{i}
\newcommand{\dimib}{j}
\def\dimb{\mydim}
\def\dimc{\mydim}
\newcommand{\Exp}[1]{\Ebb \big [ #1 \big ]}
\newcommand{\Expect}[1]{\Ebb \big [ #1 \big ]}

\newcommand{\transpose}{\mathrm{T}}
\newcommand{\transp}{\mathrm{\! T}}

\newcommand{\Process}[1]
{
	\ifthenelse{\equal{#1}{}}{\Fbm}{\Fbm(#1)}
}

\newcommand{\Processa}[1]
{
	\ifthenelse{\equal{#1}{}}{\Fbm}{\Fbm(#1)}
}

\newcommand{\Processb}[1]
{
	\ifthenelse{\equal{#1}{}}{\Fbm'}{\Fbm'(#1)}
}

\newcommand{\intM}{\llbracket 1, M \hspace{0.015cm} \rrbracket}
\newcommand{\intk}{\llbracket 1, K \hspace{0.015cm} \rrbracket}
\newcommand{\intn}{\llbracket 1, N \hspace{0.015cm} \rrbracket}
\newcommand{\intd}{\llbracket 1, \mydim \hspace{0.01cm} \rrbracket}

\newcommand{\intNfn}[1]{\llbracket #1 , \infty \llbracket}
\newcommand{\intNrestrict}{\llbracket 2, \infty \llbracket}

\newcommand{\intP}{\llbracket 1, P \hspace{0.01cm} \rrbracket}
\newcommand{\matcov}{\Smat}

\newcommand{\obs}{\ybm}
\newcommand{\Obs}{\Ybm}
\newcommand{\YObs}{\Ybm}
\newcommand{\Squel}[2]{{\Obs \! \left ({#1,#2} \right )}}

\newcommand{\SquelY}[2]{\YObs({#1,#2})}

\newcommand{\oobs}{\zbm}

\newcommand{\mynorm}[2][\cdot]{ \|  \, #1  \, \|_{#2}}
\newcommand{\norm}[1][\cdot]{ \| #1 \|}
\newcommand{\anothernorm}[1]{\nu \! \left ( #1 \right )}
\newcommand{\anothernormsh}{\nu}

\newcommand{\mybig}{\big}

\newcommand{\test}{\Tfrak}
\newcommand{\TestFamily}{\Tfrak^\bullet}

\def\level{\alpha}

\newcommand{\Proba}[1]{\Pbb \left [ \, #1 \, \right ]}

\newcommand{\Identity}[1]{\Imat_{#1}}

\newcommand{\Open}[2]{{] #1, #2 [}}

\newcommand{\RightOpen}[2]{{[ #1, #2 [}}
\def\intervalzeroone{\Open{0}{1}}

\def\tribu{\Acal}
\newcommand{\Ctr}{\varphibm}
\newcommand{\ctr}{\varphi}

\newcommand{\Marcum}{Q_{{\mydim}/2}}

\newcommand{\Id}{\Identity{{\mydim}}}

\newcommand{\CCtr}{\xibm}

\newcommand{\MCov}{\Cmat}
\newcommand{\MCovn}{\MCov_n}

\newcommand{\MCovnstar}{\MCov_{n^{\star}}}

\newcommand{\estMCov}{{{\Cmat'}}}

\newcommand{\estMCovn}{\estMCov_{\! \! n}}
\newcommand{\estMCovone}{\estMCov_{\! \! 1}}
\newcommand{\estMCovN}{\estMCov_{\! \! N}}

\newcommand{\estMCovnstar}{\estMCov_{\! \! n^{\star}}}
\newcommand{\MCovnm}{\MCov_{n_m}}

\newcommand{\investMCovn}{\estMCovn^{\, -1}}

\newcommand{\invsqrtestMCovn}{\estMCovn^{\, -1/2}}

\newcommand{\Ortho}{\Umat}

\newcommand{\CtrMap}{\Phibm}
\newcommand{\CtrMapBis}{\Xibm}
\newcommand{\CtrSet}{\mathscr C}

\newcommand{\TheSet}{\Ccal}
\newcommand{\limlimsupbasic}[3]{\displaystyle \lim_{#1} \left \vert \mylimsup{\, #2 \to \infty} {\Big \| #3 \Big \|} \, \right \vert_{\infty}}
\newcommand{\limlimsup}[3]{\displaystyle \lim_{#1} \left \vert \, \mylimsup{\, #2 \to \infty} \mynorm[ #3 ] \, \right \vert_{\infty}}
\newcommand{\limlimsupnu}[3]{\displaystyle \lim_{#1} \big \vert \mylimsup{\, #2 \to \infty} \nu (#3) \, \big \vert_{\infty}}

\newcommand{\esssuplim}[2]{\displaystyle \mathop{\text{ess. lim. sup}}_{#1 \, , \, #2 \to \infty}}
\newcommand{\esslimsup}[2]{\displaystyle \mathop{\text{ess. lim. sup}}_{#1 \, , \, #2 \to \infty}}
\newcommand{\esl}{\text{ess.lim.sup}}
\newcommand{\white}{\MCov_n^{-1/2}}

\newcommand{\mymin}{\text{min}}
\newcommand{\mymax}{\text{max}}

\newcommand{\cctr}{\xi}

\newcommand{\thetavec}{\thetabm}

\newcommand{\normpdf}{\, p_{\mathcal{N}}}
\newcommand{\Fchi}{\Fbb_{\chi^2_{\mydim}}}

\newcommand{\abs}[1]{\vert #1 \vert}
\newcommand{\integral}{\Upsilon(x_1)}
\newcommand{\sgn}{\text{sign}}
\newcommand{\meanvec}{\thetavec}

\newcommand{\meanvecc}{\xivec}
\newcommand{\meanveccc}{\xi}
\newcommand{\mean}{\theta}

\newcommand{\EstCtrSet}{\widehat{\hspace{1sp}\CtrSet}}% Le hspace permet de mieux centrer le chapeau sur la lettre

\newcommand{\risk}{\Rcal}

\newcommand{\statest}{\upsilon}

\newcommand{\elrisk}{\varrho_N}

\newcommand{\Dvec}{\Dbm}
\newcommand{\Diff}{\Dvec}
\newcommand{\truc}{\eta}
\newcommand{\machin}{\varepsilon}

\newcommand{\seq}{\kappa}
\newcommand{\Basefn}[2]{\Base_{#1}(#2)}
\newcommand{\Base}{\Bscr}

\newcommand{\GoodBase}[1]{\Base_{\Setbis{\Ctr_0}} \! \left ( {#1} \right )}
\newcommand{\Set}{{\Sset}}
\newcommand{\Setbis}[1]{\Set_{\ctrind}^{#1}}
\newcommand{\aSetbis}{\Setbis{\Ctr_0}}
\newcommand{\bSetbis}{\Setbis{\Rmat^\transpose \! \Ctr_0}}

\newcommand{\psit}{\widetilde{\Rmat}}
\newcommand{\invpsit}{\widetilde{\Rmat^\transpose}}

\newcommand{\psitfn}{\widetilde{\Rmat}_{(\ctrind,\Ctr_0)}}
\newcommand{\invpsitfn}{\widetilde{\Rmat^\transpose}_{(\ctrind,\Rmat^\transpose \Ctr_0)}}

\newcommand{\aSet}{\Xset}
\newcommand{\ObsSeq}{\YObsSeq}
\newcommand{\YObsSeq}{\Ybm}
\newcommand{\aObsSeq}{\aObsSeqFull}
\newcommand{\aObsSeqFull}{\Obs(\CtrMap)}
\newcommand{\YObsSeqFull}{\ObsSeq(\CtrMap)}

\newcommand{\variable}{x}

\newcommand{\thecoeff}{a}
\newcommand{\cost}{J_N}

\newcommand{\mylim}{\zeta}

\newcommand{\metric}{{H}}

\newcommand{\myeps}{\varepsilon}

\newcommand{\mywedge}{\text{\large\ensuremath {\bm \wedge}}}

\newcommand{\NsetRestricted}{{\llbracket K,\infty \llbracket}}
\newcommand{\cardkN}{\card \! \left ( \TheSet_{k,N} \right )}
\newcommand{\cardkoN}{\card \! \left ( \TheSet_{k_0,N} \right )}
\newcommand{\card}{\operatorname{card}}

\DeclareMathOperator*{\argmax}{argmax}

\renewcommand{\eqref}{\autoref}

\newcommand{\thespace}[1]{\,#1}
\newcommand{\mycomma}{\thespace{,}}
\newcommand{\myfstop}{\thespace{.}}

\newcommand{\spd}{symmetric positive-definite}
\newcommand{\spdm}{symmetric positive-definite matrix}
\newcommand{\spdms}{symmetric positive-definite matrices}

\newcommand{\assignfct}{\text{Assign}}

\DeclareMathOperator*{\argmin}{argmin}

\includecomment{twocol}

\includecomment{onecol}
\begin{document}

\title{A Novel Theoretical Analysis for Clustering Heteroscedastic Gaussian Data without Knowledge of the Number of Clusters}

\author{\name Dominique Pastor \email dominique.pastor@imt-atlantique.fr \\
	\addr Lab-STICC, \\
	IMT Atlantique, \\
	29238 Brest, France
	\AND
	\name Elsa Dupraz \email elsa.dupraz@imt-atlantique.fr \\
	\addr Lab-STICC, \\
	IMT Atlantique, \\
	29238 Brest, France
	\AND
	\name Ismail Hbilou \email ismail.hbilou@imt-atlantique.net \\
	\addr Lab-STICC, \\
	IMT Atlantique, \\
	29238 Brest, France
	\AND
	\name Guillaume Ansel \email guillaume.ansel@optisea.fr \\
	\addr OPTI'SEA \\
	23 rue du Bourgneuf \\
	29300 Quimperlé, France
}

 \editor{~}

\maketitle

\vspace{-0.1cm}
\begin{abstract}
	This paper addresses the problem of clustering measurement vectors that are heteroscedastic in that they can have different covariance matrices. From the assumption that the measurement vectors within a given cluster are Gaussian distributed with possibly different and unknown covariant matrices around the cluster centroid, we introduce a novel cost function to estimate the centroids. The zeros of the gradient of this cost function turn out to be the fixed-points of a certain function. As such, the approach generalizes the methodology employed to derive the existing Mean-Shift algorithm. But as a main and novel theoretical result compared to Mean-Shift, this paper shows that the sole fixed-points of the identified function tend to be the cluster centroids if both the number of measurements per cluster and the distances between centroids are large enough. As a second contribution, this paper introduces the Wald kernel for clustering. This kernel is defined as the p-value of the Wald hypothesis test for testing the mean of a Gaussian. As such, the Wald kernel measures the plausibility that a measurement vector belongs to a given cluster and it scales better with the dimension of the measurement vectors than the usual Gaussian kernel. Finally, the proposed theoretical framework allows us to derive a new clustering algorithm called \algo~that works by estimating the fixed-points of the identified function. As Mean-Shift, \algo~requires no prior knowledge of the number of clusters. It relies on a Wald hypothesis test to significantly reduce the number of fixed points to calculate compared to the Mean-Shift algorithm, thus resulting in a clear gain in complexity. Simulation results on synthetic and real data sets show that \algo~has comparable or better performance than standard clustering algorithms K-means and Mean-Shift, even when the covariance matrices are not perfectly known. 

\end{abstract}


\section{Introduction}\label{sec:intro}
Clustering is an unsupervised learning task that consists in dividing a data set into clusters such that data inside a cluster are similar with each other, and different from data belonging to other clusters \citep{jain10PRL, Aggarwal2013}. 

Clustering is commonly used in various Machine Learning tasks \citep{ComprehensiveSurveyClusteringAlgorithms, ComprehensiveReviewClusteringTechniquesKnowledgeDiscovery,DataClusteringApplicationAndTrends}, including document classification~\citep{steinbach2000comparison}, information retrieval, or data set segmentation~\citep{huang1998extensions}. It is also widely employed in multisensor signal processing applications, as highlighted by \citet{abbasi2007survey}.

Each clustering algorithm is based on underlying assumptions that determine the prior information required by the algorithm. For example, K-means~\citep{jain10PRL}, spectral clustering~\citep{von2007tutorial}, and WARD clustering~\citep{ward1963hierarchical} algorithms require the number $K$ of clusters~\citep{jain10PRL}. Given that it can be challenging to determine practically the appropriate value of $K$, numerous techniques have been proposed to jointly estimate $K$ and perform clustering: among others, ~\citet{pelleg2000x, Hess2020, Oluknami2019,dinh2019estimating,Hamerly2004,mohammadi2021k,ClusterNumberAssistedK-means,UnsupervisedKMeansClusteringAlgorithm}.
In contrast, algorithms like DB-SCAN~\citep{ester1996density}, OPTICS~\citep{ankerst1999optics}, and BIRCH~\citep{zhang1996birch}, do not need any prior knowledge of $K$. Instead, they depend on parameters such as the maximum distance between two points in a cluster or the minimum number of points per cluster. These parameters actually characterize the intra-cluster density, and they have a significant influence on the clustering performance.

Mean-Shift~\citep{comaniciu2002mean} is another clustering method requiring no knowledge of $K$. It was shown to achieve good clustering performance in various applications \citep{PAYER2019, Cariou2021, Chakraborty2021, HUO2022, YANG2022, You2022}. The Mean-Shift algorithm is derived by formulating clustering as a kernel density estimation (KDE) problem over a data set, with a fixed kernel. Mean-Shift estimates one after the other the modes of the probability density function (pdf) of the data set, which are then considered as the cluster centroids. To do so, a location M-estimator~\citep{comaniciu2002mean,sinova2018} is derived from a cost function $\cost$ \citep[Sec.~2.6]{comaniciu2002mean}, where $N$ is the number of data in the data set. The critical points (the points for which the derivative equals $0$) of $\cost$ are shown to be the fixed-points of a certain function $\Thefuncbasic$; these fixed-points are evaluated by a standard iterative calculation and considered as estimates of the sought centroids.

%\newnew{Centroid estimation via Mean-Shift thus boils down to minimizing $\cost$. The minimizers of $\cost$ turn out to be fixed-points of a function $\Thefunc$ obtained by differentiating $\cost$ \citep[Secs. 2.3 \& 2.6]{comaniciu2002mean}. These fixed-points are calculated by applying $\Thefunc$ iteratively \citep[Sec. 2.2]{comaniciu2002mean}. Centroid estimates calculated by Mean-Shift are thus approximates of fixed points of $\Thefunc$. The current literature does not guarantee that they minimize $\cost$ or that they approximate, in some sense, the true centroids.} %\newnew{The minimizers of this cost function turn out to be fixed-points of a function $\Thefunc$ obtained by differentiating $\cost$.}
%\newnew{Here, we should not say that the fixed-points minimize the cost function (there is no proof of this). We should say that the values than make the gradient equal to $0$ are the fixed-points of a certain function. And later on we can say that what is missing is to show that the fixed-points are actually the centroids.}

\medskip
The Mean-Shift algorithm has three main drawbacks. 
\begin{enumerate}
    \item A substantial effort has been made to investigate the convergence of the fixed-point iterative calculation process~\citep{ghassabeh2013convergence,ghassabeh2015sufficient,ghassabeh2018modified,yamasaki2019properties}. The convergence was established in the one-dimensional case only~by \citet{ghassabeh2013convergence}. For higher dimensions, a restrictive sufficient convergence condition was proposed for Gaussian kernels in~\citet{ghassabeh2015sufficient}, while necessary and sufficient convergence conditions were provided for more complex variants of the Mean-Shift iterative procedure by \citet{ghassabeh2018modified} and \citet{yamasaki2019properties}.
However, investigating this convergence is not sufficient in itself, given that the following two questions are also crucial for the theoretical analysis of Mean-Shift and its variants: (i) do the fixed-points of $\Thefuncbasic$, which we know to be critical points, minimize the cost function $\cost$? (ii) do the minimizers of $\cost$ approximate the centroids with respect to a certain criterion?
%However, investigating convergence alone is not sufficient since there is no evidence that the fixed-points of $\Thefuncbasic$ actually minimize the function $\cost$, nor that these minimizers approximate the centroids with respect to a certain criterion. 
In this paper, we directly prove that the fixed-points of an extended and randomized version of $\Thefuncbasic$ converge to the centroids in a certain probabilistic sense introduced in this paper, without having to show that these fixed-points minimize $\cost$.

%Therefore, the primary objective of this paper is to address these issues. To do so, we analyze the convergence, in a certain probabilistic sense, of an extended and randomized version of $\Thefuncbasic$.  
%\newnew{Rephrase to take previous remark into consideration. The key difficulty here is the the function is random! So studying its second-order derivative in general is complicated.}

\item  The Mean-Shift algorithm depends on a kernel function $w$ which weighs the contribution of each data point in the function $g_N$. Most works use a Gaussian kernel that depends on a parameter $h$ called bandwidth~\citep{wu02PR,comaniciu2002mean,ghassabeh2015sufficient}. The value of $h$ significantly affects the clustering performance and various methods have been proposed to estimate it~\citep{comaniciu2001variable,sheather1991reliable}. 
However, although these approaches yield satisfactory clustering results for small data dimensions ($\mydim<20$), the Gaussian kernel does not scale well as the data dimension $\mydim$ increases~\citep{huang2018convergence}. Thus, the secondary objective of this paper is to introduce an alternative kernel that can better accommodate higher dimensions. We propose to choose the kernel as the p-value of an hypothesis test to decide on whether a data point belongs to a certain cluster. 

\item The third issue is that the Mean-Shift algorithm estimates the fixed points of $\Thefuncbasic$ by merging the results of $N$ iterative procedures, each initialized with a different data point among the $N$ available. Alternatively, we propose to apply a hypothesis test after each centroid estimation to remove from the set of initializers all data points that likely belong to the current cluster. 
%The third issue is that the complexity of the Mean-Shift algorithm increases linearly with the number of data. \edm{Is this in $O(N^2)$ instead, compared to $O(NK)$ for our approach?}
%Indeed, this algorithm estimates the fixed points of $\Thefuncbasic$ by merging the results of $N$ iterative procedures, each initialized with a different data among the $N$ available. Alternatively, we propose to apply an hypothesis test after each centroid estimation so as to eliminate from the set of initializers the data that likely belongs to the current cluster. 
\end{enumerate}

\section{Contributions and paper outline}\label{sec:outline}

While Mean-Shift initially considered Gaussian vectors all with identical diagonal covariance matrix, we address the case where each data vector may have its own specific covariance matrix  (heteroscedasticity). By making such an hypothesis, we also have in mind the problem of clustering measurements  collected through a network of sensors where thermal noise spectral properties can differ from one sensor to another. In addition, we analyze to what extent exact knowledge of the covariance matrices can be mitigated. Of course, we also discuss the simpler and more usual case where all covariance matrices are identical and diagonal \citep{comaniciu2002mean, wu02PR, Dupraz_ICASSP2018}. Our main contributions can then be summarized as follows:
\begin{enumerate}
	\item We introduce a new cost function $\cost$ encompassing the case where each data vector has its own covariance matrix. This cost function generalizes those considered by \citet{comaniciu2002mean} and \citet{wu02PR}. As with Mean-Shift, the critical points of our new cost function $\cost$ are the fixed-points of a function $\Thefuncbasic$ %$\Thefuncbasic = \Thefunc$ 
	obtained by differentiating $\cost$. For the general case of heteroscedastic Gaussian measurement vectors, we then show the strong theoretical result that each centroid is the fixed-point of a randomized version $\TheFuncbasic$ of $\Thefuncbasic$, when the number $N$ of data and the distances between this centroid and all the others tend to infinity. This result stated in Theorem~\ref{Theorem: fixed points} follows from a novel theoretical analysis that consists of studying the convergence of $\TheFuncbasic$ in a specific probabilistic sense defined in the paper. The approach is non-Bayesian in that, unlike Gaussian Mixture Models (GMM)~\citep{reynolds2009gaussian}, the probabilities that the data belong to the different clusters are not required in the derivation. The parametric form of our cost function also induces that the result holds even when the data covariance matrices are unknown and approximated.
%    \vspace{-0.1cm}
	\item We introduce a novel kernel $w$ for clustering Gaussian clusters. This kernel is hereafter called the Wald kernel because it derives from the p-value of Wald's Uniformly Best Constant Power (UBCP) test \citep[Proposition 3]{Wald1943} for testing the mean of a Gaussian. By construction, this kernel scales better with the data dimension than the standard Gaussian kernel. 
%    \vspace{-0.1cm}
	\item From the asymptotic result of Theorem \ref{Theorem: fixed points}, the centroids of heteroscedastic Gaussian clusters can be estimated by the fixed-points of $\TheFuncbasic$. Hence, we introduce a novel clustering algorithm called \algo~that estimates the cluster centroids, similarly to Mean-Shift, by looking for the fixed-points of the function $\Thefuncbasic$, seen as a realization of $\TheFuncbasic$. \algo~relies on Wald's UBCP test \citep[Proposition 3]{Wald1943} again, to eliminate from the set of potential initializers all the measurement vectors that are close enough to an already estimated fixed-point. This elimination strategy greatly reduces the complexity of \algo~compared to Mean-Shift. Note that, similarly to Mean-Shift, \algo{} can be used with the Gaussian kernel or the Wald kernel. In the special case of clusters with identical covariance matrices proportional to identity, we derive a Maximum Likelihood Estimator (MLE) of the variance of the data. This MLE can be used by \algo~and the Mean-Shift algorithm to perform the clustering.
\end{enumerate}
%\vspace{-0.2cm}
We provide numerical simulation results to assess \algo. The results show that \algo~has the same clustering performance as Mean-Shift when the two algorithms use the same kernel. However, \algo~provides a significant reduction in complexity. %This finding validates the selection of \algo, which is less complex than Mean-Shift. 
Furthermore, we observe that the Wald kernel outperforms the Gaussian kernel, particularly when dealing with high-dimensional data ($\mydim=100$). We further evaluate the robustness of \algo~when using the Wald kernel. For identical covariance matrices proportional to identity, our results show that the proposed MLE of the variance of the data incurs only a minor performance degradation compared to the scenario where this variance is known. We also analyze the case of diagonal covariance matrices that are not scaled identity and identify effective strategies for selecting the variance parameters for the data vectors. All these experimental results pinpoint the relevance of Theorem \ref{Theorem: fixed points}, even in non-asymptotic situations, and provide evidence of the robustness and computational efficiency of the proposed approach.

The outline of the paper is as follows. Section~\ref{sec:model} gives our notation and formally describes the considered Gaussian model. Section~\ref{sec:robust} provides the cost function. Section~\ref{subsec:Fixed-point analysis} presents the theoretical analysis for centroid estimation. Section~\ref{sec:w_function} introduces the Wald kernel. Section~\ref{sec:algo} describes the proposed algorithm \algo. Section~\ref{sec:experiments} provides the numerical results.

\section{Notation and cluster model}\label{sec:model}
\indent
This section introduces pieces of notation used throughout the paper. It also provides the statistical model of clusters used to address the centroid estimation problem in Section \ref{sec:robust}. This model will be extended in Section \ref{sec: measurement vector models} to achieve our convergence analysis.

\label{sec: notation}

%\subsubsection*{Elementary sets, functions and sequences}
\noindent\textbf{Sets, functions and sequences.}
Throughout the paper, $\Nset$ is the set of positive integers. % all non-negative integers and $\Nset^* = \Nset \setminus \{0\}$. 
Given two positive integers $p$ and $q$ such that $p \leqslant q$, $\llbracket p,q \rrbracket$ denotes the set of integers $\{p, p+1, \ldots, q \}$ and similarly, $\llbracket q, \infty \llbracket$ (~resp. $\rrbracket q , \infty \llbracket \,$) is the set of all integers $n \geqslant q$ (resp. $n > q$). 
$\Rset$ is the set of all real values. Given two elements $a$ and $b$ of the extended real line such that $-\infty \leqslant a \leqslant b \leqslant \infty$, $\Open{a}{b}$ is the interval of all $x \in \Rset$ such that $a < x < b$. \\
\indent
The cardinality of a set $\Xset$ is denoted by $\textrm{card}(\Xset)$. Given any two sets $\Xset$ and $\Yset$, the notation $f: \Xset \to \Yset$ designates a function. Given such a function $f: \Xset \to \Yset$, the inverse image $\{x \in \Xset: f(x) \in B \}$ of $B \subset \Yset$ by $f$ will be denoted as $f^{-1}(B)$, and the direct image $\{f(x): x \in A \}$ of $A \subset \Xset$ by $f$ is denoted as $f(A)$. \\
\indent
A function defined on $\Nset$, or on $\llbracket 1,N \rrbracket$ for some $N \in \Nset$, is called a sequence. The limsup of a sequence of real values $(x_n)_{n \in \Nset}$ and the limsup of a sequence $(f_n)_{n \in \Nset}$ of functions valued in $\Rset$ are denoted by $\mylimsup{n \to \infty} \, x_n$ and $\mylimsup{n \to \infty} \, f_n$, respectively.

%\subsubsection*{Vectors and matrices}
\noindent
\textbf{Vectors and matrices.}
Throughout the paper, we consider some fixed $\dimb \in \Nset$ to designate the common dimension of all the random vectors encountered below. The space of all $\dimb$-dimensional real column vectors is denoted by $\Rset^{\dimb}$. Bold characters are used to denote vectors, matrices and, more generally, functions valued in $\Rset^{\dimb}$. 
\\
\indent
In this respect, given a set $\aSet$, a sequence of functions $\fbm_n: \aSet \to \Rset^{\mydim}$ for $n \in \Nset$ is denoted by a bold character $\fbm = (\fbm_n)_{n \in \Nset}$. We also deal with functions $\fbm$, still denoted by bold characters, that assign to each $x \in \aSet$ a sequence $\fbm(x)$ of functions valued in $\Rset^{\mydim}$. In this case, for each $x \in \aSet$, the elements of the sequence $\fbm(x)$ are denoted by $\fbm(x,n)$ for $n \in \Nset$. Therefore, we have $\fbm(x) = (\fbm(x,n))_{n \in \Nset}$, where each function $\fbm(x,n)$ is valued in $\Rset^\mydim$. \\
\indent
We denote by $\matset$ and $\GL$ the set of all $\mydim \times \mydim$ matrices with real coefficients and the General Linear Group of $\Rset^{\dimb}$, respectively. In the usual way, $\cdot^\transpose$ denotes the transpose operator for vectors and matrices. %The spectrum of any $\Amat \in \matset$ is denoted by $\Spec(\Amat)$. 
The identity matrix in $\Rset^{\dimb}$ is denoted by $\Identity{\dimb}$. We write $\diag(a_1, \ldots,a_{\mydim})$ for a $\mydim \times \mydim$ diagonal matrix whose diagonal elements are the real numbers $a_1, \ldots, a_{\mydim}$. If the diagonal elements of $\Coeff = \diag(\eig_1, \ldots,\eig_{\mydim})$ are non-null, we set $\Coeff^{-1/2} = \diag(\sqrteig_1^{-1}, \ldots, \sqrteig_\mydim^{-1})$.
\\
\indent
Given $\xvec \in \Rset^\mydim$ and $\dimi \in \intd$, $(\xvec)_\dimi$ designates the $\dimi^{\text{th}}$ coordinate or $\xvec$. Most often, we simply set $x_\dimi = (\xvec)_\dimi$ to alleviate the notation so that we generally write $\xvec = (x_1, \ldots, x_{\dimb})^\transpose$. We also write $\xvec = (x_\dimi)_{1 \leqslant \dimi \leqslant \mydim}$ to shorten the notation, which does not explicitly include the transpose operator. The reader must thus keep in mind that we only handle column vectors. According to these conventions, for any $\xvec \! \in \! \Rset^{\mydim}$, any $\Amat \in \matset$ and any $i \in \intd$, $(\Amat \xvec)_i$ designates the $\dimi^{\text{th}}$ coordinate of $\Amat \xvec$ so that $\Amat \xvec = \left ( (\Amat \xvec)_1, \ldots, (\Amat \xvec)_{\mydim} \right)^\transpose$. For vectors $\Amat \xvec$ with coordinates given by rather long formulas, it will be more convenient to write $\Amat \xvec = \left ( (\Amat \xvec)_{\dimi} \right )_{1 \leqslant \dimi \leqslant \mydim}$, again without recalling the transposition.

\vspace{0.1cm}
%\subsubsection*{Mahalanobis norm}
\noindent
\textbf{Mahalanobis norm.}
Throughout, we will extensively use Mahalanobis norms. We recall that the Mahalanobis norm $\Maha{\Smat}$, associated with a given $\mydim \times \mydim$ \spdm~$\Smat$, is defined by
\begin{equation}
	\label{eq: definition of the initial mahanorm}
	\forall \xvec \in \Rset^{\mydim}, \Maha{\Smat}(\xvec) = \sqrt{\xvec^T \Smat^{-1} \xvec} \myfstop
\end{equation}
If $\Smat$ is the identity matrix $\Identity{\mydim}$, then the Mahalanobis norm $\Maha{\Smat}$ is the Euclidean norm in $\Rset^{\mydim}$. As usual, this Euclidean norm will be denoted by $\norm$.
\\
\indent
Given a $\mydim \times \mydim$ \spdm~${\Smat}$, we recall the existence of an eigen decomposition 
\begin{equation}
	\label{eq: eigen decomposition of MCov init}
	\Smat = \Ortho \Coeff \Ortho^{\transpose} \mycomma
\end{equation}
where $\Coeff = \text{diag} (\eig_1, \ldots, \eig_{\mydim})$ is a diagonal matrix whose diagonal values are the eigenvalues $\eig_1, \ldots, \eig_{\mydim}$ of $\Smat$, and the columns of $\Ortho$ are the corresponding eigenvectors. 

\vspace{0.1cm}
\noindent
\textbf{Random vectors.}
Random vectors are denoted with capital letters. Throughout, all random vectors are assumed to be defined with respect to the same probability space denoted by $(\Omega, \tribu, \Pbb)$. By $\mydim$-dimensional real random vector, we mean any function $\Zbm: \Omega \to \Rset^{\mydim}$ measurable with respect to $\tribu$ and the Borel $\sigma$-algebra of $\Rset^{\mydim}$. A random variable is a random vector valued in $\Rset$. The set of all $\mydim$-dimensional real random vectors is denoted $\rv$, without recalling the underlying $\sigma$-algebras to simplify the notation. We simply write $\fnrv{\Omega}{\Rset}$ when $\mydim = 1$. According to our conventions above for vectors of $\Rset^\mydim$, given $\Zbm \in \rv$, we write $Z_i = (\Zbm)_i$ and thus $\Zbm = (Z_1, \ldots, Z_{\mydim})^\transpose$. In the same vein, given $\Zvec \in \rv$, $\Amat \in \matset$ and $i \in \intd$, we write $(\Amat \Zvec)_{\dimi}$ for the $\dimi^{\text{th}}$ coordinate of $\Amat \Zvec$ and thus $\Amat \Zvec = \left ( (\Amat \Zvec)_1, \ldots, (\Amat \Zvec)_{\mydim} \right)^\transpose$, or even $\Amat \Zvec = \left ( (\Amat \Zvec)_{\dimi} \right )_{1 \leqslant \dimi \leqslant \mydim}$ in case of coordinates with heavy expressions. If $\Zbm: \Omega \to \Rset^{\mydim}$ is a random vector such that $\Proba{\Zbm \in B} = 1$ for some Borel set $B$ of $\Rset$, we write $\Zbm \in B$ ($\Pbb$-a.s) for {\em almost surely with respect to $\Pbb$}. For any random variable $X \in \fnrv{\Omega}{\Rset}$, the essential supremum of $\vert X \vert$ will hereafter be denoted by $\vert X \vert_{\infty}$. As often in the literature, we write cdf for cumulative distribution function. As already mentioned in the introduction, we write pdf for probability density function. 

\subsection{Cluster model}
\label{subsec: signal model}
Throughout the paper, we consider a Gaussian model for the data~\citep{bouveyron14CSDA,wu02PR,forero11SP,Hess2020,comaniciu2002mean}. 

At this stage, we assume that we are given $N \in \Nset$ and a sequence $(\Obs_{n})_{n \in \intn}$ of $\mydim$-dimensional real Gaussian random vectors defined on $(\Omega, \tribu, \Pbb)$. In what follows, these random vectors are called measurement vectors and we suppose that the data to cluster are their realizations $\obs_1 = \Obs_1(\omega)$, \ldots, $\obs_N = \Obs_N(\omega)$ for some $\omega \in \Omega$. Given $n \in \intn$, $\MCov_n$ denotes the covariance matrix of $\Obs_n$. We then define a sequence of centroids as follows.
\begin{Definition}
	\label{def: sequence of centroids}
    Given $K\in\Nset$, a sequence of centroids is any injective map $\CtrMap: \intk \to \Rset^{\mydim}$, the centroids being the $K$ vectors $\CtrMap(1), \ldots, \CtrMap(K)$. The set of all sequences of centroids is denoted by $\Set$. 
\end{Definition}

Next, we assume the existence of a surjective finite sequence $\seq = (\seq_n)_{n \in \intn}$ valued in $\intk$ such that
\begin{equation}
	\label{eq: data model (first version)}
	\forall n \in \intn, \Obs_n \thicksim \Ncal(\CtrMap(\seq_n), \MCov_{n}) \myfstop
%	\nonumber
\end{equation}
The surjectivity of the sequence $\seq$ is required to guarantee the existence of at least one random vector $\Obs_n$ whose expectation is $\CtrMap(k)$. With this definition, the $k^{\text{th}}$ cluster is the set of all measurement vectors $\Obs_n$ whose expectation equals $\CtrMap(k)$, and no cluster is empty. %\modif{Equivalently}, the $k^{\text{th}}$ cluster is unequivocally specified by the set $\{n \in \intn: \seq_n = k\}$. 

By clustering, we mean that we aim to estimate the centroids issued from $\CtrMap$. We assume no prior knowledge of the number $K$ of clusters.

\section{Centroid estimation}\label{sec:robust}
In this section, we first introduce a new cost function to estimate $\CtrMap$ from the measurement vectors $(\Obs_n)_{n \in \llbracket 1,N\rrbracket}$. We show that the centroids $\CtrMap(k)$ can be estimated by calculating the fixed-points of a certain function $\Thefuncbasic$ obtained after differentiating this cost function. We will see that the {\em mean-shift} \citep{comaniciu2002mean} is a particular case of $\Thefuncbasic$.

\subsection{Cost function for centroid estimation}
\label{subsec:Cost Function for centroid estimation}
Given an injective map $\FreeMap: \intk \to \Rset^{\mydim}$ and a realization $(\obs_1, \ldots,\obs_N)$ of $(\Obs_1, \ldots, \Obs_N)$, we introduce the following cost function $J$:
\begin{equation}
	\label{eq:cost_function}
	J(\FreeMap) = \displaystyle \sum_{k=1}^K \sum_{n=1}^N \risk(\Maha{\estMCovn}^2( \obs_n - \FreeMap(k))) \myfstop
\end{equation}
In this equation, $\risk: \mathbb{R} \rightarrow \mathbb{R}$ is an increasing and differentiable function that verifies $\risk(x)=0 \Rightarrow x=0$, and $(\estMCovn)_{n \in \intn}$ is a sequence of \spdms. Although the cost function $J$ involves the number $K$ of clusters, we will see in Section \ref{sec: Critical points of the cost function} that the derivation of the critical points of $J$ does not depend on $K$.

\begin{Remark}
	\label{rmk: fundamental remark}
	If the covariance matrices $\MCov_n$ are known, a natural choice consists of setting $\estMCovn = \MCov_n$ for each $n \in \llbracket 1,N\rrbracket$. Otherwise, our cost function allows us to use estimates $\estMCovn$ for these covariance matrices.
\end{Remark}
%\modif{At this stage, for any $n \in \intn$, $\estMCovn$ can be any covariance matrix. We discuss below the practical choices of interest for $\estMCovn$.}

Comparatively, two cost functions are introduced for Mean-Shift \citep{comaniciu2002mean}. The first one is given by a \emph{multivariate kernel density estimator} to be maximized with respect to the vectors $\FreeMap(k)$, following a KDE approach~\citep{wand1994kernel}. The second one is expressed as a \emph{location estimator} to be minimized with respect to the vectors $\FreeMap(k)$, following an M-estimation approach~\citep{zoubir12robust}. Interestingly, it is shown in~\citet{comaniciu2002mean} that both cost functions lead to the same theoretical framework and clustering algorithm. The cost function of~\eqref{eq:cost_function} generalizes those introduced in~\citet{wand1994kernel}, \citet{comaniciu2002mean}, \citet{wu02PR}, to the case of possibly different covariance matrices $\MCov_n$ for each measurement vector $\Obs_n$. In the next section, we proceed as in previous works by differentiating the cost function $J$ above.

\subsection{Critical points of the cost function}
\label{sec: Critical points of the cost function}
In order to estimate the centroids, we would like to minimize the cost function of~\eqref{eq:cost_function} among all injective maps $\FreeMap: \intk \to \Rset^{\mydim}$. For this, we follow the same steps as \citet{comaniciu2002mean} and first calculate the gradient of $J$. Since the risk function $\risk$ is differentiable, we have, for any $k \in \intk$,
% \begin{onecol}
\begin{equation}
\label{eq: differentiation-1}
	\partial_{\FreeMap(k)} J(\FreeMap) = \displaystyle \sum_{n=1}^N \partial_{\FreeMap(k)} \Big ( \Maha{\estMCovn}^2(\obs_n - \FreeMap(k)) \Big ) \times \risk' \left ( \Maha{\estMCovn}^2(\obs_n - \FreeMap(k)) \right ) \mycomma
\nonumber
\end{equation}
% \end{onecol}
% \begin{twocol}
% \begin{equation}
% \label{eq: differentiation-1}
% \begin{array}{lll}
% 	\hspace{-0.25cm} \partial_{\FreeMap(k)} J(\FreeMap) \vspace{0.1cm} \\
% 	= \displaystyle \sum_{n=1}^N \partial_{\FreeMap(k)} \Big ( \Maha{\estMCovn}^2(\obs_n - \FreeMap(k)) \Big ) \times \risk' \left ( \Maha{\estMCovn}^2(\obs_n - \FreeMap(k)) \right ),
% \end{array}
% \nonumber
% \end{equation}
% \end{twocol}
where $\partial_{\FreeMap(k)} J(\FreeMap)$ is the partial derivative of $J(\FreeMap)$ with respect to $\FreeMap(k)$. Since $$\Maha{\estMCovn}^2(\obs_n - \FreeMap(k)) = (\obs_n - \FreeMap(k) )^\transpose  
\investMCovn (\obs_n - \FreeMap(k)) \mycomma$$ it follows from standard matrix differentiation~\citep[Sec.~2.4]{Matrixcookbook} that, for any $n \in \intn$,
\begin{equation}
	\label{eq: differentiation-2}
	\partial_{\FreeMap(k)} \Big ( \Maha{\estMCovn}^2(\obs_n - \FreeMap(k)) \Big ) =  - 2 \investMCovn (\obs_n - \FreeMap(k)) \mycomma
	\nonumber
\end{equation}
so that
% \begin{onecol}
\begin{equation}
	\label{eq: differentiation-3}
		\partial_{\FreeMap(k)} J(\FreeMap) = - 2 \displaystyle \sum_{n=1}^N w \left ( \Maha{\estMCovn}^2(\obs_n - \FreeMap(k)) \right ) \investMCovn (\obs_n - \FreeMap(k)) \mycomma
	%	\nonumber
\end{equation}
% \end{onecol}
% \begin{twocol}
% \begin{equation}
% 	\label{eq: differentiation-3}
% 	\begin{array}{lll}
% 		\hspace{-0.5cm} \partial_{\FreeMap(k)} J(\FreeMap) = 
% 		\vspace{0.1cm} \\
% 		- 2 \displaystyle \sum_{n=1}^N \investMCovn (\obs_n - \FreeMap(k)) w \! \left ( \Maha{\estMCovn}^2(\obs_n - \FreeMap(k)) \right ),
% 	\end{array}
% 	%	\nonumber
% \end{equation}
% \end{twocol}
where the derivative
\begin{equation}
	\label{eq:definition of w}
	w = \risk'
\end{equation}
of $\risk$ is hereafter called the kernel. We henceforth assume that $w$ is valued in $\Open{0}{\infty}$. As a consequence, because each $\investMCovn$ is itself symmetric and positive-definite,  $$\sum_{n=1}^N w \big ( \Maha{\estMCovn}^2 (\obs_n-\FreeMap(k)) \big ) \investMCovn \in \GL \myfstop$$
It thus follows from \eqref{eq: differentiation-3} that, for all $k \in \intk$, $\partial_{\FreeMap(k)} J(\FreeMap) = 0$ if and only if
% \begin{onecol}
\begin{equation}
	\label{eq: equation to solve to find the minimizers-0}
		\left ( \sum_{n=1}^N w \big ( \Maha{\estMCovn}^2 (\obs_n-\FreeMap(k)) \big ) \investMCovn \right )^{-1} \times \, \, \, \sum_{n=1}^N w \big ( \Maha{\estMCovn}^2 (\obs_n-\FreeMap(k)) \big ) \investMCovn \obs_n = \FreeMap(k) \myfstop
\end{equation}
% \end{onecol}
% \begin{twocol}
% \begin{equation}
% 	\label{eq: equation to solve to find the minimizers-0}
% 	\begin{array}{lll}
% 		\hspace{-0.2cm} 
% 		\left ( \sum_{n=1}^N w \big ( \Maha{\estMCovn}^2 (\obs_n-\FreeMap(k)) \big ) \investMCovn \right )^{-1} 
% 		\vspace{0.1cm} \\ 
% 		\hspace{0.5cm} \times \sum_{n=1}^N w \big ( \Maha{\estMCovn}^2 (\obs_n-\FreeMap(k)) \big ) \investMCovn \obs_n = \FreeMap(k)
% 	\end{array}
% \end{equation}
% \end{twocol}
Hence, finding an injective map $\FreeMap: \intk \to \Rset^{\mydim}$ satisfying \eqref{eq: equation to solve to find the minimizers-0} amounts to looking for the fixed-points of the function $\Thefuncbasic$ defined for every $\Ctr \in \Rset^{\mydim}$ by setting
% \begin{onecol}
\begin{equation}
	\label{eq: our general function}
	\Thefuncbasic(\Ctr)
		= \left ( \sum_{n=1}^N w \big ( \Maha{\estMCovn}^2 (\obs_n-\Ctr) \big ) \investMCovn \right )^{-1}  \times \, \, \, \sum_{n=1}^N w \big ( \Maha{\estMCovn}^2 (\obs_n-\Ctr) \big ) \investMCovn \obs_n \myfstop
\end{equation}
% \end{onecol}
% \begin{twocol}
% \begin{equation}
% 	\label{eq: our general function}
% 	\begin{array}{lll}
% 		\hspace{-0.5cm} 
% 		\modif{\Thefuncbasic(\Ctr)}
% 		= \left ( \sum_{n=1}^N w \big ( \Maha{\estMCovn}^2 (\obs_n-\Ctr) \big ) \investMCovn \right )^{-1} 
% 		\vspace{0.1cm} \\ 
% 		\hspace{1.5cm} \times \sum_{n=1}^N w \big ( \Maha{\estMCovn}^2 (\obs_n-\Ctr) \big ) \investMCovn \obs_n.
% 	\end{array}
% \end{equation}
% \end{twocol}
The function $\Thefuncbasic$ will hereafter be called the {\em mean-shift function} for the following reason. Consider the particular case where there exists a positive real number $h$ such that, for each $n \in \Nset$, $\MCov_n = \estMCovn = h^2 \Id$. In this case, for every $\Ctr \in \Rset^{\mydim}$, we have:
\begin{equation}
	\label{eq: MyNewfunc}
	\hspace{-0.3cm} 
	\Thefuncbasic(\Ctr)
%	\Thefunc(\Ctr, \obs_1, \ldots, \obs_N)
%	\! = \! 
	= \dfrac{\sum_{n = 1}^{N} w \left ( {\norm[ \, \obs_n - \Ctr \, ]^2}/{h^2} \right ) \obs_n }{\sum_{n = 1}^{N} w \left ( {\norm[ \, \obs_n - \Ctr \, ]^2}/{h^2} \right )} \myfstop
\end{equation}
The {\em mean-shift} identified by~\citet[Eq.~17]{comaniciu2002mean} after differentiating the cost function is thus a particular case of $\Thefuncbasic$.
%$\Thefunc(\Ctr, \obs_1, \ldots, \obs_N)-\Ctr$. 
However, no theoretical evidence was given by~\citet{comaniciu2002mean} to justify that the fixed points of this function are the cluster centroids $\CtrMap(k)$. 

We could envisage using methods from conventional optimization theory to study whether these fixed points are minimizers of the cost function $J$, for instance by verifying conditions on the Hessian of $J$. However, such methods are difficult to apply here, due to the probabilistic model for the vectors $\Obs_n$. We thus develop below an alternative approach that involves studying the convergence in a specific probabilistic sense of a randomized version $\TheFuncbasic$ of $\Thefuncbasic$.
As a novel and strong theoretical result, we show that, at least asymptotically, the centroids $\CtrMap(k)$ are the sole fixed points of this random function. 

\section{Convergence analysis}
\label{subsec:Fixed-point analysis}
This section presents our main theoretical result stated in Theorem \ref{Theorem: fixed points}. This result is asymptotic and involves countably infinite sequences of measurement vectors. To ease the reading of this theorem, we first introduce and comment most of the material required to state it.

\vspace{-0.25cm}
\subsection{Preliminary material}
Here, we provide all the material required to define in Section \ref{sec: Essential convergence in limsup} the convergence criterion used in Theorem \ref{Theorem: fixed points}. We begin with the modeling of a sequence of measurement vectors as a function of a sequence of centroids. We also define our randomized mean-shift function, of which $\Thefuncbasic(\Ctr)$ of \eqref{eq: our general function} is a realization. We then introduce the notion of distance between a sequence of centroids and a fixed vector. Finally, we recall some definitions and properties of bases and limits to define our convergence criterion in Section \ref{sec: Essential convergence in limsup}.

\subsubsection{Centroids, clusters and measurement vector models}
%\subsubsection{Centroids and clusters}
%\label{sec: centroids and clusters}
\label{sec: measurement vector models}
Since our analysis requires making $N$ tend to infinity, we need a countably infinite {--- in short, denumerable ---} sequence of measurement vectors. We denote it by $(\Obs_n)_{n \in \Nset}$. 

%We consider a sequence of centroids in the form of an injective map $\CtrMap$, as introduced in Definition~\ref{def: sequence of centroids}. To extend the {initial} model of~\eqref{eq: data model (first version)} to the case of a sequence of measurement vectors $(\Obs_n)_{n \in \Nset}$, we {must} move from the surjective finite sequence $\seq$ of Section \ref{subsec: signal model} to a surjective denumerable sequence $\seq = (\seq_n)_{n \in \Nset}$ of integers valued in $\intk$. Given that the convergence criterion considered in Theorem \ref{Theorem: fixed points} requires $\CtrMap$ to be considered as a variable, the following definitions us to unequivocally associate centroids to measurement vectors. This mapping is achieved via an element of the set $\Set$ of all sequences of centroids as , and a surjective sequence $\seq = (\seq_n)_{n \in \Nset}$ valued in $\intk$.
We consider a sequence of centroids in the form of an injective map $\CtrMap$, as introduced in Definition~\ref{def: sequence of centroids}. To extend the {initial} model of~\eqref{eq: data model (first version)} to the case of a sequence of measurement vectors $(\Obs_n)_{n \in \Nset}$, we {must} move from the surjective finite sequence $\seq$ of Section \ref{subsec: signal model} to a surjective denumerable sequence $\seq = (\seq_n)_{n \in \Nset}$ of integers valued in $\intk$. Given that the convergence criterion considered in Theorem \ref{Theorem: fixed points} requires $\CtrMap$ to be considered as a variable, the following definitions allow to associate unequivocally centroids to measurement vectors. We remind that $\Set$ is the set of all sequences of centroids as defined in Definition \ref{def: sequence of centroids}.
%\mysout{This mapping is achieved via an element of the set $\Set$ of all sequences of centroids, and a surjective sequence $\seq = (\seq_n)_{n \in \Nset}$ valued in $\intk$.}
\begin{Definition}
	\label{def: family F}
	Given a sequence $\MCov = (\MCov_n)_{n \in \Nset}$ of \spdms~and a {surjective} sequence $\seq = (\seq_n)_{n \in \Nset}$ valued in $\intk$, {$\family$ is defined as the set of sequences $\Wvec = (\Wvec_{n})_{n \in \Nset}$ of ${\mydim}$-dimensional real Gaussian random vectors for which {there} exists $\CtrMap \in \Set$ such that}
	\begin{equation}
		\label{eq: property of family}
		\forall n \in \Nset, \Wvec_{n} \thicksim \Ncal(\CtrMap({\seq_n}),\MCov_n) \myfstop
	\end{equation}
\end{Definition}

By definition, given $\Wvec \in \family$, there is a unique $\CtrMap \in \Set$ satisfying \eqref{eq: property of family}. Indeed, if $\CtrMap' \in \Set$ satisfies \eqref{eq: property of family}, $\CtrMap'(\seq_n) = \CtrMap(\seq_n)$ for all $n \in \Nset$. Since $\seq$ is surjective, given any $k \in \intk$, there exists at least one $n \in \Nset$ such that $k = \seq_n$, which implies that $\CtrMap'(k) = \CtrMap(k)$. Whence the next definition.
\begin{Definition}
	\label{Definition: CtrMap}
	Given a sequence $\MCov = (\MCov_n)_{n \in \Nset}$ of \spdms~and a {surjective} sequence $\seq = (\seq_n)_{n \in \Nset}$ valued in $\intk$, we define $$\CtrSeq: \family \to \Set$$ as the function assigning to each given $\Wvec = (\Wvec_{n})_{n \in \Nset} \in \family$ the unique sequence of centroids $\CtrMap \in \Set$ satisfying \eqref{eq: property of family}. 
    \vspace{0.1cm} \\
    \indent
    We denote by $\MVM$ the set of all functions $\Obs: \Set \to \family$ such that: %$\CtrSeq \circ \Obs = \id_{\Set}$: 
	\begin{equation}
		\label{eq: CtrSeq circ Obs = id}
		{\forall \CtrMap \in \Set, \CtrSeq(\aObsSeq) = \CtrMap} \myfstop
	\end{equation}
\end{Definition}

%\subsubsection{Measurement vector models}
%\label{sec: measurement vector models}
With the definitions introduced above, we enrich our model and notation for the measurement vectors to {explicitly} show the variables with respect to which the asymptotic result will be established. Specifically, consider any $\Obs \in \MVM$ and any $\CtrMap \in \Set$. The sequence $\aObsSeq = (\Squel{\CtrMap}{n})_{n \in \Nset}$ is assigned by $\Obs$ to $\CtrMap$. {According to \eqref{eq: CtrSeq circ Obs = id} and \eqref{eq: property of family}, we thus have:}
\begin{equation}
	\label{eq: property of Squel}
	\forall n \in \Nset, \Squel{\CtrMap}{n} \thicksim \Ncal(\CtrMap({\seq_n}), \MCov_{n}) \myfstop
\end{equation}
From now on, the measurement vectors are modelled and denoted by the random vectors $\Squel{\CtrMap}{n}$, and $\aObsSeq$ is the sequence of measurement vectors, on the basis of which we will estimate the centroids. 

\subsubsection{Randomized mean-shift function}
\label{sec: Randomized mean-shift function}
Let $\estMCov = (\estMCovn)_{n \in \Nset}$ be any sequence of \spdms. For any $N \in \Nset$, the first $N$ elements of $\estMCov$ can be used to specify $\Thefuncbasic(\Ctr)$ in \eqref{eq: our general function}. Given any sequence $\Wvec \! = \! (\Wvec_n)_{n \in \Nset}$ of elements of $\rv$ and any $\Ctr\in\Rset^\mydim$, set
% \begin{onecol}
	\begin{equation}
		\label{eq: new h (a) - 0}
		\TheFunc(\Ctr,\Wvec) 
		= \left ( \sum_{n=1}^N w \big ( \Maha{\estMCovn}^2 (\Wvec_n-\Ctr) \big ) \investMCovn \right )^{-1} \times \, \, \, \sum_{n=1}^N w \big ( \Maha{\estMCovn}^2 (\Wvec_n-\Ctr) \big ) \investMCovn \Wvec_n \myfstop
	\end{equation}
% \end{onecol}
Given any $\Obs \in \MVM$, the definition of $\Thefuncbasic(\Ctr)$ given by \eqref{eq: our general function} can thus be extended to the random function
% \begin{onecol}
\begin{equation}
	\label{eq: new h (a) - 1}
	\begin{array}{lll}
	\hspace{-0.4cm} 
	\TheFunc(\Ctr,\aObsSeq) = \left ( \displaystyle \sum_{n=1}^N w \big ( \Maha{\estMCovn}^2 (\Squel{\CtrMap}{n} - \Ctr) \big ) \investMCovn \right )^{-1} \\
    \hspace{5cm}
	\times \displaystyle \sum_{n=1}^N w \big ( \Maha{\estMCovn}^2 ( \Squel{\CtrMap}{n} - \Ctr) \big ) \investMCovn \Squel{\CtrMap}{n} \myfstop
\end{array}
\end{equation}
% \end{onecol}
The random function $\TheFunc(\Ctr,\aObsSeq)$ is the {randomized mean-shift function} we need in what follows. In particular, {$\Thefuncbasic(\Ctr)$} is the realization of $\TheFunc(\Ctr,\aObsSeq)$ when $\obs_1 = \Squel{\CtrMap}{1}(\omega), \obs_2 = \Squel{\CtrMap}{2}(\omega), \ldots, \obs_N = \Squel{\CtrMap}{N}(\omega)$ for some $\omega \in \Omega$. {Unlike $\Thefuncbasic$, {the notation $\TheFunc$ mentions explicitly} the sequence $\estMCov$ {appearing in its expression}. This sequence can be regarded as a parameter of $\TheFunc$. This heavier notation is needed because the statement of Theorem~\ref{Theorem: fixed points} involves both $\estMCov$ and another sequence {$\estDiagmat$} of covariance matrices to parameterize the randomized mean-shift function.}

\subsubsection{Distance between a sequence of centroids and a fixed vector}
\label{sec: Distance between a sequence of centroids and a fixed vector}
In Theorem \ref{Theorem: fixed points}, we make the distance between a given centroid and all the other ones grow to {infinity}. {We need Definitions \ref{definition: subset of injective maps} and \ref{definition: definition of the subset of interest} below to arrive at Definition \ref{definition: the distance} used to define the distance between a centroid and all others.}
\begin{Definition}
	\label{definition: subset of injective maps}
	Given $\ctrind \in \intk$ and $\Ctr_0 \in \Rset^\mydim$, $\Setbis{\Ctr_0} \subset \Set$ is the set of all $\CtrMap \in \Set$ whose $\ctrind^{\text{th}}$ element is $\Ctr_0$:
	$$\Setbis{\Ctr_0} = \left \{ \CtrMap \in \Set: \CtrMap(\ctrind) = \Ctr_0 \right \} \myfstop$$ 
\end{Definition}
We thus use $\Setbis{\Ctr_0}$ to fix the $\ctrind^{\text{\scriptsize{th}}}$ centroid at $\Ctr_0$ and make all the other ones move away from it. Proving that the sole fixed point of $\TheFunc(\Ctr,\aObsSeq)$ is asymptotically $\Ctr_0$, when the distance between this centroid and all the other ones tend to {infinity}, requires considering the vector measurement models $\Obs \in \MVM$ whose domain is $\Setbis{\Ctr_0}$. Thence the following definition and notation.
\begin{Definition}
	\label{definition: definition of the subset of interest}
	Given any $\ctrind \! \in \! \intk$ and any $\Ctr_0 \! \in \! \Rset^\mydim$, $\MVMbis$ is defined as the set of all functions $\Obs: \Setbis{\Ctr_0} \to \family$ such that 
	\begin{equation}
		\label{eq: CtrSeq circ Obs = id_{Setbis{Ctr_0}}}
		{\forall \CtrMap \in \Setbis{\Ctr_0}, \CtrSeq \left (\Obs(\CtrMap) \right ) = \CtrMap} \myfstop
	\end{equation}
\end{Definition}
\indent
To make centroids other than the $\ctrind^{\text{\scriptsize{th}}}$ move away from a given $\Ctr_0 \in \Rset^\mydim$, we use the following definition.
\begin{Definition}
	\label{definition: the distance}
	Given $\ctrind \! \in \! \intk$ and $\Ctr_0 \! \in \! \Rset^\mydim$, we define the function $\mydist{}: \Setbis{\Ctr_0} \to {\RightOpen{0}{\infty}}$ for any $\CtrMap \in \Setbis{\Ctr_0}$ by setting
	\begin{equation}
		\mydist{\CtrMap} = \min \big \{ \norm[ \CtrMap(k) - \Ctr_0 ] : k \in \intk \! \setminus \! \{\ctrind\} \big \} \myfstop
		\label{eq: definition of distmin}
	\end{equation}
\end{Definition}
Given $\CtrMap \in \Setbis{\Ctr_0}$, $\mydist{\CtrMap}$ measures the distance between the centroids $\CtrMap(k)$ with $k \neq \ctrind$ and $\Ctr_0 \in \Rset^{\mydim}$. We make this distance increase to {infinity} by using an appropriate base in $\Setbis{\Ctr_0}$, as stated in the next paragraph.

\subsubsection{Bases and limits}
\label{sec: bases and limits}
We recall that a base $\Base$ --- or filter base --- in a given set $\aSet$ is a nonempty subset of the power set of $\aSet$ such that \citep[p.~128, Definition 11]{Zorich2004}:
\begin{itemize}
    \item [(i)] $\emptyset \notin \Base$~;
    \item[(ii)] For all $(B_1,B_2) \in \Base \times \Base$, there exists $B \in \Base$ such that $B \subset B_1 \cap B_2$~. 
\end{itemize}
\indent
Consider a base $\Base$ in a set $\aSet$. Since all norms on $\Rset^{\dimb}$ are equivalent, it follows from \citet[p.~129, Definition 12]{Zorich2004} that $\mylim \in \Rset^{\dimb}$ is the limit of a function $\fbm : \aSet \to \Rset^{\dimb}$ over $\Base$, which we write $\lim_{\Base} \fbm = \mylim$, if 
\begin{equation}
	\label{eq: def of a limit-any norm}
	\hspace{-0.2cm}
	\forall \eta \in \Open{0}{\infty}, \exists B \in \Base, \forall \variable \in \aSet, \left ( \variable \in B \Rightarrow \norm[\fbm(\variable) - \mylim] < \eta \right ) \myfstop
\end{equation}
As a basic and illustrative example, $\Base_\Nset 
= \{ \, \rrbracket n,\infty \llbracket: n \in \Nset \, \}$ is, in $\Nset$, the base with respect to which we define the convergence of a sequence $(x_n)_{n \in \Nset}$ of real values when $n$ tends to $\infty$.

The following facts will be helpful. If $\Base$ is a base in a set $\aSet$ and $\transform: \aSet \to \aSet'$ is a bijection, then the set $\transform(\Base) = \{ \transform(B): B \in \Base \}$ of all the direct images of the elements of $\Base$ is a base in $\aSet'$ and \eqref{eq: def of a limit-any norm} implies that, for any $\fbm : \aSet \to \Rset^{\dimb}$ with a limit over $\Base$, 	
\begin{equation}
	\label{eq: transform of a base and limits}
	\displaystyle \lim_{\Base} \fbm = \displaystyle \lim_{\transform(\Base)} (\fbm \circ \transform^{-1}) \myfstop
\end{equation}
Considering any set $\aSet$, if $H: \aSet \to \RightOpen{0}{\infty}$ is a function such that $H^{-1} \big( \, \Open{r}{\infty} \, \big ) \neq \emptyset$ for all $r \in \RightOpen{0}{\infty}$, then 
\begin{equation}
	\label{eq: example of a base via a function}
	\Basefn{\aSet}{H} = \Big \{ H^{-1} \big ( \, \Open{r}{\infty} \, \big ) : r \in \RightOpen{0}{\infty}) \Big \}
\end{equation}
is a base in $\aSet$. Given a function $\fbm: \aSet \to \Rset^{\mydim}$ and $\mylim \in \Rset^{\mydim}$, we write $\displaystyle \lim_{H(\variable) \to \infty} \fbm(\variable) = \mylim$ to mean that $\lim_{\Basefn{\aSet}{H}} \fbm = \mylim$. Therefore, according to \eqref{eq: def of a limit-any norm}, we have
\begin{onecol}
	\begin{equation}
		\label{eq: lim over B(H)}
		\displaystyle \lim_{\metric(\variable) \to \infty} \fbm(\variable) = \mylim \myiff \Big ( \, \forall \eta \in \Open{0}{\infty}, \exists \, r \in \RightOpen{0}{\infty}, \forall \variable \in \aSet, \metric(\variable) > r \Rightarrow \norm[\fbm(\variable) - \mylim] < \eta \, \Big ) \myfstop
	\end{equation}
\end{onecol}
% \begin{twocol}
% \begin{equation}
% 	\label{eq: lim over B(H)}
% 	\begin{array}{lll}
% 		\hspace{-0.25cm}
% 		\displaystyle \lim_{\metric(\variable) \to \infty} f(\variable) = \mylim \vspace{0.1cm} \\
% 		\hspace{0.1cm} 
% 		\myiff \Big ( \, \forall \eta \in \Open{0}{\infty}, \exists \, r \in \Open{0}{\infty}, \forall \variable \in \aSet, \\
% 		\qquad \qquad \qquad \quad \metric(\variable) > r \Rightarrow \norm[f(\variable) - \mylim] < \eta \, \Big )
% 	\end{array}
% 	\hspace{-0.5cm}
% \end{equation}
% \end{twocol}
In addition, suppose that we have a bijection $\gamma : \aSet \to \aSet'$ between two arbitrary sets and a function $H: \aSet \to \RightOpen{0}{\infty}$ such that, for any $r \in \RightOpen{0}{\infty}$, $H^{-1} \big( \, \Open{r}{\infty} \, \big ) \neq \emptyset$. We derive from the foregoing that $\Basefn{\aSet'}{H \circ \gamma^{-1}}$, defined according to \eqref{eq: example of a base via a function}, equals $\gamma(\Basefn{\aSet}{H})$ and is a base in $\aSet'$. It thus follows by using \eqref{eq: transform of a base and limits} and the writing conventions introduced above that, for any $f : \aSet \to \Rset^{\dimb}$ for which $\displaystyle \lim_{H(\variable) \to \infty} \fbm(\variable)$ exists, we have:
\begin{equation}
	\label{eq: change of variable}
	\displaystyle \lim_{H(\variable) \to \infty} \fbm(\variable) = \displaystyle \lim_{(H \circ \gamma^{-1})(y) \to \infty} (\fbm \circ \gamma^{-1}) (y) \myfstop
\end{equation}
% Below, we need the following bases. 
% First, consider a sequence $(x_N)_{N \in \Nset}$ of real numbers. In \eqref{eq: def of a limit-any norm}, set $\aSet = \Nset$, define $f: \Nset \to \Rset$ for each $N \in \Nset$ by $f(N) = x_N$ and choose the absolute value for the norm in $\Rset$. It results that
% \begin{equation}
% 	\label{eq: base in Nset}
% 	\BaseNset = \{ \, {\rrbracket N,\infty \llbracket}: N \in \Nset \, \}
% \end{equation} 
% is the base in $\Nset$ with respect to which the convergence of the sequence $(x_N)_{N \in \Nset}$ when $N$ tends to infinity is defined.
\\
\indent
Now, fix $\ctrind \in \intk$ and $\Ctr_0 \in \Rset^{\mydim}$. If $\metric: \Setbis{\Ctr_0} \to \RightOpen{0}{\infty}$ is a function such that $\metric^{-1}\big ( \, \Open{r}{\infty} \, \big ) \neq \emptyset$ for any $r \in \RightOpen{0}{\infty}$ then 
\begin{equation}
	\label{eq: our base}
	\GoodBase{\metric} = \Big \{ \metric^{-1} \big ( \, \Open{r}{\infty} \, \big ) : r \in \RightOpen{0}{\infty} \Big \}
\end{equation}
is a base in $\Setbis{\Ctr_0}$. 
In particular, we identify
\begin{equation}
	\label{eq: the base}
	\hspace{-0.35cm} 
	\GoodBase{\mydist{}} = \Big \{ \left ( \mydist{} \right)^{-1} \big ( \, \Open{r}{\infty} \, \big ) : r \in \RightOpen{0}{\infty} \Big \}
\end{equation}
as the suitable base in $\Setbis{\Ctr_0}$ to study the limit of a function $\fbm: \Setbis{\Ctr_0} \to \Rset^{\mydim}$ when $\mydist{\CtrMap}$ tends to $\infty$. If $\fbm$ has a limit $\mylim$ over $\GoodBase{\mydist{}}$, we thus write $$\displaystyle \lim_{\mydist{\CtrMap} \to \infty} \fbm(\CtrMap) = \mylim$$ 
and it follows from \eqref{eq: lim over B(H)} \& \eqref{eq: the base} that
\begin{onecol}
\begin{equation}
	\label{eq: lim over B(Delta)}
	\begin{array}{lll}
		\hspace{-0.3cm} \displaystyle \lim_{\mydist{\CtrMap} \to \infty} \fbm(\CtrMap) = \mylim \vspace{0.1cm} \\
		\quad \, \myiff \Big ( \, \forall \eta \in \Open{0}{\infty}, \exists \, r \in \RightOpen{0}{\infty}, \forall \, \CtrMap \in \Setbis{\Ctr_0}, \mydist{\CtrMap} > r \Rightarrow \norm[ \fbm(\CtrMap) - \mylim ] < \eta \, \Big ) \myfstop
	\end{array}
\end{equation}
\end{onecol}
% \begin{twocol}
% 	\begin{equation}
% 		\label{eq: lim over B(Delta)}
% 		\begin{array}{lll}
% 			\hspace{-0.3cm} \displaystyle \lim_{\mydist{\CtrMap} \to \infty} f(\CtrMap) = \mylim \vspace{0.1cm} \\
% 			\quad \, \myiff \Big ( \forall \eta \in \Open{0}{\infty}, \exists r \in \Open{0}{\infty}, \forall \CtrMap \in \Setbis{\Ctr_0}, \\
% 			\hspace{2cm} \mydist{\CtrMap} > r \Rightarrow \norm[ f(\CtrMap) - \mylim ] < \eta \Big )
% 		\end{array}
% 		%	\nonumber
% 	\end{equation}
% \end{twocol}
In Theorem \ref{Theorem: fixed points}, we make $N$ and $\mydist{\CtrMap}$ grow to $\infty$. This requires the specific convergence criterion introduced in the next section.

\subsection{Essential convergence in limsup (ess.lim.sup)}
\label{sec: Essential convergence in limsup}
The definition of $\esl$ is novel. It formalizes and extends the convergence criterion \citep[Sec.~2.1]{Pastor-CSDA} used to state \citet[Theorem 1]{Pastor-CSDA}.

Define a parametric discrete random process as any function $\Process{}: \aSet \times \Nset \to \rv$, where $\aSet$ is some set, typically a set of parameters. Such a function $\Process{}$ assigns to each pair $({\variable},N) \in \aSet \times \Nset$ a unique $\mydim$-dimensional real random vector $\Process{\variable,N} \in \rv$. For a given $\Obs \! \in \! \MVM$ and a given $\Ctr \in \Rset^{\mydim}$, the function  $(\CtrMap,N) \mapsto \TheFunc(\Ctr,\aObsSeq)$ defined in \eqref{eq: new h (a) - 1} is a parametric discrete random function defined on $\MVM \times \Nset$. \\
\indent Given any parametric discrete random process $\Process{}: \aSet \times \Nset \to \rv$, we now define the $\esl$ convergence of $\Process{\variable,N}$ over a base $\Base$ in $\aSet$ when $N$ tends to $\infty$ as follows.
\begin{Definition}
	\label{Definition: lim of limsup}
	Let $\mylim \in \Rset^{\mydim}$ and $\Process{}: \aSet \times \Nset \to \rv$ be a parametric discrete random process.
	% $$\begin{array}{llll}
	% 	\Process{}: & \aSet \times \Nset & \to & \rv \\
	% 	& (\variable,N) & \mapsto & \Process{\variable,N}
	% \end{array}$$
	For any pair $(\variable,N) \in \aSet \times \Nset$, set
	$$\begin{array}{llll}
		\norm[\Process{\variable,N}-\mylim \,]: & \Omega & \to & \Rset \\
		& \omega & \mapsto & \norm[\Process{\variable,N}(\omega) - \mylim \,] \myfstop
	\end{array}$$
	Given a base $\Base$ in $\aSet$, we write $$\esssuplim{\Base}{N}\Process{\variable,N} = \mylim$$ and say that $\Process{\variable,N}$ converges essentially in limsup --- or simply, in $\esl$ --- to $\mylim$, over $\Base$ when $N$ tends to $\infty$, if $$\limlimsup{\Base}{N}{\Process{\variable,N} - \mylim} = 0  \myfstop$$
%	\qed
\end{Definition}
With the same notation as in the definition above, we straightforwardly have
% \begin{onecol}
 	\begin{equation}
 	\label{eq: esl centred}
 	\esssuplim{\Base}{N}\Process{\variable,N} = \mylim \, \myiff \, \esssuplim{\Base}{N} \left ( \Process{\variable,N} - \mylim \right ) = 0 \myfstop
\end{equation}
% \end{onecol}
% \begin{twocol}
% 	\begin{equation}
% 	\label{eq: esl centred}
% 	\begin{array}{lll}
% 		\esssuplim{\Base}{N}\Process{\variable,N} = \mylim \vspace{0.1cm} \\
% 		\hspace{2cm} \myiff \esssuplim{\Base}{N} \left ( \Process{\variable,N} - \mylim \right ) = 0
% 	\end{array}
% 	\end{equation}
% \end{twocol}
Below, we consider the particular case where $\aSet = \Setbis{\Ctr_0}$ and $\GoodBase{\mydist{}}$ defined by \eqref{eq: the base} are used to make $\mydist{\CtrMap}$ tend to $\infty$. Thus, with the notation introduced above, given $\mylim \in \Rset^{\mydim}$ and any parametric discrete random process $\Process{}: \Setbis{\Ctr_0} \times \Nset \to \rv$, we have :
% \begin{onecol}
\begin{equation}
	\label{eq: esl of a discrete process in our case}
	\esssuplim{\mydist{\CtrMap} \to \infty}{N}\Process{\CtrMap,N} = \mylim
	\, \, \, \myiff \limlimsup{\mydist{\CtrMap} \to \infty}{N}{\Process{\CtrMap,N} - \mylim} = 0 \myfstop
%	\nonumber
\end{equation}
% \end{onecol}
% \begin{twocol}
% 	\begin{equation}
% 		\label{eq: esl of a discrete process in our case}
% 		\begin{array}{lll}
% 			\hspace{-0.3cm} \esssuplim{\mydist{\CtrMap} \to \infty}{N}\Process{\CtrMap,N} = \mylim \vspace{0.1cm} \\
% 			\hspace{0.1cm} \myiff \limlimsup{\mydist{\CtrMap} \to \infty}{N}{\Process{\CtrMap,N} - \mylim} = 0
% 		\end{array}
% 		\nonumber
% 	\end{equation}
% \end{twocol}

\subsection{Main result}
\label{subsec: main results}
In this section, we state our main theoretical result, namely Theorem \ref{Theorem: fixed points}. This asymptotic result relies on the following assumptions, completing those made in Sections \ref{sec:model} and \ref{sec:robust}. These assumptions are presented and commented below before stating the theorem. 
\\
\indent
Regarding the sequence $\seq = (\seq_n)_{n \in \Nset}$ that assigns the $\seq_n^{\text{th}}$ centroid to the measurement vector number $n$, we make the following assumption.

\begin{Assumption}
\label{Assumption: on alpha}
 $\seq = (\seq_n)_{n \in \Nset}$ is a surjective sequence of integers in $\intk$ for which {there} exists $\thecoeff_0 \in \Open{0}{1}$ such that $$\forall (k,N) \in \intk \times \NsetRestricted, \thecoeff_0 \leqslant \textrm{card} \left \{ n \in \intn: \seq_n = k \right \} / N \myfstop$$ 
\end{Assumption}

\begin{assumptioncomment}
The existence of the uniform lower bound $a_0 > 0$ for the proportion of measurement vectors belonging to a given cluster implies that the {cardinality of each cluster cannot be $0$ and} will increase to {infinity} as the number $N$ of data grows. {Thereby, no cluster becomes negligible in cardinality with respect to the others as $N$ tends to {infinity}.}
\end{assumptioncomment}

\begin{Assumption}
\label{Assumption: on w}
The kernel $w: \RightOpen{0}{\infty} \to \Open{0}{\infty}$ is a continuous and non-increasing function verifying $$\displaystyle \lim\limits_{t \rightarrow \infty} {t w ( t )} = 0 \myfstop$$
\end{Assumption}

\begin{assumptioncomment}
We consider kernels that are smooth enough to achieve the estimation of the centroids. This assumption is mild enough to encompass a large class of possible kernels, among which the Gaussian one.
\end{assumptioncomment}
\indent
The covariance matrices $\MCov_n$ of the measurement vectors are not necessarily equal, even when they correspond to data within the same cluster. We however introduce some coherence between these covariance matrices. In this respect, we make the following assumption.

\begin{Assumption}
\label{Assumption: on C}
The covariance matrices $\MCov_n$ for $n \in \Nset$ are {diagonalizable in the same orthonormal basis of $\Rset^\mydim$ and there exists an interval $[\eigvalmin, \eigvalmax] \subset \Open{0}{\infty}$ to which their eigenvalues all belong}.
\end{Assumption}

\begin{assumptioncomment}
The coherence between the covariance matrices is {in the sense} that they all have {the} same eigenvectors. {In contrast, the spectra of the covariance matrices are not necessarily identical. The bounds $\eigvalmin$ and $\eigvalmax$ are the sole parameters needed in the sequel to analyze the impact of the eigenvalues on the randomized mean-shift function $\TheFunc$.}
\end{assumptioncomment}
%In addition, the eigenvalues of the covariance matrices $\MCov_n$ can be unknown in practice but we bound our possible lack of prior knowledge on these eigenvalues.
\indent
We need a sequence $(\estMCovn)_{n \in \Nset}$ to define {the randomized mean-shift function} $\TheFunc(\Ctr,\Wvec)$ defined in \eqref{eq: new h (a) - 0}. {Considering, as suggested in Remark \ref{rmk: fundamental remark}, that $\estMCovn$ is an estimate of $\MCov_n$ for each $n \in \Nset$, we make the following similar assumptions for the matrices $\estMCovn$ as for the matrices $\MCov_n$}.

\begin{Assumption}
\label{Assumption: on lambda}
The covariance matrices $\estMCovn$ are diagonalizable in the same orthonormal basis as the covariance matrices $\MCov_n$ for $n \in \Nset$, and there exists an interval $[\esteigvalmin, \esteigvalmax] \subset \Open{0}{\infty}$ to which all their eigenvalues belong.
\end{Assumption}
% such that, for all $n \! \in \! \Nset$, $\Spec(\estMCovn) \! \subset \! \esteiginterval$.

\begin{assumptioncomment}
Considering, according to Remark \ref{rmk: fundamental remark}, that the matrices $\estMCovn$ are estimates of the matrices $\MCov_n$, we assume that the former share the same eigenvectors as the latter. On the other hand, the eigenvalues of the matrices $\estMCovn$ can be unknown, but {they must have} bounded possible variations. Note that the interval of variations of these eigenvalues can differ from $[\eigvalmin, \eigvalmax]$. As for \cref{Assumption: on C}, $\esteigvalmin$ and $\esteigvalmax$ are the sole parameters used below to study how the eigenvalues impact $\TheFunc$.
\end{assumptioncomment}

%\vspace{0.1cm}
The following consequences of \cref{Assumption: on C,Assumption: on lambda} will be useful to state Theorem \ref{Theorem: fixed points}. %To state Theorem \ref{Theorem: fixed points}, we also need the following material. 
First, in what follows, when we are given any two sequences $\MCov = (\MCov)_{n \in \Nset}$ and $\estMCov = (\estMCovn)_{n \in \Nset}$ of \spdms~satisfying \cref{Assumption: on C,Assumption: on lambda}, we can write that
\begin{subequations}
	\label{eq: eigen decompositions}
	\begin{empheq}[left={\forall n \in \Nset,} \empheqlbrace]{align}
		& \MCov_n = \Rmat \Diagmatn \Rmat^\transpose \mycomma\label{eq: eigen decomposition of MCov} \\
		& \estMCovn = \Rmat \estDiagmatn \Rmat^\transpose \mycomma \label{eq: eigen decomposition of estMCov}
	\end{empheq}
\end{subequations}
where:
\begin{itemize}
	\item $\Rmat$ is a $\mydim \times \mydim$ orthogonal matrix whose columns are the eigenvectors of all the matrices $\MCov_n$ and $\estMCovn$ for $n \in \Nset$;
	\item For any given $n \in \Nset$, $\Diagmat_n$ and $\estDiagmatn$ are diagonal matrices whose diagonal elements are the eigenvalues of $\MCov_n$ and $\estMCovn$, respectively.
\end{itemize}
%\vspace{0.1cm}
According to \eqref{eq: definition of the initial mahanorm}, the Mahalanobis norms $\Maha{\MCov_n}$ and $\Maha{\estMCovn}$ associated with $\MCov_n$ and $\estMCovn$, respectively, are thus given by:
\begin{subequations}
	\label{eq: Mahalanobis norms}
	\begin{empheq}[left={\forall n \in \Nset, \forall \xvec \in \Rset^\mydim,} \empheqlbrace]{align}
		& \Maha{\MCov_n}(\xvec) = \Maha{\Diagmat_n} \! (\Rmat^\transpose {\xvec}) \mycomma
		\label{eq: Mahalanobis norm associated with MCov_n}
		\\
		& \Maha{\estMCovn}(\xvec) = \Maha{\estDiagmatn} \! (\Rmat^\transpose \xvec) \myfstop
		\label{eq: estimated Mahalanobis norm - relation with euclidean}
	\end{empheq}
\end{subequations}

Second, consider $\Wvec = (\Wvec_{n})_{n \in \Nset} \in \family$ (see Definition \ref{def: family F}) and set $\CtrMap = \CtrSeq(\Wvec)$ (see Definition \ref{Definition: CtrMap}). From \eqref{eq: property of family} \& \eqref{eq: eigen decomposition of MCov}, we thus have:
\begin{equation}
	\label{eq: distribution of Rmat'Wn}
	\forall n \in \Nset, \Rmat^\transpose \Wvec_{n} \thicksim \Ncal(\Rmat^\transpose \CtrMap(\seq_n), \Diagmat_n) \myfstop
\end{equation}
Since $\Diagmat \! = \! (\Diagmat_n)_{n \in \Nset}$ is a sequence of diagonal positive-definite matrices and $\Rmat^\transpose \CtrMap \in \Set$, it follows from \eqref{eq: distribution of Rmat'Wn} that we can define the function: 
\begin{equation}
	\label{eq: definition of W bar}
	\hspace{-0.2cm}
	\begin{array}{lllll}
		\fn{\Rmat^\transpose} : 
		& \family & \to & \familyfn{\Diagmat} 
		\vspace{0.1cm} \\
		& \Wvec & \mapsto &	%\fn{\Rmat^\transpose} \Wvec = 
		(\Rmat^\transpose \Wvec_{n})_{n \in \Nset} \myfstop
	\end{array}
\end{equation}

%\vspace{0.1cm}
\indent
{With the material introduced above, we can now state our main theoretical result.}
\begin{Theorem}
\label{Theorem: fixed points}
Consider a sequence $\seq \! = \! (\seq_n)_{n \in \Nset}$ of integers valued in $\intk$, a function $w: \RightOpen{0}{\infty} \to \Open{0}{\infty}$ and two sequences $\MCov = (\MCov)_{n \in \Nset}$ and $\estMCov = (\estMCovn)_{n \in \Nset}$ of \spdms. If \cref{Assumption: on alpha,Assumption: on w,Assumption: on C,Assumption: on lambda} are verified, then:
% \begin{onecol}
\begin{equation}
	\label{eq: the limit of the essential supremum of the limsup (vers 2)}
	\begin{array}{cccc}
		\hspace{-2cm} 
		\forall \, \ctrind \in \intk, \forall \, (\Ctr_0,\Ctr) \in \Rset^{\mydim} \times \Rset^{\mydim}, 
		\forall \, \YObsSeq \in \MVM, \medskip \\
		\begin{array}{lll}
			\Ctr = \Ctr_0 
			& \stackrel{\myclubsuit}{\myiff} & \esssuplim{\mydist{\CtrMap} \to \infty}{N} \TheFunc(\Ctr,\aObsSeq) = \Ctr 
			\medskip \\
			& \stackrel{\myspadesuit}{\myiff} & \esssuplim{\mydist{\CtrMap} \to \infty}{N} \TheFuncb(\Rmat^\transpose \Ctr,\fn{\Rmat^\transpose}(\aObsSeq)) = \Rmat^\transpose \Ctr \myfstop
		\end{array}
	\end{array}
\end{equation}
% \end{onecol}
% \begin{twocol}
% \begin{equation}
% 	\label{eq: the limit of the essential supremum of the limsup (vers 2)}
% 	\begin{array}{cccc}
% 		\hspace{-0.1cm} \forall \ctrind \! \in \! \intk, \forall (\Ctr_0,\Ctr) \! \in \! \Rset^{\mydim} \! \times \! \Rset^{\mydim}, 
% 		\forall \YObsSeq \! \in \! \MVM, \medskip \\
% 		\begin{array}{lll}
% 		\hspace{-0.7cm}
% 		\Ctr = \Ctr_0 \medskip \\
% 		\hspace{-0.4cm}
% 		\stackrel{\myclubsuit}{\myiff} \esssuplim{\mydist{\CtrMap} \to \infty}{N} \TheFunc(\Ctr,\aObsSeq) = \Ctr 
% 		\medskip \\
% 		\hspace{-0.4cm}
% 		\stackrel{\myspadesuit}{\myiff} \esssuplim{\mydist{\CtrMap} \to \infty}{N} \TheFuncb(\Rmat^\transpose \Ctr,\fn{\Rmat^\transpose}(\aObsSeq)) = \Rmat^\transpose \Ctr
% 		\end{array}
% 	\end{array}
% 	\vspace{0.25cm}
% \end{equation}
% \end{twocol}
where $\estDiagmat = (\estDiagmatn)_{n \in \Nset}$ is the sequence of matrices defined by \eqref{eq: eigen decomposition of estMCov}, and $\TheFunc$ and $\TheFuncb$ are calculated according to \eqref{eq: new h (a) - 1}.
%\begin{onecol}
%\begin{equation}
%\label{eq: new h (a)}
%	\TheFunc(\Ctr,\aObsSeq) = \left ( \sum_{n=1}^N w \big ( \Maha{\estMCovn}^2 (\Squel{\CtrMap}{n}-\Ctr) \big ) \investMCovn \right )^{-1} \sum_{n=1}^N w \big ( \Maha{\estMCovn}^2 (\Squel{\CtrMap}{n}-\Ctr) \big ) \investMCovn \Squel{\CtrMap}{n}
%\end{equation}
%\end{onecol}
%\begin{twocol}
%\begin{equation}
%	\label{eq: new h (a)}
%	\begin{array}{lll}
%		\hspace{-0.3cm} 
%		\TheFunc(\Ctr,\aObsSeq)
%		\vspace{0.15cm} \\
%		= \left ( \sum_{n=1}^N w \big ( \Maha{\estMCovn}^2 (\Squel{\CtrMap}{n}-\Ctr) \big ) 	\investMCovn \right )^{-1} 
%		\vspace{0.15cm} \\ 
%		\hspace{0.8cm} \times \sum_{n=1}^N w \big ( \Maha{\estMCovn}^2 (\Squel{\CtrMap}{n}-\Ctr) \big ) \investMCovn \Squel{\CtrMap}{n}
%	\end{array}
%\end{equation}
%\end{twocol}
\end{Theorem}
\begin{proof}
	Because of its length and technicality, the proof is postponed to Appendix \ref{App: the proof}.
\end{proof}
%\vspace{0.1cm}
\indent

Suppose that $\CtrMap: \intk \to \Rset^{\mydim}$ is a sequence of centroids and that $(\obs_1, \ldots, \obs_N)$ is a realization of $(\Obs_1, \ldots, \Obs_N)$ under the hypotheses of Section \ref {subsec: signal model}. We readily have $\CtrMap \in \Setbis{\CtrMap(\ctrind)}$ for each $\ctrind \in \intk$. To estimate $\CtrMap$ in practice via our algorithm \algo~described in Section \ref{sec:algo}, we will implicitly assume that the centroids $\CtrMap(1), \CtrMap(2), \ldots, \CtrMap(K)$ are sufficiently far away from each other to apply Theorem \ref{Theorem: fixed points} iteratively for each $\ctrind \in \intk$. We will thus consider that, given any $\ctrind \in \intk$,  $N$ and $\mydistfn{\CtrMap}{\CtrMap(\ctrind)}$ are large enough for $\CtrMap(\ctrind)$ to be a fixed point of the realization ${\Thefuncbasic}$ of $\TheFunc(\cdot,\aObsSeq)$, where $\aObsSeq = (\Squel{\CtrMap}{n})_{n \in \Nset}$ with $\Squel{\CtrMap}{1} = \Obs_1, \Squel{\CtrMap}{2} = \Obs_2, \ldots, \Squel{\CtrMap}{N} = \Obs_N$. In the next section, we explore how the fixed-points of ${\Thefuncbasic}$ can be estimated iteratively.
\begin{Remark}
	\label{rmk: transforms}
	{According to \cref{Assumption: on lambda}, the matrices $\MCov_{n}$ and $\estMCovn$ share the same eigenvectors concatenated in $\Rmat$. This may appear as a restrictive assumption. However, many works have shown that the Fourier coefficients and the wavelet and wavelet packet coefficients of stationary random signals are asymptotically uncorrelated (see for instance, \citet{Pearl73, Hamidi1975, pesq99, mielniczuk2005wavelets, atto07, Atto2008CentralLT, Zhang1994}, among others). In this respect, $\fn{\Rmat^\transpose}$ models such a transform chosen in advance for its ability to diagonalize sufficiently well the covariance matrices of the measurement vectors. This is why, in the statement of Theorem \ref{Theorem: fixed points}, we have two equivalences: $\myclubsuit$ involves the initial measurement vectors, whereas $\myspadesuit$ is applied to the outcome of the transform $\fn{\Rmat^\transpose}$. In the latter case, estimating the cluster centroids in practice} will be achieved by considering the transformed data $\oobs_1 = \Rmat^\transpose \obs_1, \ldots, \oobs_N = \Rmat^\transpose \obs_N$ and calculating the fixed points $\CCtr_1, \CCtr_2, \ldots$ of the function
	\begin{onecol}
	\begin{equation}
		\label{eq: function after transform}
		\begin{array}{lll}
			\hspace{-0.5cm}
			{\widetilde{\Thefuncbasic}(\CCtr)}
			= \left ( \displaystyle \sum_{n=1}^N w \big ( \Maha{\estDiagmatn}^2 (\oobs_n-\CCtr) \big ) \estDiagmatn^{-1} \right )^{-1} 
			\hspace{-0.1cm} \times \displaystyle \sum_{n=1}^N w \big ( \Maha{\estDiagmatn}^2 (\oobs_n-\CCtr) \big ) \estDiagmatn^{-1} \oobs_n \myfstop
		\end{array}
	\end{equation}
	\end{onecol}
	% \begin{twocol}
	% \begin{equation}
	% 	\label{eq: function after transform}
	% 	\begin{array}{lll}
	% 		\hspace{-0.5cm}
	% 		\modif{\widetilde{\Thefuncbasic}(\CCtr)}
	% 		= \left ( \sum_{n=1}^N w \big ( \Maha{\estDiagmatn}^2 (\oobs_n-\CCtr) \big ) \estDiagmatn^{-1} \right )^{-1} 
	% 		\vspace{0.1cm} \\ 
	% 		\hspace{1cm} \times \sum_{n=1}^N w \big ( \Maha{\estDiagmatn}^2 (\oobs_n-\CCtr) \big ) \estDiagmatn^{-1} \oobs_n
	% 	\end{array}
	% \end{equation}
	% \end{twocol}
	{obtained by replacing $\MCov_{1}, \ldots, \MCov_N$ in $\Thefuncbasic$ by $\estDiagmatone, \ldots, \estDiagmatN$}, and estimating the centroids by $\Rmat \CCtr_1, \Rmat \CCtr_2, \ldots$.
\end{Remark}

\subsection{Iterative estimation of the fixed-points}\label{sec:fixed_point_est}

Theorem~\ref{Theorem: fixed points} justifies that the centroids can be estimated by calculating the fixed points of ${\Thefuncbasic}$ defined in \eqref{eq: our general function}.  One method widely employed in practice to estimate the fixed-points of {a function $f$ consists of iteratively computing $\Ctr^{(\ell+1)} = f(\Ctr^{(\ell)})$ until a certain stopping criterion is reached~\citep{zoubir12robust}. Such an iterative computation is used by \citet{comaniciu2002mean} to calculate the fixed-points of the \emph{mean-shift} \citep[Eq.~17]{comaniciu2002mean}. However, the proof that the resulting sequence $(\Ctr^{(\ell)})_{\ell \in \Nset}$ converges is not correct in \citet{comaniciu2002mean}, as shown by \citet{ghassabeh2018modified}, and remains an open problem~\citep{ghassabeh2013convergence,ghassabeh2015sufficient,ghassabeh2018modified,yamasaki2019properties}.} Although the next result does not prove the convergence, it provides us with a means to devise a stopping criterion for the iterative procedure $\Ctr^{(\ell+1)} = \Thefuncbasic(\Ctr^{(\ell)})$. {This stopping criterion applies to the \emph{mean-shift} as a particular case.}
\begin{Proposition}
\label{Proposition:the iterative method converges}
	Suppose that the function $w: \RightOpen{0}{\infty} \to \Open{0}{\infty}$ is differentiable, bounded and non-increasing. Let $(\obs_1, \obs_2, \ldots, \obs_N)$ be $N$ elements of $\Rset^d$. If $\Ctr^{(\ell+1)} = {\Thefuncbasic(\Ctr^{(\ell)})}$ for any $\ell \in \Nset$ and  $\Ctr^{(1)}$ is arbitrarily chosen in $\Rset^{\mydim}$ then:
	\begin{onecol}
	\begin{equation}
		\forall \varepsilon > 0, \exists \ell_0(\varepsilon) \in \Nset, \forall L \in \Nset, \forall \ell \in \Nset, \ell \geqslant \ell_0(\varepsilon) \Rightarrow \norm[\Ctr^{(\ell+L)} - \Ctr^{(\ell)}] \leqslant L \varepsilon \myfstop
		\label{eq:pseudo-convergence of xil}
	\end{equation}
	\end{onecol}
	% \begin{twocol}
	% \begin{equation}
	% 	\begin{array}{lll}
	% 		\forall \varepsilon > 0, \exists L_0(\varepsilon) \in \Nset, \forall L \in \Nset, \forall \ell \in \Nset, \vspace{0.1cm} \\ 
	% 		\qquad \qquad \ell \geqslant L_0(\varepsilon) \Rightarrow \norm[\Ctr^{(\ell+L)} - \Ctr^{(\ell)}] \leqslant L \varepsilon
	% 	\end{array}
	% 	\label{eq:pseudo-convergence of xil}
	% \end{equation}
	% \end{twocol}
\end{Proposition}
\begin{proof} See Appendix \ref{App: proof of the convergence}.
\end{proof}
\vspace{0.1cm}
\indent
Especially, for $L=1$, the proposition states that $\norm[\Ctr^{(\ell+1)} - \Ctr^{(\ell)}]$ can be made arbitrarily small for large enough $\ell$, which suggests considering $\norm[\Ctr^{(\ell+1)} - \Ctr^{(\ell)}] \leqslant \epsilon$ as a stopping criterion for the iterative procedure. 
Numerical simulations will confirm that using this iterative procedure together with the identified criterion leads to correct centroid estimation, as also empirically observed in other clustering algorithms like Mean-Shift~\citep{comaniciu2002mean}. 

From a practical point of view, there are several other issues to address so as to devise an efficient method to estimate the $K$ centroids of the set $\CtrSet$. For instance, we must 
%should 
decide how to initialize the recursive computation {$\Ctr^{(\ell+1)} = {\Thefuncbasic(\Ctr^{(\ell)})}$}, and how to find a way to retrieve not only one, but $K$ fixed points, while the value of $K$ is \emph{a priori} unknown. The clustering algorithm introduced in Section~\ref{sec:algo} addresses these issues to correctly build $K$ clusters on the basis of Theorem \ref{Theorem: fixed points} and thus, without prior knowledge of $K$. 

{In addition,} the theoretical results of Theorem~\ref{Theorem: fixed points} and Proposition~\ref{Proposition:the iterative method converges} rely on the properties of the weight function $w$. This is why, before describing the clustering algorithm, we introduce a new weight function satisfying~\cref{Assumption: on w} and the assumptions of Proposition~\ref{Proposition:the iterative method converges}.

\section{The Wald kernel}
\label{sec:w_function}
Theorem~\ref{Theorem: fixed points} holds for a large class of weight functions $w$. As a first option, the Gaussian weight function $w(t)= \exp(-t/2h)$, for $t \in \RightOpen{0}{\infty}$, considered in~\citet{wu02PR},\citet{comaniciu2002mean}, \citet{ghassabeh2015sufficient}, verifies the properties required in Theorem~\ref{Theorem: fixed points} and~Proposition \ref{Proposition:the iterative method converges}.
However, the bandwidth parameter $h$ of this weight function must be chosen empirically and its optimal value varies with the dimension $\mydim$ \citep{huang2018convergence}. %and with the noise parameters $\sigma_n^2$.
A poor choice of $h$ can dramatically impact the performance of the clustering algorithms proposed in~\citet{wu02PR} and \citet{comaniciu2002mean}. In contrast, we introduce a novel kernel whose expression explicitly depends on the dimension $\mydim$. This kernel is called the Wald kernel because it derives from the Wald test \citep{Wald1943} for testing the mean of a Gaussian. Because of the specificity of the material needed to introduce the Wald kernel, we begin by recalling basics about the Wald test and its p-value.

\subsection{The Wald test for testing the mean of a Gaussian}
\label{subsec: Wald}
Let $\Xvec \thicksim \Ncal(\xivec,\matcov)$ be a $\mydim$-dimensional Gaussian random vector, where the $\mydim \times \mydim$ matrix $\matcov$ is \spd. 
Consider the problem of testing whether $\Xvec$ is centered or not, that is, the problem of infering whether $\xivec = 0$ or not when we are given a realization of $\Xvec$. Equivalently, we also say that we test the null hypothesis $\Hcal_0$ that $\xivec = 0$ against its alternative $\Hcal_1$ that $\xivec \neq 0$. We summarize this problem as:
%following testing problem
\begin{equation}
	\label{Eq:ddtpb}
	\left \{
	\begin{array}{lll}
		\text{{Observation:}} \, \Xvec \thicksim \Ncal(\xivec,\matcov) \mycomma \vspace{0.1cm} \\
		\text{Hypotheses:} \,
		\left \{
		\begin{array}{lll}
			\! \! \Hcal_0: \xivec = 0 \mycomma \\
			\! \! \Hcal_1: \xivec \neq 0 \myfstop
		\end{array}
		\right.
	\end{array}
	\right.
\end{equation}
%This problem consists of testing the null hypothesis $\Hcal_0: \xivec = 0$ ($\Xvec$ is centered) against its alternative $\Hcal_1: \xivec \neq 0$ ($\Xvec$ is not centered). 
This corresponds to a binary hypothesis testing problem \citep{Lehmann2005}. To solve it, consider all the (measurable) functions $\test: \Rset^{\mydim} \to \{0,1\}$. Any such function is called a test. %Given a test $\test$, the value returned by the composite function $\test(\Xvec)$ is defined as the index of the accepted hypothesis when we observe $\Xvec$. Otherwise said, 
We say that test $\test$ accepts $\Hcal_0$ (resp. $\Hcal_1$) if $\test(\Xvec) = 0$ (resp. $\test(\Xvec) = 1$). The size of a test $\test$ for testing the mean of $\Xvec$ is $\Pbb [ \test(\Xvec) =1 ]$ when $\xivec = 0$. The size of $\test$ is thus the probability of rejecting $\Hcal_0$ when this hypothesis holds true. The power of $\test$ for testing the mean of $\Xvec$ is $\Pbb [ \test(\Xvec) =1 ]$ when $\xivec \neq 0$. Thus, it is the probability of accepting $\Hcal_1$ when this hypothesis is true.

Among all possible tests, we are interested by those whose size does not exceed some value $\level \in \Open{0}{1}$ for testing the mean of $\Xvec$. Such tests are said to have level $\level$. Given $\alpha \in \Open{0}{1}$, set 
\begin{equation}
	\label{Eq: Wald threshold}
	\TheThreshold{\level} = \sqrt{\Fbb^{-1}_{\chi^2_{\mydim}}(1-\alpha)} \mycomma
\end{equation} 
where $\Fbb_{\chi^2_{\mydim}}$ is the cumulative distribution function of the chi-squared distribution with $\mydim$ degrees of freedom. According to \citet[Definition III \& Proposition III, p.~450]{Wald1943}, the test defined for any $\xvec \in \Rset^{\mydim}$ by
\begin{equation}
	\label{Eq:Thresholding test from above}
	\Topt(\xvec) = \left \{
	\begin{array}{lll}
		0 & \hbox{ if } & \thenorm( \xvec ) \leqslant \TheThreshold{\level} \mycomma\\
		1 & \hbox{ if } & \thenorm( \xvec ) > \TheThreshold{\level} \myfstop
	\end{array}
	\right.
\end{equation}
has size $\level$ for the problem described by~\eqref{Eq:ddtpb}. This test is called the Wald test with size $\alpha$ for testing the mean of $\Xvec$. The Wald test $\Topt$ turns out to be optimal with respect to several optimality criteria and within several {subsets of tests with level $\level$} \citep[Definition III \& Proposition III, p.~450]{Wald1943}, \citep[Proposition 2]{RDT}. %$\ClassDet$ %In particular, it is UMP among all spherically invariant tests belonging to $\ClassDet$ and has Uniformly Best Constant Power (UBCP) on the spheres centered at the origin of $\Rset^{\dim}$ \citep[Definition III \& Proposition III, p. 450]{Wald1943}, \citep[Proposition 2]{RDT}. 
A full description of the optimal properties satisfied by the Wald test is unnecessary to introduce the Wald kernel. In contrast, the p-value of the Wald tests is key {to define our weight function}.

\subsection{{p-value of the Wald test}}
\label{subsec:pval of Wald's test}
%The p-value of the Wald test for testing the mean of a Gaussian is given as follows.
Let $\TestFamily_{\matcov}$ be the function that assigns the Wald test $\Topt$ to a given $\alpha \in \Open{0}{1}$. As detailed in Appendix \ref{Sec: RDT pvalue}, we can define the p-value \citep[p.~63, Sec.~3.3]{Lehmann2005}
$$\pval{\TestFamily_{\matcov}}{\Xvec} =  \inf \mybig \{ \level \in \Open{0}{1}: \Topt(\Xvec) = 1 \mybig \}$$ 
and prove that 
\begin{equation}\label{eq:RDT-pvalue}
%\forall \xvec \in \Rset^{\dim}, 
\pval{\TestFamily_{\matcov}}{\Xvec} =  1 - \Fbb_{\chi^2_{\mydim}}(\nu^2_{\matcov}(\Xvec)) %\myfstop %= \MyMarcum\left( \new{\thenorm}(\xvec)\right)
\end{equation}
for $\Xvec \thicksim \Ncal(\xivec,\matcov)$.The function $\text{pval}_{\TestFamily_{\matcov}}(\xvec) = 1 - \Fbb_{\chi^2_{\mydim}}(\nu^2_{\matcov}(\xvec))$ decreases with $\thenorm(\xvec)$ and tends to $0$ (resp. $1$) when $\thenorm(\xvec)$ increases (resp. decreases). Therefore, as mentioned in \citet[Sec.~3.3]{Lehmann2005}, $\text{pval}_{\TestFamily_{\matcov}}(\Xvec)$ can be seen as a measure of the plausibility of the null hypothesis $\Hcal_0$ when we observe $\Xvec$: when $\thenorm(\Xvec)$ is large (resp. small), $\Hcal_0$ is likely to be false (resp. true) and accordingly,  $\pval{\matcov}{\Xvec}$ is close to $0$ (resp. $1$).

%\vspace{-0.1cm}

\subsection{Definition and properties of the Wald kernel	}
\label{subsec: weight functions}

Consider a finite sequence $(\Obs_n)_{n \in \intn}$ of measurements and a sequence $\CtrMap$ of centroids that satisfy the Gaussian cluster model of Section \ref {subsec: signal model}. If $\CtrMap$ were known, we could test, for each pair $(k,n) \in \intk \times \intn$, whether $\Expect{\Obs_{n}}$ is $\CtrMap(k)$ or not. This would amount to testing whether $\Obs_{n}-\CtrMap(k)$ is centered or not. According to Section \ref{subsec:pval of Wald's test}, given $\alpha \in \Open{0}{1}$, we must use the Wald test $\test^\alpha_{\MCov_n}$ to perform this testing. The p-value $\pval{\TestFamily_{\MCov_n}}{\Obs_n - \CtrMap(k)}$ thus measures the plausibility that $\Expect{\Obs_{n}} = \CtrMap(k)$ when we observe $\Obs_n$, and we have:
\begin{equation}
	\pval{\TestFamily_{\MCov_n}}{\Obs_n - \CtrMap(k)} = 1 - \Fbb_{\chi^2_{\mydim}} \left ( \Maha{\MCov_n}^2(\Obs_{n} - \CtrMap(k) \right ) \myfstop
	\label{eq: pval for testing-1}
\end{equation}
\noindent
{Suppose that the matrices $\MCov_n$ are known. According to Remark \ref{rmk: fundamental remark}, we naturally set} $\estMCovn = \MCov_n$ for each $n \in \intn$ in $\TheFunc$ defined by \eqref{eq: new h (a) - 1}. {In this case, the weighting coefficients in $\TheFunc$ are the terms $w \left( \Maha{\MCov_n}^2(\Obs_{n} - \Ctr) \right)$ for $n \in \intn$. In particular, when $\Ctr = \CtrMap(k)$, it is desirable that the closer the weighting coefficient $w \left ( \Maha{\MCov_n}^2(\Obs_{n} - \CtrMap(k)) \right )$ is to $1$ (resp. $0$), the more likely $\Expect{\Obs_{n}} = \CtrMap(k)$ is to be true (resp. false).} Therefore, we choose $w$ such that 
$$w \left ( \Maha{\MCov_n}^2(\Obs_{n} - \CtrMap(k) \right ) = \pval{\TestFamily_{\MCov_n}}{\Obs_{n} - \CtrMap(k)} \myfstop$$ 
By identification with \eqref{eq: pval for testing-1}, we thus define the wald Kernel $w: \RightOpen{0}{\infty} \to \Open{0}{\infty}$ for any {real value} $t \geqslant 0$ by setting 
\begin{equation}
	\label{eq: Wald kernel}
	w(t) =  1 - \Fbb_{\chi^2_{\mydim}}(t) \myfstop
\end{equation}
\indent 
The next lemma shows that the Wald kernel satisfies the required assumptions of Theorem~\ref{Theorem: fixed points} and Proposition~\ref{Proposition:the iterative method converges}.
\begin{Lemma}\label{Lemma:weight_function}
The Wald kernel defined by \eqref{eq: Wald kernel} is differentiable, positive, non-increasing, and verifies $\lim\limits_{t \to \infty} t w(t) = 0$. 
\end{Lemma}
\begin{proof}
The properties of $\Fbb_{\chi^2_{\mydim}}$ imply directly that the Wald kernel is positive, differentiable and decreasing. Now, it follows from the first equality in \citet[Eq.~(11), p.~1168]{Sun2010} that $\forall t \in \RightOpen{0}{\infty}$, $w(t) = \Marcum(0,\sqrt{t})$ where, {for any pair $(a,b)$ of non-negative real values}, $\Marcum(a,b)$ is the {\em generalized Marcum $Q$-function of order $\mydim/2$.} According to the second equality in \citet[Eq.~(11), p.~1168]{Sun2010}, we have 
$$\forall t \in \RightOpen{0}{\infty}, w(t) = \frac{1}{2^{\mydim/2} \Gamma(\mydim/2)} \int_{t}^\infty x^{\mydim/2-1} e^{-x/2} \dx \mycomma$$
where $\Gamma$ is the standard Gamma function \citet[Eq.~8.310.1, p.~892]{GR2007}. We make the change of variable $y = x/2$ in the integral above to get
$$\forall t \in \RightOpen{0}{\infty}, w(t) = \dfrac{1}{\Gamma(\mydim/2)} \displaystyle \int_{t/2}^\infty y^{\mydim/2-1} e^{-y} \dy = \dfrac{\Gamma(\mydim/2,t/2)}{\Gamma(\mydim/2)} \mycomma$$
where the second equality derives from \citet[Eq.~8.350.2, p.~899]{GR2007}. Thus, from \citet[Eq.~8.357.1, p.~902]{GR2007}, we obtain $\lim\limits_{t \to \infty} t w(t) = 0$.
\end{proof}

\section{Clustering Algorithm}
\label{sec:algo}
We now introduce the algorithm~\algo~that performs clustering over $N$ measurement vectors $(\Obs_n)_{n \in \llbracket 1,N \rrbracket}$, without prior knowledge of the number of clusters. This algorithm works in two stages: it first estimates the cluster centroids, and then assigns each measurement vector to a cluster. 
As discussed in Section~\ref{sec:robust}, the cluster centroids can be estimated by looking for the fixed points of the function ${\Thefuncbasic}$ (see \eqref{eq: our general function}), given a realization $(\obs_1, \ldots, \obs_N)$ of $(\Obs_1, \ldots, \Obs_N)$. 

\indent
The Mean-Shift algorithm~\citep{comaniciu2002mean} also estimates the cluster centroids by looking for the fixed-points of ${\Thefuncbasic}$. To do so, it runs $N$ times the iterative procedure $\Ctr^{(\ell+1)} = \Thefuncbasic(\Ctr^{(\ell)})$, where each measurement vector $\obs_n$ serves as an initializer once, by setting $\Ctr^{(1)} = \obs_n$. The Mean-Shift algorithm thus outputs $N$ candidate centroids. It then performs an additional step to fuse the candidate centroids that are close enough to each other, thus resulting in a number of centroids near $K$. Although Mean-Shift works well in practice, it becomes very complex when $N$ increases. Alternatively, in \algo, we rely on the Wald hypothesis test {described in} Section~\ref{sec:w_function} to eliminate, from the set of potential initializers, all the measurement vectors that are close enough to an already estimated centroid, thus leading to a much less complex algorithm. Since we observed that this method often produces a number of centroids slightly larger than $K$, we still consider the fusion step of Mean-Shift in order to end-up with the correct number of clusters. 

In the next subsection, we first describe \algo~for any given \spdms~$\estMCovone$ \dots $\estMCovN$. Next, in the case of matrices $\MCov_{1}, \ldots, \MCov_N$ that are equal and proportional to $\Id$ with unknown proportionality factor {$\sigma^2$, we propose an MLE $\sigmaMLEsq$ of $\sigma^2$ in Section \ref{sec:est_sigma} to tweak \algo~with $\estMCovn = \sigmaMLE^2 \, \Id$ for each $n \in \intn$. Empirical solutions for estimating unknown matrices $\MCovn$, possibly different from each other, are also considered in Section~\ref{sec:varying_sigma}.

\subsection{\algo~algorithm}
\label{subsec:Algorithm description}
The first step of \algo~is to return an estimate $\widehat{\CtrSet}$ of the set of centroids $\CtrSet$ of $(\obs_n)_{n \in \intn}$. To do so, \algo~initializes $\widehat{\CtrSet}$ as the {empty set} $\emptyset$ and then estimates the centroids one after the other by relying on the iterative fixed-point computation method described in Section~\ref{sec:fixed_point_est}. {Each time this method is applied, one initializer is picked within the set $\unmarked$ of indices of \emph{unmarked} vectors: the marked vectors are all the initializers used so far and all the vectors close enough, in a sense precised below, to previously estimated centroids. Therefore,} $\unmarked$ is initialized as $\unmarked = \intn$. The centroid estimation then comprises three steps (pick, estimate, mark) to be repeated until $\unmarked$ is empty. Once the centroid estimation is complete, a fusion algorithm is applied so as to retrieve the correct number of clusters. 
Finally, the cluster assignment is performed conventionally by assigning to each measurement the closest estimated centroid. It is completed by removing centroids to which no data have been assigned and have thus produced empty clusters. We outline the algorithm as follows:

\clearpage
%\vspace{0.1cm}

\begin{description}
	\item[Initialization:] %\hspace{1cm} 
	$\unmarked = \intn$, $\EstCtrSet = \emptyset$
    %\vspace{0.1cm}
	\item[Centroid estimation:]
		\item[\quad WHILE] %\hspace{0.5cm} 
		$\unmarked \neq \emptyset$
		\item[\qquad Step \#1:] %\hspace{0.9cm} 
		\textbf{[Pick]} at random an index $n^\star$ in the set $\unmarked$. % and set $\obs^{\star} = \obs_{n^\star}$.
		\item[\qquad Step \#2:] %\hspace{0.9cm} 
		\textbf{[Estimate]} a new centroid $\Ctr$ from the fixed-point estimation process described in Section~\ref{sec:fixed_point_est}, initialized with $\obs^{\star}$. Then update $\widehat{\CtrSet}$ as $\EstCtrSet \leftarrow \EstCtrSet \cup \{\Ctr\}$. 
		\item[\qquad Step \#3:] %\hspace{0.9cm} 
		\textbf{[Mark]} the measurement vectors that are close enough to $\Ctr$ to consider that they belong to the cluster with centroid $\Ctr$. Define $\mathcal{V}$ as the set of indices of these vectors and update ${\mathcal U}$ as $\unmarked \leftarrow \unmarked \setminus \left ( {\mathcal V} \cup \{n^\star\} \right )$.
    \item[Fusion:] 
    \item[\quad \textbf{[Merge]}] %\hspace{0.6cm} 
    all centroids $\Ctr, \Ctr' \in \EstCtrSet$ such that $\| \Ctr - \Ctr' \| / \mydim \leqslant \epsilon_f$
	\item[Cluster assignment:]
	\item[\quad FOR] $n \in \intn$
	\item[\qquad] \textbf{[Assign]} $\obs_n$ to the closest $\Ctr \in \EstCtrSet$
    \item[\quad \textbf{[Remove]} empty clusters]
%\end{itemize}
\end{description}

The algorithm does not need prior knowledge of the number $K$ of centroids, since its stopping condition bears on the emptiness of ${\mathcal U}$. The normalization by $\mydim$ in \textbf{[Merge]} to make have $\epsilon_f$ independent of the data dimension. 
% \red{\sout{We now provide a full description of the successive steps of \algo, and discuss key practical issues.}}
We now describe into more details the centroid estimation, the cluster assignment and the fusion. We also introduce additional heuristics.

\begin{Remark}
\label{Remark: preliminary transform of the measurement vectors}
According to $\myspadesuit$ in Theorem \ref{Theorem: fixed points}, instead of performing the clustering on the basis of $(\obs_n)_{n \in \intn}$, the clustering can alternatively be performed as follows:
\begin{enumerate}
	\item Transform the initial data into $\oobs_n = \Rmat^\transpose \obs_n$ for $n \in \intn$,
	\item Calculate $\EstCtrSet$ by running \algo~where the input vectors of the centroid estimation are now $(\oobs_n)_{n \in \llbracket 1,N\rrbracket}$ and the covariance matrices used to perform the centroid estimation are $\estDiagmat = (\estDiagmatn)_{n \in \llbracket 1,N\rrbracket}$ defined according to \eqref{eq: eigen decomposition of estMCov}, the empirical parameters $\epsilon_e$, $L$ and $\alpha$ remaining unchanged.
	\item Estimate the centroids by calculating $\Rmat \cctr_1, \Rmat \cctr_2, \ldots$ where $\cctr_1, \cctr_2, \ldots$ are the elements of $\EstCtrSet$ returned by \algo.
\end{enumerate}
\end{Remark}

\subsubsection{Centroid estimation}

The centroid estimation part of \algo~is fully stated in Algorithm~\ref{algo:centrex}. 
Among the four steps needed to estimate the centroids, the step \textbf{[Pick]} is taken as it is: it randomly picks $n^\star \in \unmarked$. 

To \textbf{[Estimate]}, the recursive fixed-point estimation process is initialized with $\Ctr^{(1)} = \initobs$. We calculate $\Ctr^{(2)} = \Thefuncinit(\initobs)$ where $\Thefuncinit$ is $\Thefuncbasic$ parameterized with $\estMCovn + \estMCovnstar$ instead of $\estMCovn$ for each $n \in \intn$. Indeed, for the sequence $\CtrMap$ of centroids we want to estimate, $\initobs$ is a realization of $\initObs \thicksim \Ncal(\CtrMap(\initseq), \MCov_{n^\star})$. For any $n \neq n^\star$, we thus have $(\Obs_n - \initObs) \thicksim \Ncal(\CtrMap(\seq_n) - \CtrMap(\initseq), \MCovn + \MCovnstar)$. 

For $\ell \geq 1$, we compute $\Ctr^{(\ell+1)} = {\Thefuncbasic}(\Ctr^{(\ell)})$ until $\Maha{\Qmat} \big ( \Ctr^{(\ell+1)} - \Ctr^{(\ell)} \big ) / \mydim \leq \epsilon_e$ or $L$ iterations have been performed, where $\Maha{\Qmat}$ is the Mahalanobis norm associated with the \spdm~$\Qmat = \frac{1}{N} \sum_{n=1} ^N \estMCovn$. The 
normalisation by $\mydim$ and the use of $\Maha{\Qmat}$ in the stopping criterion aim to temper the influence of $\mydim$ and the \spdms~$\estMCovn$ on the empirical parameter $\epsilon_e$. Note that $L$ is also an empirical parameter of the algorithm. In addition, the choice of considering $\Thefuncinit$ instead of $\Thefuncbasic$ at the first iteration is a heuristic that was shown to improve the clustering performance in our simulations. 

Finally, the \textbf{[Mark]} step discards from $\mathcal{U}$ the indices of the vectors $\obs_n$ that are sufficiently close to the newly estimated centroid to consider them as elements of the same cluster. If ${\Ctr}$ denotes this estimated centroid then, for a given realization $\obs_n$ of $\Obs_n \thicksim \Ncal(\CtrMap(\seq_n), \MCovn)$, the marking amounts to testing the hypothesis that $\CtrMap(\seq_n) = {\Ctr}$, which is the condition for $\obs_n$ to belong to the cluster represented by $\Ctr$. This testing is performed via the Wald test from \eqref{Eq:Thresholding test from above} with $\matcov = \estMCovn$ and where $\level$ is empirically chosen to calculate the threshold value $\TheThreshold{\level}$ according to~\eqref{Eq: Wald threshold}. If $\MCovn$ is known, we set $\estMCovn = \MCovn$ and the test is the optimal Wald test to make our decision; otherwise, $\estMCovn$ is an estimate of $\MCovn$ and we cannot but expect that the use of this estimate in the expression of the Wald test will not affect too much the performance of the test. In this regard, the experimental results of Section \ref{sec:experiments} illustrate the robustness of \algo~to various estimates $\estMCovn$ of the matrices $\MCovn$.

\clearpage 

\begin{algorithm}[h]
	\caption{\begin{center} Centroid estimation in \algo~\\ (\textbf{Centroid\_Estim}) \end{center}}
	\begin{algorithmic} \label{algo:centrex}
		\STATE \textbf{Inputs}: ${\obs} = (\obs_n)_{n \in \llbracket 1,N\rrbracket}$, $\estMCov = (\estMCovn)_{n \in \llbracket 1,N\rrbracket}$, $\epsilon_e$, $L$, $\alpha$ \myvspace
		\STATE \textbf{Initialization}: $\EstCtrSet \leftarrow \{ \emptyset \}$, $\mathcal{U} \leftarrow \llbracket 1,N\rrbracket$ \myvspace
		\WHILE{ $\left(\mathcal{U} \neq \{ \emptyset \}\right)$ } \myvspace
		\STATE \#\# \textbf{[Pick]} \#\# \myvspace
		\STATE Pick randomly an $n^{\star}$ in $\mathcal{U}$ \myvspace
		\STATE \#\# \textbf{[Estimate]} \#\# \myvspace
		\STATE Initialize $\Ctr^{(1)} = \obs_{n^{\star}}$, $\Ctr^{(2)} = \Thefuncinit(\Ctr^{(1)})$, $\ell=2$ \myvspace
        \STATE $\Qmat \leftarrow \frac{1}{N} \sum_{n=1} ^N \estMCovn$ \myvspace
        \WHILE{ $\Big(\big( \, \Maha{\Qmat} \big ( \Ctr^{(\ell+1)} - \Ctr^{(\ell)} \big ) / \mydim \geq \epsilon_e \, \big) \text{ AND } \big( \ell+1 < L \big)\Big)$ } \myvspace
		\STATE {$\ell \leftarrow \ell+1$} \myvspace
		\STATE Compute $\Ctr^{(\ell+1)} = \Thefuncbasic(\Ctr^{(\ell)}) $ \myvspace \\
        \#\# \textbf{Comment}: $\Thefuncbasic$ is given by \eqref{eq: our general function} and is parameterized by $\estMCov$ and $\obs$ \myvspace
		\ENDWHILE \myvspace
		\STATE $\EstCtrSet \leftarrow \EstCtrSet \cup \left\{ \Ctr^{(\ell)} \right\}$ \myvspace
		\STATE \#\# \textbf{[Mark]}  \#\# \myvspace
		\FOR{$n$ in $\mathcal{U}$} \myvspace
		% \IF{$\MCovn$ is known} \myvspace
		% \STATE Set $\matcov = \MCovn$ \myvspace
		% \ELSE \myvspace
		% \STATE Set $\matcov = \estMCovn$ \myvspace
		% \ENDIF \myvspace
		\IF{$ \big ( \Maha{\estMCovn}(\obs_n - \Ctr^{(\ell)}) \leq \TheThreshold{\level}  \big ) $  or $\big ( n = n^\star \big )$} \myvspace
		\STATE $\mathcal{U} \leftarrow \mathcal{U} \setminus \{ n \}$ \myvspace
		\ENDIF \myvspace
		\ENDFOR \myvspace
		\ENDWHILE \myvspace
		\STATE \textbf{Outputs}: $\EstCtrSet$ \myvspace
	\end{algorithmic}
\end{algorithm}

\clearpage

\subsubsection{Fusion}
In our simulations, we observed that the centroid estimation performed by \algo~may output too many centroids. Therefore, to improve the overall clustering performance, we use a second algorithm, called \emph{fusion}, to be applied to the output set $\EstCtrSet$ provided by \algo. The fusion algorithm is described in Algorithm~\ref{algo:fusion}. This algorithm first identifies the two centroids $(\Ctr, \Ctr') \! \in \! \EstCtrSet$ that are the closest to each other, among all estimated centroids in $\EstCtrSet$. It then merges these two centroids if $\| \Ctr - \Ctr' \| / \mydim \leqslant \epsilon_f$, where the division by $\mydim$ significantly reduces the dependence of $\epsilon_f$ with the data dimension. The merge is realized by removing $\Ctr$ and $\Ctr'$ from $\EstCtrSet$, and replacing them with the average centroid $\frac{\Ctr+\Ctr'}{2}$. The fusion is repeated until the closest two centroids in $\EstCtrSet$ have distance larger than $\epsilon_f$. {At the end, the cluster assignment step performed after the last merge provides the output of the clustering algorithm.} %Practical values for $\epsilon_f$ will be discussed in Section \ref{sec:experiments}. 
This fusion step is identical to the additional step for merging the centroids in Mean-Shift. \medskip \\
\begin{algorithm}[h]
	\caption{ \begin{center} Fusion in \algo \\ \textbf{(Fusion)} \end{center}}
	\begin{algorithmic}\label{algo:fusion}
		\STATE \textbf{Inputs}: $\EstCtrSet$, $\epsilon_f$ %, \revoir{$\obs = (\obs)_{n \in \intn}$, $\estMCov = (\estMCovn)_{n \in \intn}$} \myvspace
		%\#\# \newnew{See notation in Algorithm \ref{algo:centrex}} \#\# \myvspace
		\STATE \textbf{Initialization}: $ \Delta = 0$ \myvspace
		\WHILE{\Big($\Delta <\epsilon_f $ AND $\card(\EstCtrSet)>1$\Big)} \myvspace
		\STATE Compute $(\Ctr,\Ctr') = \displaystyle \argmin_{(\xvec,\xvec') \in \EstCtrSet \times \EstCtrSet, x\neq x'} \, \, \| \xvec - \xvec' \|$ \\
		{$\Delta \leftarrow \| \Ctr - \Ctr' \| / \mydim$} \myvspace \\
		\IF{ ${\Delta} <\epsilon_f$} \myvspace
		\STATE \#\# \textbf{[Merge]} \#\# \myvspace
		\STATE Update $\hspace{0.15cm} {\EstCtrSet} \leftarrow \EstCtrSet \setminus \{ \Ctr,\Ctr' \}$, $\EstCtrSet \leftarrow \EstCtrSet \cup \{ \frac{\Ctr+\Ctr'}{2} \}$ \myvspace
		\ENDIF \myvspace
		\ENDWHILE \myvspace
		\STATE \textbf{Outputs}: $\EstCtrSet$ 
	\end{algorithmic}
\end{algorithm}

\subsubsection{Cluster assignment}
Once all vectors $\obs_n$ have been marked, we use the set $\EstCtrSet$ output by the fusion to perform cluster assignment. The cluster assignment step of \algo{} is fully stated in Algorithm~\ref{algo:assignment}. 
In this algorithm, the cluster assignment is realized by selecting for each $\obs_n$ the closest centroid $\Ctr \! \in \! \EstCtrSet$.

We could think of carrying the cluster assignment by resorting again to the hypothesis test of~\eqref{Eq:Thresholding test from above} used for the marking operation. However, this test could leave some data with no cluster assignment or, to the contrary, assign some data to several different clusters. This is why here, in order to avoid any underdetermination, we consider the minimum distance condition where we use the Mahalanobis norm $\Maha{\estMCovn}$. After this cluster assignment step, centroids that lead to empty clusters are removed from the set $\EstCtrSet$.

\begin{algorithm}[h]
 \caption{ \begin{center} Cluster assignment in \algo \\ \textbf{(Cluster\_Assignment)} \end{center}}
  \begin{algorithmic}\label{algo:assignment}
   \STATE \textbf{Inputs}: $\obs = (\obs_n)_{n \in \llbracket 1,N\rrbracket}$, $\estMCov = (\estMCovn)_{n \in \llbracket 1,N\rrbracket}$, $\EstCtrSet$ 
   \myvspace 
  \FOR{$n$ in $\llbracket 1,N \rrbracket$ } \myvspace
    \STATE \#\# \textbf{[Assign]} \#\# \myvspace
    \STATE Set $\widehat{\Ctr}_n = \displaystyle \arg\min_{\Ctr \in \EstCtrSet} \Maha{\estMCovn} ( \obs_n - \Ctr )$ \myvspace
    \STATE Set $\assignfct(n) = \widehat{\Ctr}_n$
  \ENDFOR \myvspace \\
  \#\# \textbf{[Remove] empty clusters} \#\# \myvspace
  \FOR{$\Ctr \in \EstCtrSet$} \myvspace
    \IF{$\{ n \in \intn: \widehat{\Ctr}_n = \Ctr \} = \emptyset$} \myvspace
      \STATE $\EstCtrSet \leftarrow \EstCtrSet \setminus \{\Ctr\}$ \myvspace
    \ENDIF \myvspace
  \ENDFOR \myvspace
  \STATE \textbf{Outputs}: $\EstCtrSet$, $\assignfct$
  \end{algorithmic}
\end{algorithm}

\subsection{The case of proportional covariance matrices}\label{sec:est_sigma}

When the \spdms~$\MCovn$ are unknown, it is crucial in practice to estimate them with sufficient precision. Unfortunately, (As. 3) encompasses a large class of \spdms~(diagonal, non-diagonal, heteroscedastic or not, etc.), which makes it difficult to devise a single method to estimate them.

For the special case where $\MCovn = {\sigma^2} \, \MCov$ for all $n \in \intn$, {with known $\MCov$ and unknown $\sigma$}, we propose the following approach to estimate $\sigma$ from the sequence of {data $(\obs_n)_{n \in \llbracket 1,N\rrbracket}$}. 
We select $P$ different {data} at random among {$(\obs_1, \ldots, \obs_N)$}, and without loss of generality, re-index them as {$(\obs_1, \ldots, \obs_P)$}. 
We then evaluate the minimum distance $\statest$ among these $P$ {data} as 
\begin{equation}\label{eq:min_dist}
\statest = \displaystyle \min_{\begin{array}{ccc} \text{\scriptsize $(p,p') \in \intP \times \intP$} \vspace{-0.1cm} \\
		\text{\scriptsize $p < p'$} \end{array}} \Mahasq{\MCov}( \obs_p- \obs_{p'}) \myfstop
\end{equation}
% \begin{equation}\label{eq:min_dist}
% \Statest = \displaystyle \min_{\begin{array}{ccc} \text{\scriptsize $(p,p') \in \{1, \ldots, P\} \times \{1, \ldots, P \}$} \\
% 		\text{\scriptsize $p < p'$} \end{array}} \nu_{\MCov}^2( \Obs_p- \Obs_{p'}) .
% \end{equation}
For a given $(p,p') \in \intP \times \intP$ with $p < p'$ such that $\obs_p$ and $\obs_{p'}$ belong to the same cluster, $(\obs_p-\obs_{p'})$ is a realization of $(\Obs_p - \Obs_{p'}) \thicksim \Ncal(0,2{\sigma^2} \MCov)$. In this case, it follows from \eqref{eq: definition of the initial mahanorm} and \eqref{eq: eigen decomposition of MCov init} that $\Mahasq{\MCov}(\Obs_p - \Obs_{p'}) / 2 \sigma^2 \thicksim \chi_{\mydim}^2$, where $\chi_{\mydim}^2$ is the Chi-squared distribution with ${\mydim}$ degrees of freedom, and thus, that $\Mahasq{\MCov}(\Obs_p - \Obs_{p'})$ has cdf $\Fbb$ defined by
\begin{equation}
    \label{eq: Fsigma}
    \forall x \in \RightOpen{0}{\infty}, \Fbb(x) = \Fbb_{\chi^2_{\mydim}}(x/2\sigma^2) \myfstop
\end{equation}
\eqref{eq:min_dist} considers $P(P-1)/2$ different pairs $(y_p,y_{p'})$. In what follows, we assume that $M$ of these pairs are such that both $y_p$ and $y_{p'}$ belong to the same cluster, this cluster being possibly different from one pair to another: given two different pairs $(y_p,y_{p'})$ and $(y_q,y_{q'})$ among these $M$ pairs, both $y_p$ and $y_{p'}$ may belong to the same cluster, say cluster $\#k$, whereas both $y_q$ and $y_{q'}$ may belong to cluster $\#k'$ with $k' \neq k$. According to this assumption, we consider that $\statest$ is also the value taken by the minimum $\min(\Upsilon_1, \ldots,\Upsilon_M)$ of $M$ random variables $\Upsilon_1, \ldots, \Upsilon_{M} \stackrel{\text{iid}}{\thicksim} \Fbb$ with $M \leqslant P(P-1)/2$.
\\
\indent
The cdf of the minimum of a finite sequence of random variables being well-known (see \citealt[Sec.~2.4, p.~87]{Serfling1980}, among others), we calculate the cdf of $\min(\Upsilon_1, \ldots,\Upsilon_M)$ as
\begin{equation}
\forall x \in \RightOpen{0}{\infty}, \Fbb_{\min(\Upsilon_1, \ldots,\Upsilon_M)}(x) = 1 - \left ( 1 - \Fbb \left ( {x} \right ) \right )^{M} = 1 - \left ( 1 - \Fbb_{\chi^2_{\mydim}} \left ( \frac{x}{2 \sigma^2} \right ) \right )^{M} \myfstop
\label{eq: cdf of minimum of chi2}
\nonumber
\end{equation}
A pdf of $\min(\Upsilon_1, \ldots,\Upsilon_M)$ is thus $f_{\min(\Upsilon_1, \ldots,\Upsilon_M)}(x) = f(x, \sigma)$ for all $x \in \RightOpen{0}{\infty}$, with
\begin{equation}
    \label{eq: fsigma}
    \forall (x,t) \in \RightOpen{0}{\infty} \times \Open{0}{\infty}, f(x, t) = \textstyle \frac{M}{2{t^2}} f_{\chi^2_{\mydim}} (\frac{x}{2 {t^2}})\left ( 1 - \Fbb_{\chi^2_{\mydim}} \left ( \frac{x}{2 {t^2}} \right ) \right )^{M-1} \mycomma
\end{equation}
where $f_{\chi^2_{\mydim}}$ denotes the pdf of the $\chi_{\mydim}^2$ distribution. Our MLE $\sigmaMLE$ is hence defined as 
\begin{equation}
    \label{eq: sigmaMLE}
    \sigmaMLE = \argmax_{t \in \Open{0}{\infty}} f(\upsilon,t) \myfstop
\end{equation}
Having no closed form for $\sigmaMLE$ such defined, $\sigmaMLE$ is calculated numerically. We then set 
\begin{equation}
    \label{eq: final estimate of MCov in the MLE case}
    \forall n \in \intn, \estMCovn = \sigmaMLEsq \, \MCov \myfstop
\end{equation}

\section{Simulations results}
\label{sec:experiments}
In this section, we evaluate the proposed algorithm \algo{} through numerical simulations. We first show clustering results on $2$D toy data sets, and then provide extensive simulation results on synthetic measurement vectors with higher dimension $\mydim=100$. At the end, we apply \algo{} to the standard data sets Ruspini and IRIS. All simulations described below were performed using MATLAB 2024.

\subsection{Two-dimensional toy data sets}
\label{subsec:Two-dimensional toy datasets}

\begin{figure}[t]
	\includegraphics[width=1\linewidth]{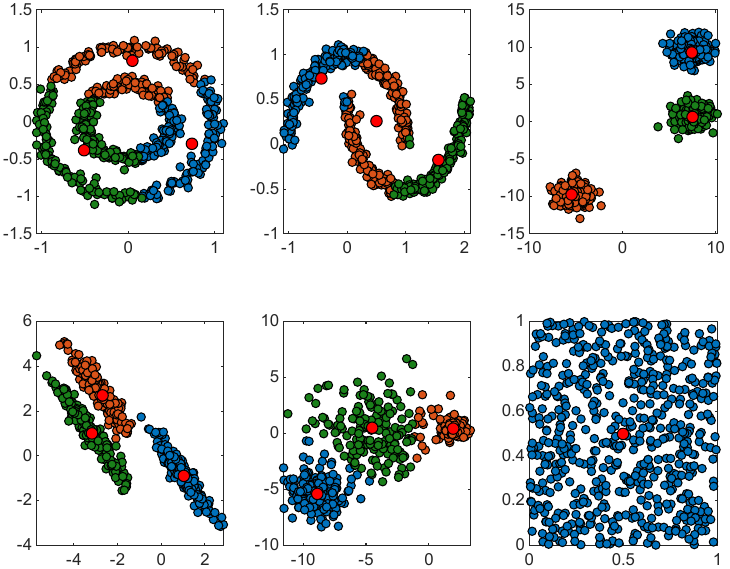}
	\caption{\small Centre-X applied on the 2D toy data sets provided in the scikit-learn documentation. The red dots are the centroids estimated by \algo{}.}	
	\label{fig:2D_centreX}
\end{figure}

In the scikit-learn documentation\footnotemark[1]\footnotetext[1]{\url{https://scikit-learn.org/stable/modules/clustering.html}}, standard clustering algorithms  (K-means, DB-SCAN, Mean-Shift, etc.)~are applied to six $2$D toy data sets. We first apply our algorithm \algo{} to these six data sets, in order to check whether the behavior of \algo{} is as expected. %The purpose of this evaluation is to check whether the behavior of~\algo{} is as expected. 
We will benchmark our algorithm {against} other methods later on in this section. We set $L=100$ maximum iterations for centroid estimation, and set the empirical parameters as $\epsilon_e = 10^{-3}$, $\epsilon_f = 0.5$, $\alpha = 10^{-3}$. For each data set, we provide the exact data covariance matrices to our algorithm (their values can be inferred from the provided Python code in the scikit-learn library), in order to observe the performance of \algo{} in its best configuration. In the same way, {all} clustering algorithms considered in the scikit-learn evaluation\footnotemark[1]  were provided with correct input parameters (for instance, K-means was given the correct value of $K$, etc.).

The clustering results obtained by \algo{} on the six data sets are shown in Figure~\ref{fig:2D_centreX}. On the three data sets that consider Gaussian data (top right, bottom left, bottom middle), we observe that \algo{} was able to correctly retrieve $K=3$ clusters of correct shapes. This result was expected, since these three data sets follow the Gaussian model of~\algo{}. On the opposite, \algo{} was not able to correctly cluster two of the other data sets (top left, top middle) since these data sets do not follow the Gaussian model. Finally, \algo{} retrieved one cluster on the last data set (bottom right), for which there is no obvious clustering solution. 

\subsection{\texorpdfstring{Measurement vectors with dimension $\mydim = 100$}{Measurement vectors with dimension d = 100}}
\label{subsec: masurement vectors}

We now extensively evaluate the performance of \algo{} on synthetic data sets with higher dimension ${\mydim} = 100$. Simulations on synthetic data are useful to evaluate the performance and robustness of \algo{}, depending on whether the \spdms~$\estMCovn$ equal the {true} covariance matrices $\MCov_n$, or not. They also permit to assess the relevance of the {proposed} Wald kernel. 

%In this section, we first specify the clustering algorithms against which we benchmark \algo{}. We then outline our simulations on synthetic data. Afterwards, we describe the synthetic data generation and provide all the simulation parameters. We then show the simulation results.

\subsubsection{Data generation}
\label{sec: data generation}

%The  measurement vectors {are} generated randomly, according to the model of Section \ref{sec: measurement vector models}. We employ diagonal matrices $\MCov_n$, and thus diagonal matrices $\estMCovn$. This is equivalent to considering the case where $\Rmat = \Identity{\mydim}$ in \eqref{eq: eigen decompositions} or, in the vein of Remarks \ref{rmk: transforms} and \ref{Remark: preliminary transform of the measurement vectors}, to considering that data are issued from the transform $\fn{\Rmat^\transpose}$, which consists of multiplying the initial measurement vectors by $\Rmat^\transpose$.

At the start of each experiment, we choose an interval $[\sigmamin,\sigmamax] \subset \Open{0}{\infty}$ to generate the measurements. Given $[\sigmamin,\sigmamax]$, we generate $800$ data sets. Each of these data sets is obtained as follows.

First, the number $K$ of centroids is randomly drawn in $ \llbracket 2,10 \rrbracket$ according to the discrete uniform law on this interval. Second, once $K$ is determined, we randomly generate $K$ centroids in $\Rset^\mydim$ by drawing $K$ realizations of a centred $\mydim$-dimensional Gaussian random vector with covariance matrix $A^2 \Id$ with $A = 20$, until the minimum distance between two different centroids exceeds $\mu_{\min} = 200$. The choice of the maximum number of centroids, that of the value of $\mu_{\min}$ and that of $A$, are arbitrary. Finally, when the $K$ centroids have been calculated, we generate $N = {400}$ {realizations $\obs_1, \ldots, \obs_N$} of measurement vectors $\Obs_1, \ldots, \Obs_N$. Each of these measurement vectors has mean randomly chosen among the $K$ available centroids, via the discrete uniform law on $\intk$, and satisfies the model of Section \ref{sec: measurement vector models} and \cref{Assumption: on C}, with $\Rmat = \Identity{\mydim}$. This is equivalent to considering the case where $\Rmat = \Identity{\mydim}$ in \eqref{eq: eigen decompositions} or, in the vein of Remarks \ref{rmk: transforms} and \ref{Remark: preliminary transform of the measurement vectors}, to considering that data are issued from the transform $\fn{\Rmat^\transpose}$, which consists of multiplying the initial measurement vectors by $\Rmat^\transpose$.

\subsubsection{Clustering algorithms and brief outline of the simulations}
\label{subsec: clustering algorithms}
Beside \algo, we systematically evaluate the K-means algorithm in our experiments, since this algorithm is widely used for clustering. To mitigate initialization issues with K-means, we actually use the K-means++ algorithm with $10$ random initializations. In addition, for K-means, we consider both the setup where $K$ is known and the one where $K$ is unknown. The latter case is referred to as the X-means algorithm, and consists in relaunching K-means with all possible values of $K$ in a certain interval, and select the $K$ that maximizes the Silhouette coefficient~\citep{e23060759}. X-means will search $K$ in the interval $\llbracket 1,10 \rrbracket $.

We will also assess the performance of Mean-Shift with the standard Gaussian kernel, the flat kernel, and the proposed Wald kernel. The Mean-Shift code in the scikit-learn library is implemented with the flat kernel only. We thus use this library to perform Mean-Shift with flat kernel. To perform Mean-Shift with either the Gaussian or the Wald kernel, we implemented the Mean-Shift algorithm as follows. According to \eqref{eq: MyNewfunc}, we particularized the mean-shift function of \algo{} to the \emph{mean-shift} of~\citet[Eq.~17]{comaniciu2002mean}; next, according to Section \ref{subsec:Algorithm description}, Mean-Shift amounts to performing \textbf{[Estimate]} for each of the $N$ input data and, skipping the marking, directly merging the candidate centroids obtained after these $N$ runs. For the fusion step in Mean-Shift, we use $\epsilon_f = 1$, as will be considered for~\algo{}.

Several empirical parameters must be provided to \algo{}. We set $\alpha = 10^{-3}$ since we observed that, most often, any value of $\alpha$ equal or lower than $10^{-2}$ leads to the same clustering performance. In the same way, the parameters $\epsilon_e$ and $L$ do not seem to influence much the clustering performance, and we set $\epsilon_e = 10^{-3}$ and $L = 100$. On the other hand, we observed that the fusion parameter $\epsilon_f$ has a greater impact on the clustering performance. Given that it is used for the fusion step, we propose to set $\epsilon_f = \mu_{\mymin}/2 \mydim$, since $\mu_\mymin/2$ is the distance at which a vector in $\Rset^\mydim$ cannot unambiguously be assigned to any two centroids. Therefore, for $\mydim = 100$ and $\mu_\mymin = 200$, we set $\epsilon_f = 1$. 

As an introduction to their detailed descriptions below, we briefly outline our simulations. In our first experiment (see Figure \ref{fig:CentreX_vs_MeanShift}), we consider the homoscedastic case of matrices $\MCov_{n}$ all equal to $\sigma^2 \Identity{\mydim}$. Because of their similarities, we compare \algo~to Mean-Shift when both employ the same kernel, either the Gaussian kernel or the Wald kernel, and both are provided with the value of $\sigma$. In Figures \ref{fig:clustering_robustness}, \ref{fig:clustering_robustness_varying-fig4} and \ref{fig:clustering_robustness_varying-fig6}, we benchmark \algo~with K-means++ and X-means to evaluate the robustness of these three algorithms. In Figure \ref{fig:clustering_robustness}, we consider {the homoscedastic case of} matrices $\MCov_{n}$ all equal to $\sigma^2 \Identity{\mydim}$, where $\sigma$ is now unknown and estimated by the MLE of Section \ref{sec:est_sigma}. In Figures \ref{fig:clustering_robustness_varying-fig4} and \ref{fig:clustering_robustness_varying-fig6}, we address two different scenarios for generating heteroscedastic diagonal matrices $\MCov_{n}$ that are not identity-scaled. The purpose is to assess the performance of \algo~tweaked with either the exact knowledge of these matrices or different estimates of them.

\subsubsection{{Performance evaluation criteria}}
\label{subsec: performance evaluation criteria}

{In the figures displayed below, the four criteria used to evaluate the performance of each benchmarked algorithm are: 1) the proportion of correctly estimated values of $K$, 2) the Silhouette coefficient, 3) the clustering error rate,  4) the average estimated value of $K$.} 

The clustering error rate is defined as follows. Given $N$ data $\obs_1, \ldots, \obs_N$, we define the upper-triangular adjacency matrix as
$$\Amat_{\obs_1, \ldots, \obs_N} = (a_{i,j})_{1 \leqslant i \leqslant N, 1 \leqslant j \leqslant N}$$
where, for ${(i,j) \in \llbracket 1, N-1 \rrbracket \times \llbracket i + 1, N \rrbracket}$, $a_{i,j} = 1$ if $\obs_i$ and $\obs_j$ belong to the same cluster and $a_{i,j} = 0$ otherwise. Similarly, after clustering by a given algorithm, we can define the output upper-triangular adjacency matrix 
$$\est{\Amat}_{\obs_1, \ldots, \obs_N} = (\est{a}_{i,j})_{1 \leqslant i \leqslant N, 1 \leqslant j \leqslant N}$$
where, for ${(i,j) \in \llbracket 1, N-1 \rrbracket \times \llbracket i + 1, N \rrbracket}$, $\est{a}_{i,j} = 1$ if $\obs_i$ and $\obs_j$ are assigned to the same cluster and $\est{a}_{i,j} = 0$ otherwise. We define the clustering error rate of the algorithm under consideration by the Hamming distance between $\Amat_{\obs_1, \ldots, \obs_N}$ and $\est{\Amat}_{\obs_1, \ldots, \obs_N}$ divided by $N(N-1)/2$. In other words, the Error Rate induced by the given algorithm is calculated as
\begin{equation}
	\label{eq: error rate}
	\text{Error Rate} = \frac{2}{N(N-1)} \sum_{i=1}^{N-1} \sum_{j = i+1}^N \vert \est{a}_{i,j} - a_{i,j} \vert \myfstop
\end{equation}

%\vspace{0.1cm}
\subsubsection{Clustering performance for constant known scaled-identity covariance matrices (see Figure \ref{fig:CentreX_vs_MeanShift})}\label{sec:fixed_sigma}

We consider the case of covariance matrices $\MCov_n$ all equal to $\sigma^2 \, \Id$, where $\sigma > 0$ is known. We thus take $\sigmamin = \sigmamax = \sigma$. In our simulations, $\sigma \in [1, 40]$. Our benchmarking involves:
\begin{itemize}
		\item \algo{} (Wald, {$\sigma^2 \Id$}) and \algo{} (Gauss, {$\sigma^2 \Id$}): \algo{} with Wald kernel and \algo{} with Gaussian kernel, both with $\estMCovn = \sigma^2 \, \Id$ for each $n \in \intn$; the Gaussian kernel used is $w_{\text{Gauss}}(t) = e^{-{t^2}/{2 c \sigma^2}} (t \geqslant 0),$ where $c$ is an empirical coefficient chosen equal to $5$ after preliminary trials;
%		\item \algo{} (Gauss, \revoir{$\sigma^2 \Id$}): \algo{} with Gaussian kernel and $\estMCovn = \sigma^2 \, \Id$ for each $n \in \intn$;
		\item Mean-Shift~(Wald), Mean-Shift~(Gauss) and Mean-Shift~(flat): Mean-Shift with Wald kernel, Gaussian kernel and flat kernel, respectively; for each variant, $h = \sigma$ in \eqref{eq: MyNewfunc} and the Gaussian kernel is $w_{\text{Gauss}}$ defined above;
 		\item K-means++ and X-means.
\end{itemize}
We observe in Figure~\ref{fig:CentreX_vs_MeanShift} that Mean-Shift (Wald) and \algo~(Wald, {$\sigma^2 \Id$}) on the one hand, and Mean-Shift (Gauss) and \algo~(Gauss, {$\sigma^2 \Id$}) on the other hand, yield very close results. We also notice that \algo{} (Gauss, {$\sigma^2 \Id$}) slightly outperforms Mean-Shift (Gauss). As expected, Mean-Shift (flat) yields the worst performance measurements, even for small values of $\sigma$, noticeably. That Mean-Shift and \algo~yield similar performance measurements, {when both are equipped with the same kernel, follows from the next remark.

\begin{Remark}
    \label{rmk: MS = centrex}
    Regardless of the covariance matrices $\MCov_n$ for $n \in \intn$, when $\estMCovn = h^2 \Id$ for some real value $h$ and for all $n \in \intn$, the mean-shift function $\Thefuncbasic$ defined by \eqref{eq: our general function} reduces to the mean-shift of \eqref{eq: MyNewfunc} used in~\citet{comaniciu2002mean}. According to their description above, the sole difference between \algo{} and the Mean-Shift algorithm is the marking step performed by \algo{} to reduce the number of estimated fixed-points to merge. Therefore, when $\estMCovn = h^2 \Id$ for some real value $h$ and all $n \in \intn$, and regardless of the covariance matrices $\MCov_n$ for $n \in \intn$, \algo{} and the Mean-Shift algorithm, whether they both use the Wald kernel or the Gaussian kernel, will perform similarly, the key advantage of \algo{} compared to Mean-Shift being that \algo{} is much less complex than Mean-Shift, as already pointed out in Section \ref{sec:algo}.
\end{Remark}

Figure~\ref{fig:CentreX_vs_MeanShift} shows also that the Wald kernel noticeably improves the clustering performance, since the performance measurements of Mean-Shift (Wald) and \algo{} (Wald, {$\sigma^2 \Id$}) are significantly higher than those of Mean-Shift (Gauss) and \algo{} (Gauss, {$\sigma^2 \Id$}).

In terms of proportion of correct $K$, \algo{} (Wald, {$\sigma^2 \Id$}) and Mean-Shift~(Wald) outperform X-means for $\sigma \leqslant 30$. For $\sigma > 30$, the proportion of correct $K$ of \algo{} (Wald, {$\sigma^2 \Id$}) and Mean-Shift~(Wald) drops, whereas X-means maintains a proportion of correct $K$ above $0.9$. Note also that the error rates of K-means++ and X-means are larger than those of \algo{} (Wald, {$\sigma^2 \Id$}) and Mean-Shift (Wald) for $\sigma$ above $7$. This trend will be verified in our next experimental results.

We must also notice the significant hump of the Silhouette curve obtained for Mean-Shift (flat). This is not contradictory to the {performance} of this algorithm in terms of proportion of correct $K$ and error rate. Indeed, Figure \ref{fig:CentreX_vs_MeanShift}(d) presents the average estimated value of $K$ in function of $\sigma$ for all the benchmarked algorithms. Since the value of $K$ is distributed uniformly between $2$ and $10$, the empirical average value of $K$ should be $6$. Mean-Shift (flat) achieves this empirical value for $\sigma \leqslant 20$. However, for $\sigma > 20$, Mean-Shift (flat) overestimates $K$, before underestimating it for $\sigma$ above $30$.

According to the foregoing, we will now always equip \algo{} and the Mean-Shift algorithm with the Wald kernel. The kernel will thus not be mentioned any longer in the legend of the next figures.

\clearpage
 
\begin{figure*}[ht]
		\subfloat[~]{ \includegraphics[width=\myscale\linewidth]{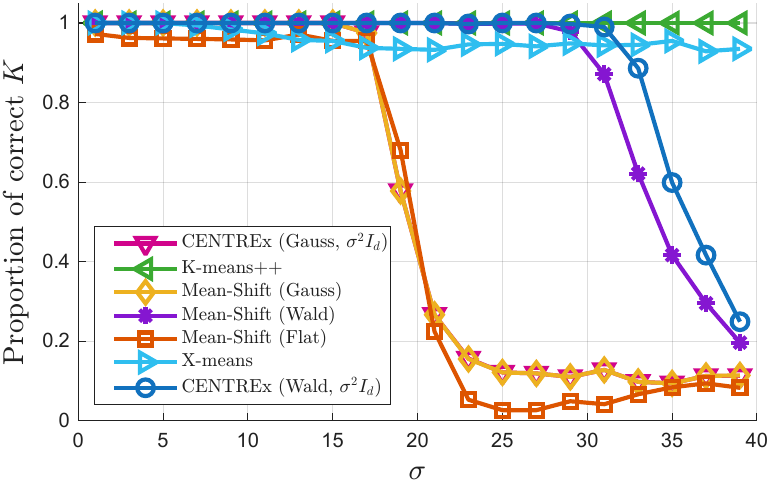}}
		\subfloat[~]{ \includegraphics[width=\myscale\linewidth]{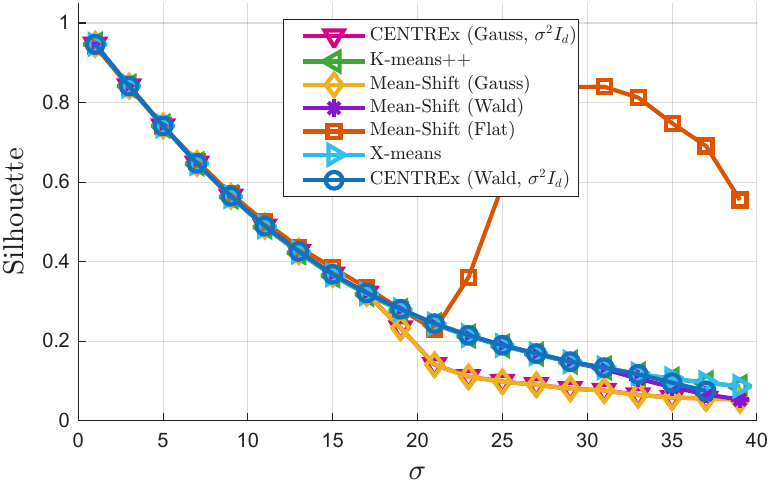}} \\
		\subfloat[~]{ \includegraphics[width=\myscale\linewidth]{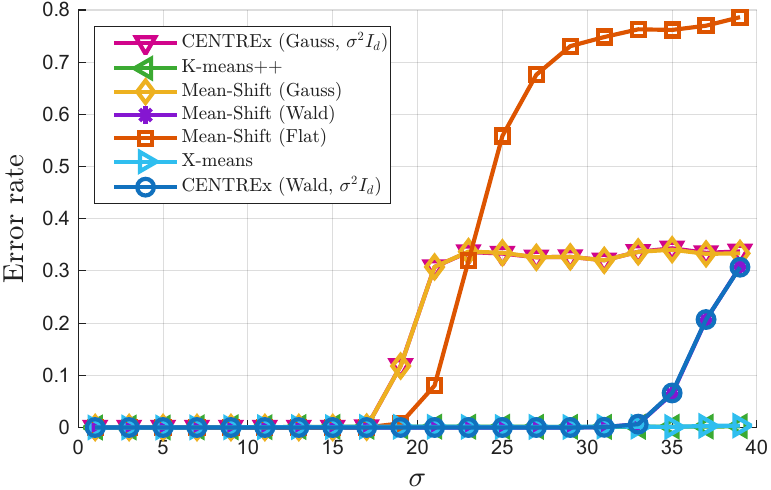}}
		\subfloat[~]{ \includegraphics[width=\myscale\linewidth]{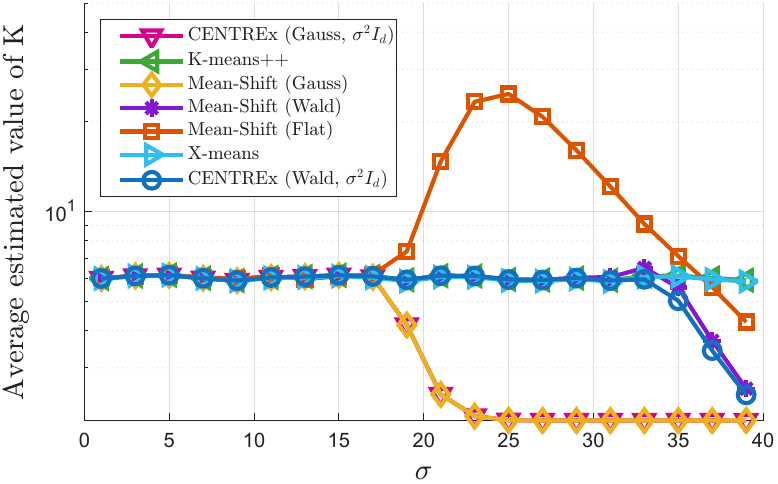}}	\\
		\subfloat[~]{ \includegraphics[width=\myscale\linewidth]{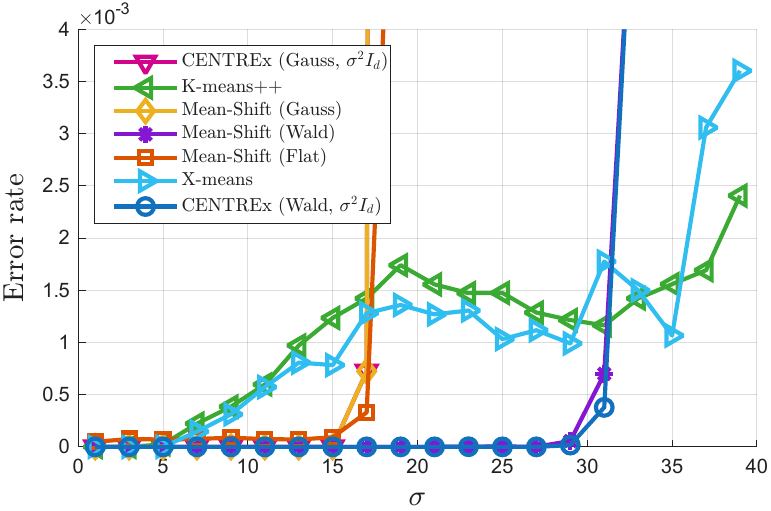}}	
	\caption{\small Clustering performance of \algo{} vs. K-means++, X-means and Mean-Shift with various kernels (flat, Gaussian, Wald). For all $n \in \intn$, $\MCov_n = \sigma^2 \, \Id$ with known $\sigma$: (a) Percentage of correctly retrieved number of clusters; (b) Silhouette; (c) Error rate; (d) Average estimated value of $K$; (e) Zoom on error rate.}
	\label{fig:CentreX_vs_MeanShift}
\end{figure*}

\clearpage

\subsubsection{Clustering performance for constant unknown scaled-identity covariance matrices (see Figure~\ref{fig:clustering_robustness})}\label{sec:fixed unknownsigma}

We still consider the homoscedastic case of equal covariance matrices $\MCov_n = \sigma^2 \, \Id$ for all $n \in \intn$. We thus take $\sigma = \sigmamin = \sigmamax$ again. This time, we consider Mean-Shift and \algo{}, both with the Wald kernel, when $\sigma$ is estimated by $\sigmaMLE$ calculated according to Subsection \ref{sec:est_sigma}, with $M = P = 50$ in \eqref{eq: fsigma} and 
$\MCov = \Id$ in \eqref{eq: final estimate of MCov in the MLE case}.
\medskip \\
\indent
To further assess their performance and robustness, we present in Figure~\ref{fig:clustering_robustness} the clustering performance of:
\begin{itemize}
\item \algo~($\sigmaMLEsq \, \Id$) with, for any $n \in \intn$, $\estMCovn = \sigmaMLEsq \, \Id$;
\item Mean-Shift~($\sigmaMLE$) where we set $h = \sigmaMLE$ in \eqref{eq: MyNewfunc};
\item \algo~($\sigma^2 \Id$) already considered in the previous experiment.
% \item \important{Mean-Shift (Wald, \revoir{$\sigmaMLE$}) and Mean-Shift (Gauss, \revoir{$\sigmaMLE$}): Mean-Shift with Wald kernel \revoir{and Mean-Shift with Gauss kernel, both with mean-shift function of \eqref{eq: MyNewfunc} where we set $h = \sigmaMLE$}};
% \item Mean-Shift (Gauss, \revoir{$\sigmaMLE$}): Mean-Shift with Gaussian kernel and mean-shift function of \eqref{eq: MyNewfunc} with $h = \sigmaMLE$; 
% \item \algo{} (Wald, \revoir{$\sigma^2 \, \Id$}), \revoir{Mean-Shift~(Wald, $\sigma$)} and Mean-Shift (Gauss, \revoir{$\sigma$}), already considered in the previous experiment \mysout{and thus adjusted with the exact value of $\sigma$}.
\item K-means++ and X-means.
\end{itemize}
\algo~($\sigmaMLEsq \, \Id$) and Mean-Shift~($\sigmaMLE$) perform similarly as explained by Remark~\ref{rmk: MS = centrex} above. \algo~($\sigmaMLEsq \, \Id$) suffers from a limited performance degradation compared to \algo{} ($\sigma^2 \Id$). The behavior of \algo~($\sigmaMLEsq \, \Id$) in comparison with X-means and K-means++ is thus similar to that of \algo~($\sigma^2 \Id$). This means that if we have a constant standard deviation for the data, \algo~can be used even if this standard deviation is unknown and estimated.

%The same holds for Mean-Shift (Gauss, \revoir{$\sigmaMLE$}) compared to Mean-Shift (Gauss, \revoir{$\sigma$}). However, Mean-Shift (Gauss, \revoir{$\sigmaMLE$}) and Mean-Shift (Gauss, \revoir{$\sigma$}) yield significantly lower performance than their counterpart with Wald kernel.

%\important{Since we consider \mysout{homoscedastic clusters with} scaled-identity covariance matrices, our mean-shift function $\Thefuncbasic$ given by \eqref{eq: our general function} equates the {\em mean-shift} of \eqref{eq: MyNewfunc} with $h = \sigmaMLE$. Thus, as in Section \ref{sec:fixed_sigma}, \algo{} (Wald, \revoir{$\sigmaMLEsq \, \Id$}) and Mean-Shift (Wald, \revoir{$\sigmaMLE$}) yield almost the same performance results, \algo{} (Wald, \revoir{$\sigmaMLEsq \, \Id$}) being however much faster than Mean-Shift (Wald, \revoir{$\sigmaMLE$}), as already mentioned above.}

%\mysout{As in Figure \ref{fig:CentreX_vs_MeanShift}, X-means is outperformed by \algo{} (Wald, MLE) in terms of estimation of $K$ and error rates. However, as above, X-means yields the best Silhouette curve, probably by attempting to optimize this criterion when estimating $K$. As above again, the computation of the average estimated values of $K$ of the algorithms shows that, as $\sigma$ grows, X-means still tends to overestimate the number of centroids, whereas the other algorithms, like K-means++, tend to underestimate this number.}

\newpage 

\begin{figure*}[t]
		\subfloat[~]{ \includegraphics[width=\myscale\linewidth]{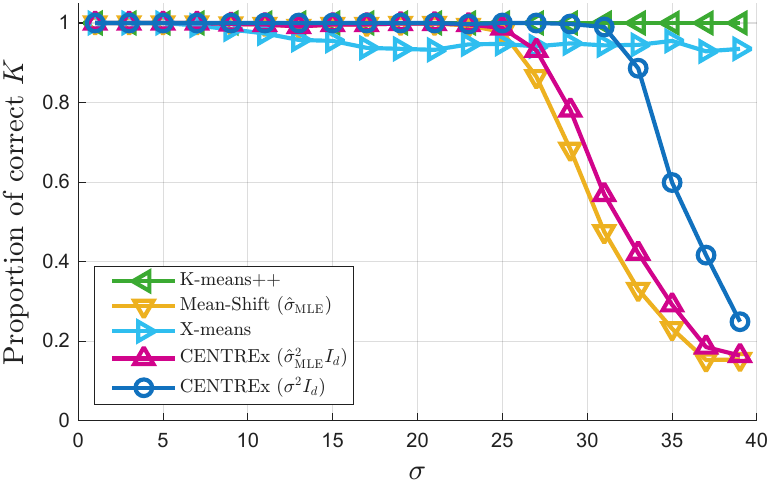}}
		\subfloat[~]{ \includegraphics[width=\myscale\linewidth]{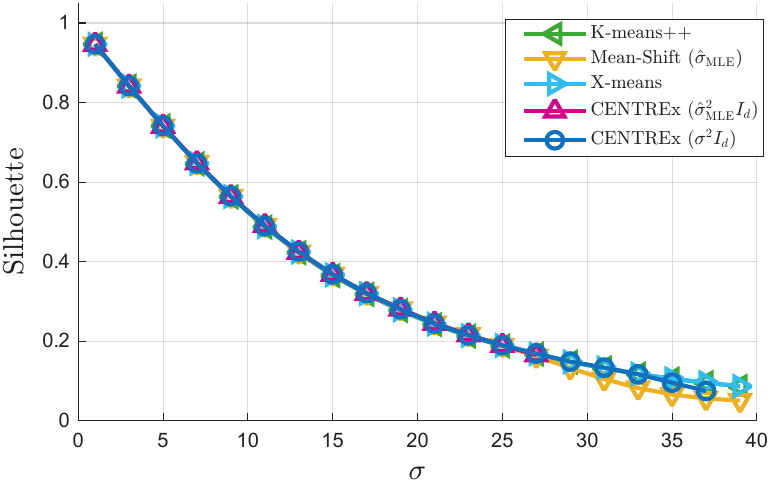}} \\
		\subfloat[~]{ \includegraphics[width=\myscale\linewidth]{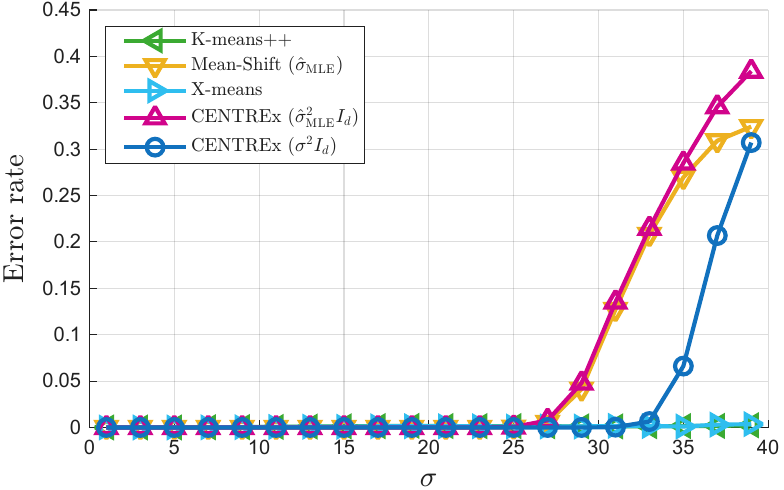}}
		\subfloat[~]{ \includegraphics[width=\myscale\linewidth]{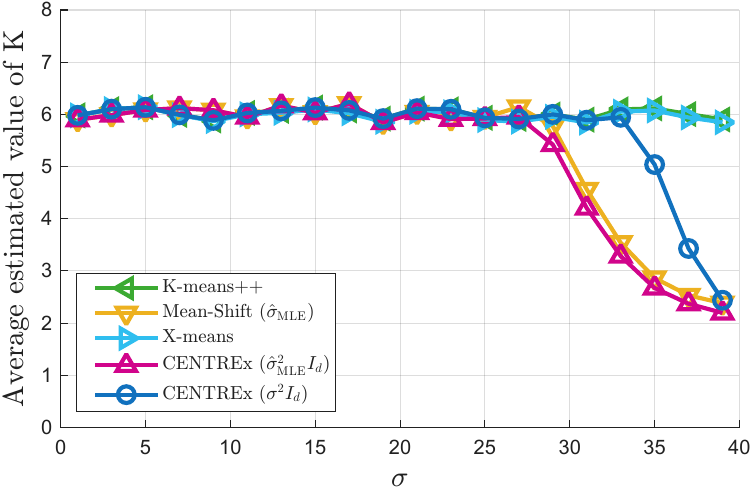}}	\\
		\subfloat[~]{ \includegraphics[width=\myscale\linewidth]{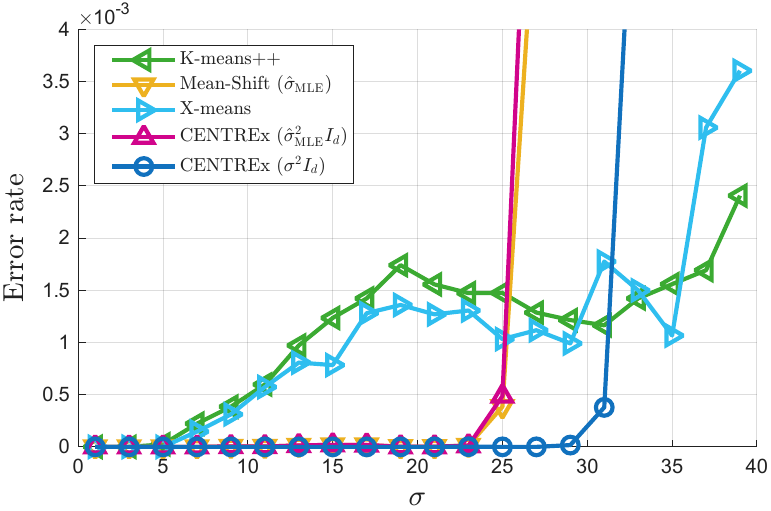}}	
	\caption{\small Clustering performance and robustness of \algo{} vs. K-means, X-means, and Mean-Shift. For all $n \! \in \! \intn$, $\MCov_n \! = \! \sigma^2 \, \Id$ and $\sigma$ is estimated by MLE if unknown to \algo: (a) Percentage of correctly retrieved number of clusters; (b) Silhouette; (c) Error rate; (d) Average estimated value of $K$; (e) zoom on error rate.}
	\label{fig:clustering_robustness}
\end{figure*}

\clearpage

\subsubsection{Clustering performance for varying diagonal covariance matrices (See Figures \ref{fig:clustering_robustness_varying-fig4} and \ref{fig:clustering_robustness_varying-fig6})}\label{sec:varying_sigma}

%In Figures \ref{fig:clustering_robustness_varying-fig4} and \ref{fig:clustering_robustness_varying-fig6}, each matrix $\MCov_n$ is diagonal but the $\mydim$ values $\sqrteigval_1, \ldots, \sqrteigval_{\mydim}$ in $[\sqrteigvalmin,\sqrteigvalmax]$ according to a process described precisely in the corresponding sections. These $\mydim$ values determine unequivocally the covariance matrix $\MCov_n$ of $\Obs_n$ via \eqref{eq: eigen decompositions}. 

%By so proceeding, we generate {heteroscedastic} Gaussian clusters according to the model described in Section~\ref{sec:model}. 

In this section, for each $n \in \intn$, we assume a diagonal covariance matrix $$\MCov_n = \Diagmat_n = \diag(\sigma^2_{n,1}, \ldots, \sigma^2_{n,\mydim})$$ where, for each $\dimi \in \intd$, $\sigma_{n,\dimi}$ is randomly chosen in the interval $[\sigmamin,\sigmamax]$, with $\sigmamax = \sigmamin + 4$ and $\sigmamin$ ranging in $[1,40]$. The probability distribution followed by the values $\sigma_{n,\dimi}$ will be specified below and will be different for Figure \ref{fig:clustering_robustness_varying-fig4} and Figure \ref{fig:clustering_robustness_varying-fig6}.

When the exact values of the coefficients $\sigma_{n,\dimi}$ are unknown to \algo{}, it would be desirable to estimate them. Unfortunately, we have no theoretical answer to this estimation problem. We thus proceed empirically instead, by setting $\estMCovn = \est{\sigma}^2 \Id$ for all $n \in \intn$, where $\est{\sigma} \in  \Open{0}{\infty}$ is chosen as a unique estimate for all the unknown standard deviations $\sqrteigval_{n,\dimi}$. In this respect, the experimental results of Figures \ref{fig:clustering_robustness_varying-fig4} and \ref{fig:clustering_robustness_varying-fig6} below are obtained by using the following values for $\est{\sigma}$: 
\medskip \\
\noindent (i) $\est{\sigma} = \sigmamin$; 
\medskip \\
\noindent
(ii) $\est{\sigma} = \sigmamax$; 
\medskip \\
\noindent 
(iii) $\est{\sigma} = \sigmaMEAN = (\sigmamin + \sigmamax)/2$; 
\medskip \\
\noindent
(iv) $\est{\sigma} = \sigmaMLE$, where $\sigmaMLE$ is calculated according to Subsection \ref{sec:est_sigma} with same parameter values as those used to produce Figure~\ref{fig:clustering_robustness}. 
\medskip \\
The difference between Figures \ref{fig:clustering_robustness_varying-fig4} and \ref{fig:clustering_robustness_varying-fig6} lies in the way we generate data. 
\medskip \\
\indent
To obtain Figure~\ref{fig:clustering_robustness_varying-fig4}, data are generated according to Section \ref{sec: data generation} with, for any $n \in \intn$ and any $\dimi \in \intd$, $\sigma_{n,\dimi}$ randomly drawn with uniform distribution in $[\sigmamin,\sigmamax]$. In Figure~\ref{fig:clustering_robustness_varying-fig4}, we then display the results achieved by:
\begin{itemize}
    \item \algo{} with the different setups (i), (ii), (iii) and (iv) proposed above and respectively denoted by \algo{} ($\sigmamin^2 \Id$), \algo{} ($\sigmamax^2 \Id$), \algo{} ($\sigmaMEAN^2 \Id$) and \algo{} ($\sigmaMLE^2 \Id$);
    \item \algo{} ($\Delta$), that is, \algo{} with perfect knowledge of the matrices $\Diagmat_1, \ldots, \Diagmat_N$;
%    \item Mean-Shift ($\sigmaMEAN^2 \Id$): Mean-Shift with Wald kernel and $h = \sigmaMEAN$;
    \item K-means++ and X-means.
\end{itemize}
The Mean-Shift algorithm is not considered because, according to Remark \ref{rmk: MS = centrex}, Mean-Shift equipped with the Wald kernel and $\est{\sigma} \in \{ \sigmamin, \sigmamax, \sigmaMEAN , \sigmaMLE\}$ would yield approximately the same results as \algo{} ($\est{\sigma}^2 \Id$).
\medskip \\
%\revoir{We first notice that heteroscedasticity has induced a performance loss of \algo{} ($\Delta$) in comparison with Figures \ref{fig:CentreX_vs_MeanShift} and  \ref{fig:clustering_robustness}}.
\indent
We first observe that \algo{} ($\sigmaMEAN^2 \Id$) yields no significant performance loss with respect to \algo{} ($\Delta$). Therefore, setting $\estMCovn = \sigmaMEAN^2 \Id$ for all $n \in \intn$ can be seen as a satisfactory solution in the experimental setting considered here. This is not surprising since $\sigmaMEAN$ is the expectation of the uniform distribution according to which the standard deviations $\sigma_{n,i}$ are chosen in $[\sigmamin, \sigmamax]$ for all $n \in \intn$ and all $\dimi \in \intd$. 

Although the theoretical hypotheses of Section \ref{sec:est_sigma} to compute $\sigmaMLE$ are clearly not satisfied under the present experimental setting, \algo{} ($\sigmaMLE^2 \Id$) undergoes only a slight performance loss in comparison to Fig. \ref{fig:clustering_robustness}. The clusters generated by choosing the standard deviations $\sigma_{n,i}$ uniformly and independently in the same interval $[\sigmamin, \sigmamax]$ are likely not heteroscedastic enough for \algo{} ($\sigmaMLE^2 \Id$) to achieve performance measurements significantly different from those of Fig. \ref{fig:clustering_robustness}.

It must also be noticed that \algo~($\sigmaMLE^2 \Id$) and \algo~($\sigmamax^2 \Id$) perform similarly. In contrast, \algo{} ($\sigmamin^2 \Id$) performs very poorly, with a significant overestimation of the number of centroids. In the present experimental setting, overestimating noise standard deviations is far better than underestimating them.

%\revoir{We must also remark that \algo{} ($\sigmamax^2 \Id$) performs similarly as \algo{} ($\Delta$), \algo{} ($\sigmaMEAN^2 \Id$) and \algo{} ($\sigmaMLE^2 \Id$). In contrast,} 

%\algo{} ($\sigmamin^2 \Id$) performs poorly in terms of proportion of correct estimates $K$ and error rates. The average estimated values of $K$ show that \algo{} ($\sigmamin^2 \Id$) overestimates the number of centroids in comparison with K-means, \algo{} ($\Delta$) and \algo{} ($\sigmaMEAN^2 \Id$). This overestimation explains the higher Silhouette coefficients of \algo{} ($\sigmamin^2 \Id$) in Figure~\ref{fig:clustering_robustness_varying-fig4}(b), this behaviour being analogous to that of Mean-Shift (flat) in Figure \ref{fig:CentreX_vs_MeanShift}(b). \revoir{The results yielded by \algo{} ($\sigmamax^2 \Id$) and \algo{} ($\sigmamin^2 \Id$) indicate that, in the present experimental setting, overestimating noise standard deviations is far better than underestimating them and, above all, does not impact significantly the performance of \algo.}

\clearpage
\begin{figure*}[t]
\vspace{-1cm}
 		\subfloat[~]{\includegraphics[width=\myscale\linewidth]{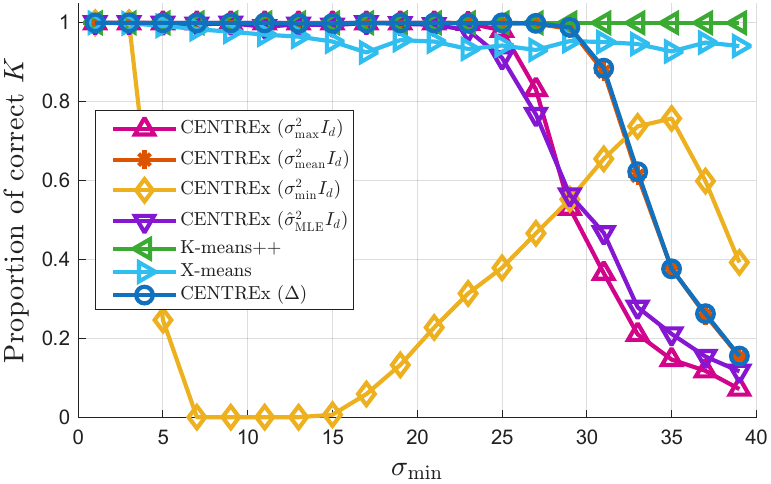}}
 		\subfloat[~]{\includegraphics[width=\myscale\linewidth]{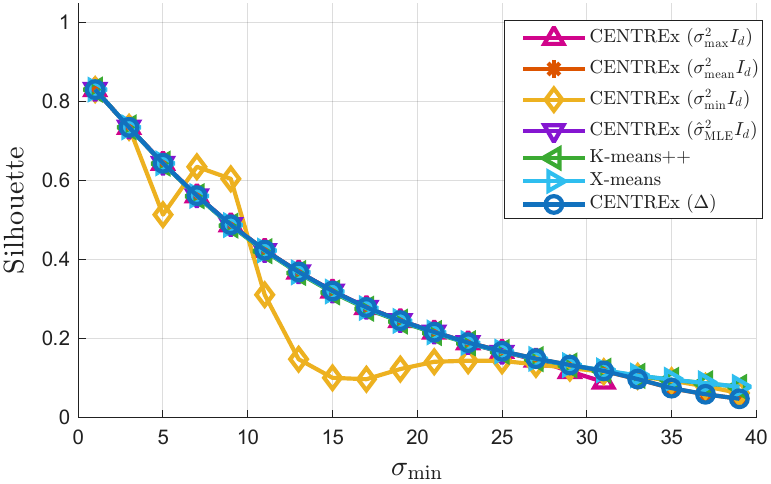}} \vspace{-0.35cm} \\
 		\subfloat[~]{\includegraphics[width=\myscale\linewidth]{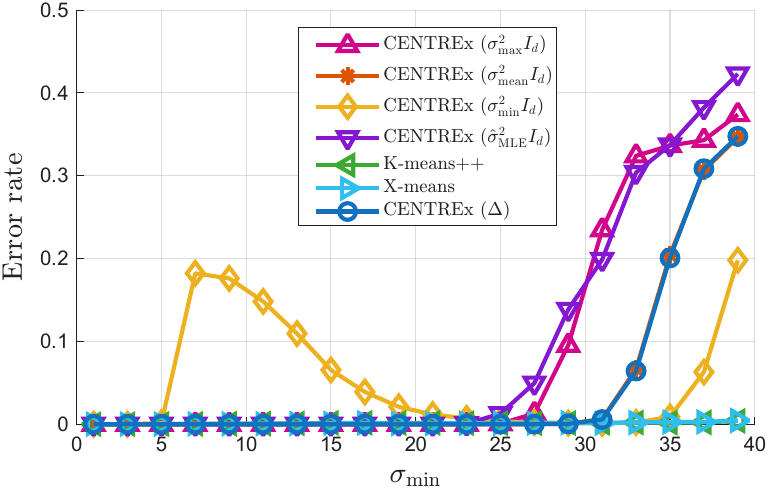}}
 		\subfloat[~]{\includegraphics[width=\myscale\linewidth]{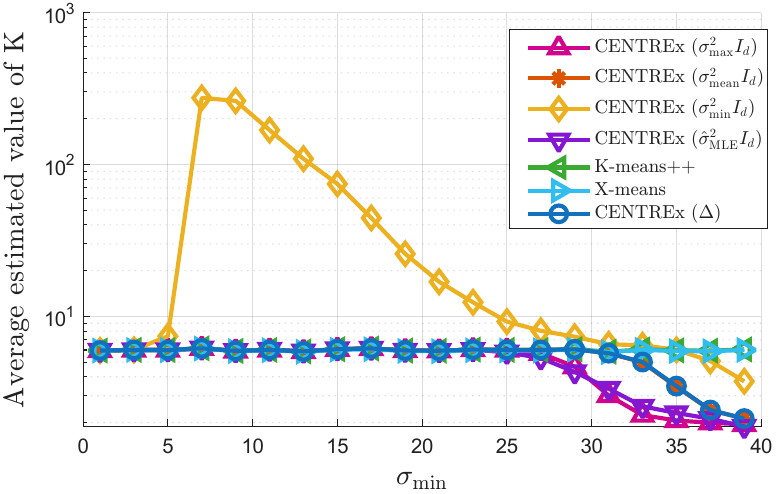}} \vspace{-0.35cm} \\
 		\subfloat[~]{ \includegraphics[width=\myscale\linewidth]{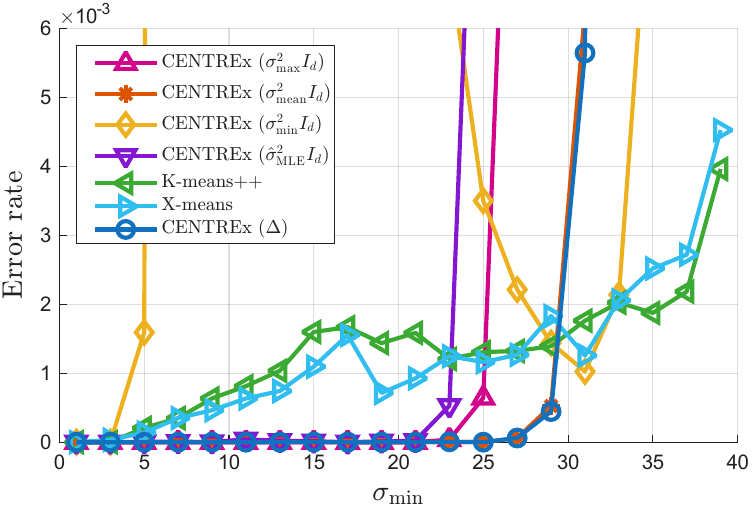}}	
        \vspace{-0.35cm}
 	\caption{\small Clustering performance of \algo{} vs. K-means++ and X-means. For each $\sigmamin \! \in \! [1,40]$ and for all $n \! \in \! \intn$, $\MCov_n$ is a diagonal matrix $\Diagmat_n$ with random diagonal elements uniformly distributed in $[\sigmamin, \sigmamax]$ with $\sigmamax = \sigmamin + 4$. \algo{} ($\Diagmat$) is \algo{} with perfect knowledge of $\Diagmat_1, \ldots, \Diagmat_N$ and, for different $\est{\sigma}$, \algo{} ($\est{\sigma}^2 \Id$) is \algo{} with $\Diagmat_n = \est{\sigma}^2 \Id$ for all $n \in \intn$: (a) Percentage of correctly retrieved number of clusters; (b) Silhouette; (c) Error rate; (d) Average estimated value of $K$; (e) zoom on error rate.}
 	\label{fig:clustering_robustness_varying-fig4}
 \end{figure*}

\clearpage

The heteroscedastic clusters generated to  obtain Figure \ref{fig:clustering_robustness_varying-fig4} does not make it possible to exhibit clear differences in performance between \algo{} ($\Delta$) and \algo~($\sigmaMEAN^2 \Id$), on the one hand, and between \algo~($\sigmaMLE$) and \algo~($\sigmamax^2 \Id$), on the other hand. Thence the following experimental setting to make appear such differences in our next Figure~\ref{fig:clustering_robustness_varying-fig6}.

\vspace{0.1cm}
Given the intervals $I_0 = [\sigmamin, \sigmamin+1]$ and $I_1 = [\sigmamax-1,\sigmamax]$, the results of Figure~\ref{fig:clustering_robustness_varying-fig6} are obtained by generating data as follows.
\begin{enumerate}
    \item Compute the number $K$ of centroids and the centroids themselves as described in Section \ref{sec: data generation}. %Let $\ctr_1, \ldots, \ctr_K$ be the $K$ centroids thus generated.
    \item For each $n \in \intn$, choose the mean of the measurement vector $\Obs_n$ among the $K$ available centroids, as specified in Section \ref{sec: data generation}. 
    \item For each $k \in \intk$, perform an independent and equiprobable Bernoulli trial $B_k$ valued in $\{0,1\}$. For all $n \in \intn$ such that $\Obs_n$ belongs to cluster $\#k$, set $$ J_n = \left \{ \begin{array}{lll} \! I_0 & \text{if $B_k = 0$} \vspace{0.1cm} \\ \! I_1 & \text{otherwise.~} \end{array}  \right.$$
    \item Compute the covariance matrix of each measurement vector $\Obs_n$ belonging to the $k^{\text{th}}$ cluster as $\MCov_n = \Diagmat_n = \diag(\sigma^2_{n,1}, \ldots, \sigma^2_{n,\mydim}),$ where the standard deviations $\sigma_{n,\dimi}$ are randomly and independently drawn with uniform distribution in the interval $J_n$.
\end{enumerate}
We thus create heteroscedastic clusters that are either small, in that the standard deviations of the coordinates of their elements are all in $I_0$, or large, in that the standard deviations of the coordinates of their elements are all in $I_1$. 

\vspace{0.1cm}
We then perform the clustering of the data via \algo, equipped with the Wald kernel, but tested with different choices for the matrices $\estMCovn$ when $n \in \intn$. In our benchmarking, we take into account that the intervals $J_1, \ldots, J_N$ basically provide a more precise location of the standard deviations $\sigma_{n,\dimi}$ than $[\sigmamin,\sigmamax]$, and that these intervals can be assumed to be known regardless of the Bernoulli trials from which they are generated. 
}

\vspace{0.1cm} 
Specifically, in Figure~\ref{fig:clustering_robustness_varying-fig6}, we assess:
\begin{itemize}
\item \algo{} ($\Diagmat$) with, for any $n \in \intn$, $\estMCovn = \Diagmat_n$;
\item \algo~($\est{\Diagmat}$) with, for any $n \in \intn$, $\estMCovn = \est{\Diagmat}_n = \diag(\est{\sigma}^2_{n,1}, \ldots, \est{\sigma}^2_{n,\mydim})$ where, for any $\dimi \in \intd$, $\est{\sigma}_{n,\dimi}$ is the middle point of the interval $J_n$ in which the value ${\sigma}_{n,\dimi}$ has been drawn;
\item \algo{} ($\est{\sigma}^2 \Id$) with, for any $n \! \in \! \intn$, $\estMCovn \! = \! \est{\sigma}^2 \! \Id$ and $\est{\sigma} \! \in \! \{ \sigmaMLE, \sigmaMEAN, \sigmamin, \sigmamax \}$ as in Figure \ref{fig:clustering_robustness_varying-fig4};
\item K-means++ and X-means.
\end{itemize}
We observe that \algo{} ($\est{\Diagmat}$) achieves the same performance as \algo{} (${\Diagmat}$). This is explained by the fact that both $I_0$ and $I_1$ are sufficiently small for their middle point to be good enough estimates of random variables uniformly distributed over each of them. Due to the heteroscedasticity of the measurement vectors, the other variants of \algo{} undergo a performance loss compared to the previous figures. Again, the hypotheses underlying the computation of $\sigmaMLE$ in Section \ref{sec:est_sigma} are clearly not satisfied under the present experimental setting where the measurement vectors used to generate the data are heteroscedastic.

\clearpage

\begin{figure*}[t]
\vspace{-0.5cm}
		\subfloat[~]{\includegraphics[width=\myscale\linewidth]{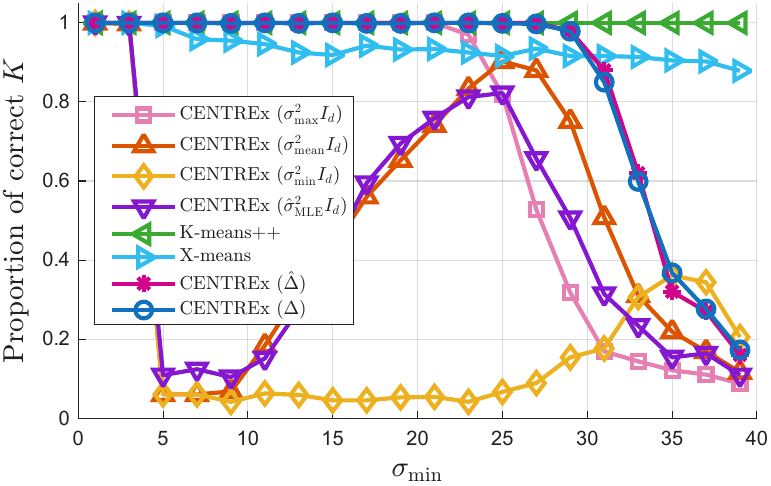}}
		\subfloat[~]{\includegraphics[width=\myscale\linewidth]{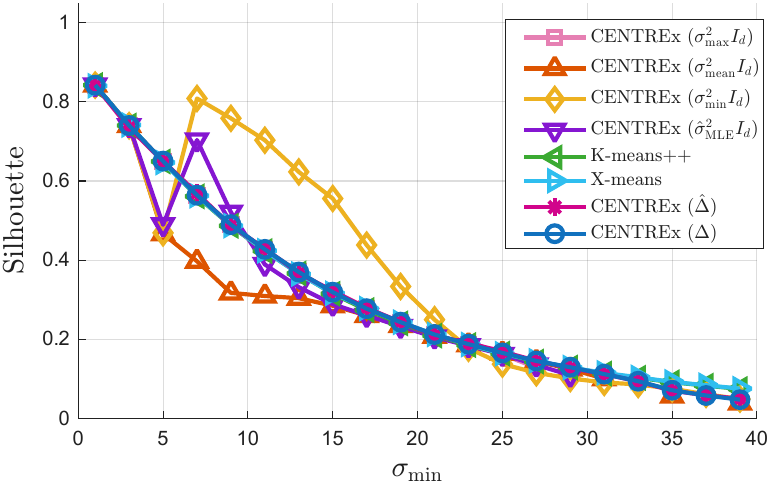}} \vspace{-0.4cm} \\
		\subfloat[~]{\includegraphics[width=\myscale\linewidth]{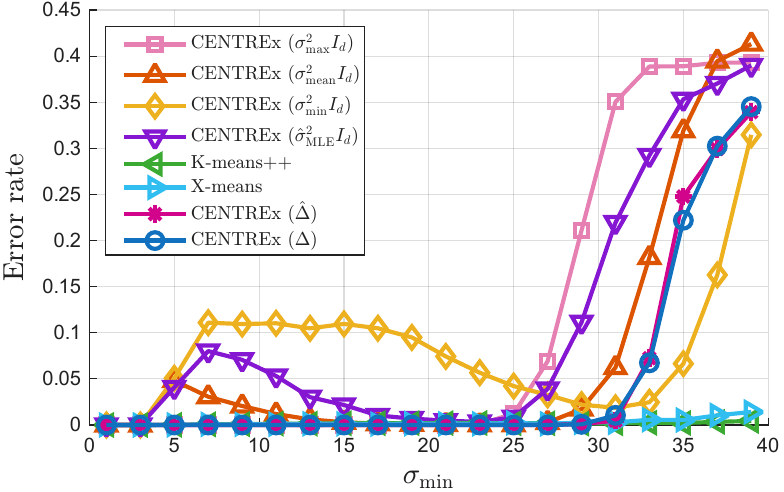}}
		\subfloat[~]{\includegraphics[width=\myscale\linewidth]{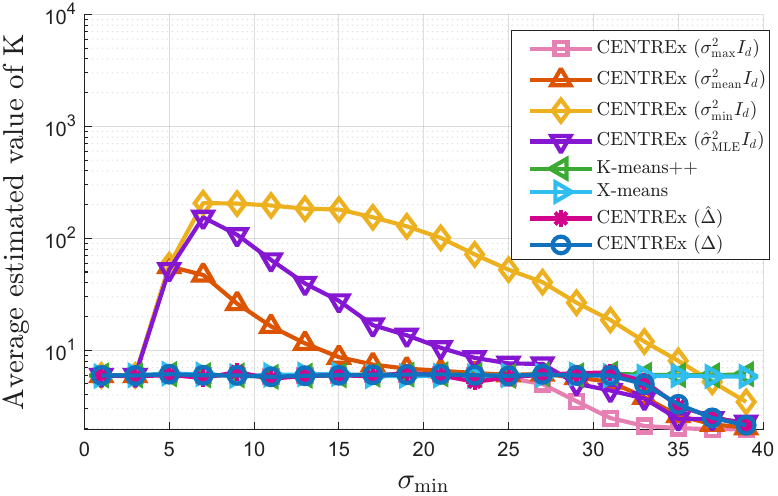}} \vspace{-0.4cm} \\
		\subfloat[~]{ \includegraphics[width=\myscale\linewidth]{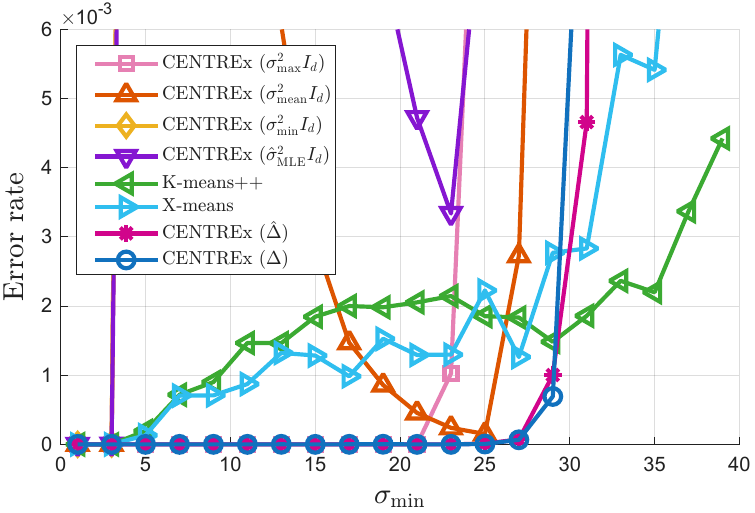}}
        \vspace{-0.4cm} 
	\caption{\small Performance of \algo{} vs. K-means++ and X-means. Given $\sigmamin \! \in \! [1,40]$ and cluster \#$k$, we equiprobably choose $I_0 \! = \! [\sigmamin, \sigmamin \! + \! 1]$ or $I_1 \! = \! [\sigmamax \! - \! 1,\sigmamax]$ with $\sigmamax \! = \! \sigmamin \! + \! 4$. The covariance matrix of each $\Obs_n$ in cluster \#$k$ is a diagonal matrix $\Diagmat_n$ with random diagonal elements uniformly distributed in the drawn interval. \algo{} ($\Diagmat$) is \algo{} with $\estMCovn \! = \! \Diagmat_n$ for all $n$, \algo{} ($\est{\Diagmat}$) is \algo{} with the $\estMCovn$ equal to estimates $\est{\Diagmat}_n$ of the $\Diagmat_n$ and, for different $\est{\sigma}$, \algo{} ($\est{\sigma}^2 \Id$) is \algo{} with all the $\estMCovn$ equal to $\est{\sigma}^2 \Id$ : (a) Percentage of correctly retrieved number of clusters; (b) Silhouette; (c) Error rate; (d) Average estimated value of $K$; (e) zoom on error rate.}
	\label{fig:clustering_robustness_varying-fig6}
\end{figure*}

\clearpage

\subsection{Clustering performance on standard data sets}

We also evaluate the performance of our algorithm on the two standard Machine Learning data sets Ruspini and IRIS. On both data sets, \algo{} is used under the assumption that the clusters have same covariance matrix equal to $\sigma^2 \Id$. It employs $\sigmaMLE$ according to Subsection~\ref{sec:est_sigma} to estimate $\sigma$, with parameter values $M \! = \! P = \! 75$ (resp. $M \! = \! P \! = \! 10$) in \eqref{eq: fsigma} for the Ruspini data set (resp. the Iris data set) and $\MCov \! = \! \Id$ again in \eqref{eq: final estimate of MCov in the MLE case} for the two data sets. We set $L=100$ maximum iterations for centroid estimation, and set $\epsilon_e = 10^{-3}$, $\epsilon_f = 1$, $\alpha = 10^{-3}$.

On the Ruspini data set, our algorithm correctly retrieves $K=4$ clusters, like most clustering algorithms. On the IRIS data set, it finds $K=2$ clusters while, in true, IRIS contains three classes. This is probably due to the IRIS data set geometry, since other clustering algorithms such as K-means with unknown $K$ also retrieve $K=2$ clusters on IRIS.

\begin{figure*}[t]
\begin{center}
  \subfloat[~]{ \includegraphics[width=\myscale\linewidth]{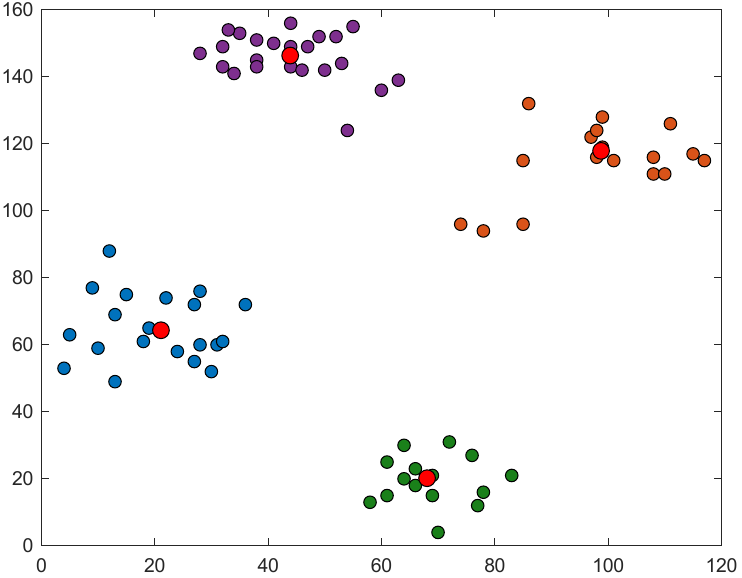}}
  \subfloat[~]{ \includegraphics[width=\myscale\linewidth]{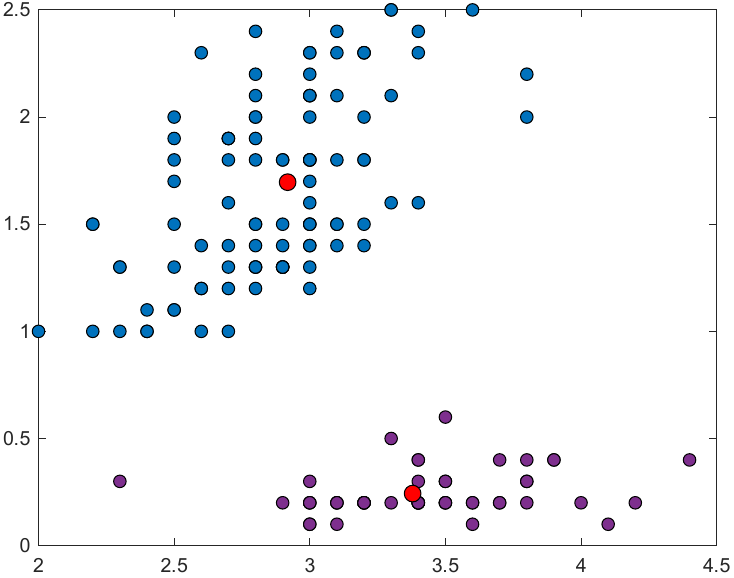}}
\end{center}
\caption{\small Clustering with \algo{} on (a) Ruspini, (b) IRIS ($2$nd and $4$th dimensions). The red dots are the estimated centroids.}
\label{fig:standard_bdd}
\end{figure*}

\subsection{{Brief summary and discussion}}
\label{subsec: summary}
Under our experimental set-up, the clustering algorithms with no prior knowledge of the number of clusters encounter difficulties to properly cluster data for $\sigmamin \geqslant 30$. Indeed, in our simulations, we set a minimum distance $\mu_{\min} = 200$ between any two centroids. With the Gaussian model, $99\%$ of the values taken by measurements of a specified cluster in a given direction are thus within $3\sigma$ of the cluster centroid. So, when $\sigma$ draws near $30$, $3\sigma \approx 90$ approaches $\mu_{\min}/2 = 100$, inducing that measurements from two clusters at a distance equal to $\mu_{\mymin}$ can mingle enough in certain directions to mislead clustering algorithms that have no prior knowledge of $K$.
\\
\indent
The simulation results presented above show that \algo{}, even when used with estimates of the measurement covariance matrices, is a relevant alternative to K-means++, Mean-Shift and X-means, when the number of clusters is unknown. These results suggest that a large difference between the numbers of centroids estimated by \algo{} and X-means respectively, means that the clusters are likely mixing too much to be separated without too many errors.
\\
\indent
In addition, considering the Wald kernel instead of the Gaussian kernel leads to significant gains in terms of clustering performance for both \algo{} and Mean-Shift. Finally, with a reduced complexity for \algo{} compared to Mean-Shift, when the covariance matrices of the measurements are all estimated by the same scaled-identity matrix, \algo{} and Mean-Shift achieve similar performance. All these experimental results emphasize the relevance of Theorem \ref{Theorem: fixed points} to estimate centroids by seeking the fixed-points of $\TheFuncbasic$, even in non-asymptotic cases.
\\
\indent
Like Mean-Shift, \algo{} and the underlying theory introduced in this paper are devised to handle Gaussian and sufficiently well-separated clusters. Although the experimental results above are promising regarding the ability of \algo{} to deal with estimates for the covariance matrices of the measurement vectors, the robustness of this algorithm should be further investigated to identify more precisely the relevant classes of estimators that could be used with it. The performance of \algo{} is also sensitive to the parameter $\epsilon_f$ used for the final fusion step required by Mean-Shift as well. Further studies on means to better control and mitigate the influence of this parameter are also required. % Figure 5 removed, replaced by Figure 6

\section{Conclusion and perspectives}\label{sec:conclusion}
In this paper, we have addressed the problem of clustering heteroscedastic Gaussian measurement vectors when the number $K$ of clusters is unknown. First, we have proposed a theoretical framework from which we have derived Theorem \ref{Theorem: fixed points} showing that the fixed-points of an extended and randomized version of the mean-shift function \citep{comaniciu2002mean} converge, in a probabilistic sense introduced in this paper, to the cluster centroids.
\\
\indent
From these theoretical results, we have derived a new clustering algorithm called \algo{}, which does not require any prior knowledge of $K$. \algo{} can be regarded as an extension of Mean-Shift with reduced complexity. We also introduced the Wald kernel, defined as the p-value of a Wald hypothesis test, in order to replace the Gaussian kernel usually considered in the literature. Extensive simulation results on synthetic data have been presented to get better insight into the behaviour and robustness of \algo{} in non-asymptotic situations. 
\\
\indent
Our approach opens several prospects to be addressed in future works. First, the properties of our extended and randomized mean-shift function $\TheFuncbasic$ could be further investigated to prove that the recursive routine used in our algorithm converges to its fixed-points. In this respect, it can be asked whether the theoretical frameworks established in \citet{ghassabeh2018modified} and \citet{yamasaki2019properties} can be combined with ours to yield a modified version of \algo{}, based on the Wald kernel and involving a convergence-guaranteed fixed-point search. On the other hand, because our theory and \algo{} algorithm extend the Mean-Shift theory and algorithm, the Mean-Shift variants in \citet{ghassabeh2018modified} and \citet{yamasaki2019properties} might in turn benefit from our theoretical results and the use of the Wald kernel. Finally, since our theoretical framework addresses the case of measurement vectors with possibly different covariance matrices, the theory and the resulting algorithm \algo{} provide us with a potential framework for distributed clustering over sensor networks, either in the centralized or the decentralized case.

\appendix

\section{Proof of Theorem \ref{Theorem: fixed points}}
\label{App: the proof}

This proof is based on many intermediate results, whose proofs are detailed as much as possible for the reader to follow the computations and arguments without undue effort. To alleviate its reading, we split the proof in {four} parts. {In Section \ref{subsec: prelim on esl}, we first state preliminary properties of the \esl~defined in Definition \ref{Definition: lim of limsup}. In Section \ref{subsec: the case of diagonal covariance matrices}, we establish $\myclubsuit$ in the particular case of any sequences $\MCov \! = \! (\MCov)_{n \in \Nset}$ and $\estMCov \! = \! (\estMCovn)_{n \in \Nset}$ of diagonal covariance matrices, that is, when $\Rmat = \Id$ in \eqref{eq: eigen decompositions}. From this particular case, we derive $\myclubsuit$ for non-diagonal covariance matrices in Section \ref{subsec:The case of any sequences of covariance matrices} and $\myspadesuit$ in Section \ref{subsec: proof of spade}.}

\subsection{{Preliminary properties of \esl}}
\label{subsec: prelim on esl}
{We begin with Proposition \ref{Proposition: Quanteur de lim of limsup} to express the $\esl$ in the form of a predicate. From this predicate, we derive useful properties that will be helpful to prove Theorem \ref{Theorem: fixed points}. To establish Proposition \ref{Proposition: Quanteur de lim of limsup}, we recall basic properties of a limsup in Lemma \ref{Lemma: limsup et quanteurs}}.
\begin{Lemma}
	\label{Lemma: limsup et quanteurs}
	Given any sequence $(x_n)_{n \in \Nset}$ of real numbers and any $\eta \in \Open{0}{\infty}$,
	\begin{onecol}
		\begin{subequations}
			\begin{empheq}[left={} \empheqlbrace]{align}
				& \mylimsup{N \to \infty} x_N < \eta \Rightarrow \big ( \exists N_0 \in \Nset, \forall N \in \Nset, N > N_0 \Rightarrow x_N < \eta \big ) \mycomma
				\label{eq: quanteur implication} \\
				& \big ( \exists N_0 \in \Nset, \forall N \in \Nset, N > N_0 \Rightarrow x_N < \eta \big ) \Rightarrow \mylimsup{N \to \infty} x_N \leqslant \eta \myfstop
				\label{eq: quanteur implication converse}
			\end{empheq}
		\end{subequations}
	\end{onecol}
\end{Lemma}
We also remind the reader that, for any random variable $X \in \fnrv{\Omega}{\Rset}$ and any $\eta \in \RightOpen{0}{\infty}$ (see \citealt[p.~66]{rudin87}):
\begin{onecol}
	\begin{equation}
		\label{eq: elementary property of the infinity norm}
		\begin{array}{lll}
			\vert X \vert_\infty \leqslant \eta 
			& \Leftrightarrow & \vert X \vert \leqslant \eta \, \, \text{\big ($\Pbb$-a.s\big )} \vspace{0.1cm} \\		
			& \Leftrightarrow & \exists \, \Omega_0 \in \tribu, 
			\big ( \, \Pbb(\Omega_0) = 1 \wedge (\forall \omega \in \Omega_0, \vert X(\omega) \vert \leqslant \eta) \, \big ) \myfstop
		\end{array}
	\end{equation}
\end{onecol}

\begin{Proposition}
	\label{Proposition: Quanteur de lim of limsup}
	Given any parametric discrete random process $\Process{}: \aSet \times \Nset \to \rv$, any base $\Base$ in $\aSet$, any $\mylim \in \Rset^{\mydim}$ and any norm $\nu$ on $\Rset^\mydim$,
	\begin{equation}
		\label{eq: the ess sup lim of a discrete random process-v0}
        \limlimsupnu{\Base}{N}{\Process{\variable,N} - \mylim \, } = 0
        %\esssuplim{\Base}{N} \Process{\variable,N} = \mylim
	\end{equation} 
	if and only if,
	\begin{onecol}
		\begin{equation}
			\label{eq: limlimsup with quantifiers-0}
			\begin{array}{lll}
%				\hspace{-0.4cm}
				\forall \eta \in \Open{0}{\infty}, \exists B \in  \Base, \\ %\vspace{0.1cm} \\
                \hspace{1.5cm} \forall \variable \in \aSet, %\vspace{0.1cm} \\
				%\qquad 
				\variable \in B \Rightarrow \exists \, \Omega_0 \in \tribu, 
				\left \{
				\begin{array}{lll} 
					\! \! \Pbb(\Omega_0) = 1 \\ %\medskip \\ 
					\! \! \mywedge \\ %\medskip \\ 
					\! \! \forall \omega \in \Omega_0, \exists N_0 \in \Nset, \forall N \in \Nset, \vspace{0.1cm} \\ %\medskip \\
					\hspace{0.15cm} N > N_0 \Rightarrow \norm[\Process{\variable,N}(\omega) - \mylim] < \eta \myfstop
				\end{array}
				\right.
			\end{array}
			\hspace{-0.6cm}
		\end{equation}
	\end{onecol}
\end{Proposition}
\begin{proof} 
    Consider a base $\Base$ in $\aSet$, $\mylim \in \Rset^{\mydim}$ and a norm $\anothernormsh$ on $\Rset^\mydim$. Suppose first that $$\limlimsupnu{\Base}{N}{\Process{\variable,N} - \mylim \, } = 0 \myfstop$$ In \eqref{eq: def of a limit-any norm}, by equivalently replacing the strict inequality by a slack one and considering $f(\variable) = \big \vert \mylimsup{N \to \infty} \anothernorm{\Process{\variable,N} - \mylim} \, \big \vert_\infty$ and the absolute value as norm on $\Rset$, we obtain:
	\begin{onecol}
		\begin{equation}
			\forall \eta \in \Open{0}{\infty}, \exists B \in \Base, \forall \variable \in \aSet, \variable \in B \Rightarrow \big \vert \, \mylimsup{N \to \infty} \anothernorm{\Process{\variable,N} - \mylim} \, \big \vert_{\infty} \leqslant \eta/2 \myfstop
			\label{eq: limlimsup and quantifiers - 1}
			\nonumber
		\end{equation}
	\end{onecol}
	According to \eqref{eq: elementary property of the infinity norm}, we can rewrite the predicate above as
	\begin{onecol}
		\begin{equation}
			\forall \eta \in \Open{0}{\infty}, \exists B \in \Base, \forall \variable \in \aSet, \variable \in B \Rightarrow \mylimsup{N \to \infty} \anothernorm{\Process{\variable,N} - \mylim } \leqslant \eta/2 %< \eta 
            \, \, \text{\big ($\Pbb$-a.s\big )} \mycomma
			\label{eq: limlimsup and quantifiers - 2}
			\nonumber 
		\end{equation}
	\end{onecol}
	and thus, via \eqref{eq: elementary property of the infinity norm} again, it follows that
	\begin{onecol}
		\begin{equation}
        \begin{array}{lll}
			\forall \eta \in \Open{0}{\infty}, \exists B \in \Base, \vspace{0.1cm} \\
            \hspace{1.2cm} \forall \variable \in \aSet, \variable \in B \Rightarrow \exists \, \Omega_0 \in \tribu, 
			\left \{
			\begin{array}{lll}
				\! \! \Pbb(\Omega_0) = 1 \\ %\medskip \\
				\! \! \mywedge \\ %\medskip \\
				\! \! \forall \omega \in \Omega_0, \displaystyle \mylimsup{\, N \to \infty} \anothernorm{\Process{\variable,N}(\omega) - \mylim } \leqslant \eta/2 < \eta 
			\end{array}
			\right. \myfstop
        \end{array}
		\label{eq: limlimsup and quantifiers - 2nd intermediate iplication}
		\nonumber
	\end{equation}
	\end{onecol}
	We obtain \eqref{eq: limlimsup with quantifiers-0} by injecting \eqref{eq: quanteur implication} into the predicate above and taking into account that all norms on $\Rset^\mydim$ are equivalent.
	\medskip \\
	\indent
	Conversely, if \eqref{eq: limlimsup with quantifiers-0} holds true, we derive from \eqref{eq: quanteur implication converse} and the equivalence of norms on $\Rset^\mydim$ that
    \begin{onecol}
		\begin{equation}
			\label{eq: limlimsup with quantifiers-converse1}
			\forall \eta' \in \Open{0}{\infty}, \exists B \in \Base, \forall \variable \in \aSet,
			\variable \in B \, \Rightarrow \exists \, \Omega_0 \in \tribu, 
			\left \{		
			\begin{array}{lll} 
				\! \! \! \Pbb(\Omega_0) = 1 \\ %\medskip \\ 
				\! \! \! \mywedge \\ %\medskip \\ 
				\! \! \! \forall \omega \! \in \! \Omega_0, \displaystyle \mylimsup{N \to \infty} \anothernorm{\Process{\variable,N} \! - \! \mylim } \! \leqslant \! \eta'
			\end{array}
			\right. \myfstop
			\nonumber
		\end{equation}
	\end{onecol}
	This predicate is equivalent to
	\begin{onecol}
		\begin{equation}
			\label{eq: limlimsup with quantifiers-converse2}
			\forall \eta' \in \Open{0}{\infty}, \exists B \in \Base, \forall \variable \in \aSet, \variable \in B \Rightarrow \displaystyle \mylimsup{N \to \infty} \anothernorm{\Process{\variable,N} - \mylim} \leqslant \eta' \, \, \text{\big ($\Pbb$-a.s\big )} \mycomma
		\end{equation}
	\end{onecol}
	and hence, according to \eqref{eq: elementary property of the infinity norm}, to:
	\begin{onecol}
		\begin{equation}
			\forall \eta' \in \Open{0}{\infty}, \exists B \in \Base, \forall \variable \in \aSet,
			\variable \in B \Rightarrow \Big \vert \displaystyle \mylimsup{N \to \infty} \anothernorm{\Process{\variable,N} - \mylim} \Big \vert_{\infty} \leqslant \eta' \myfstop
			\label{eq: limlimsup and quantifiers - intermediate implication for the converse}
		\end{equation}
	\end{onecol}
    Consider any $\eta \in \Open{0}{\infty}$ and any $\eta' \in \Open{0}{\eta}$. It follows from \eqref{eq: limlimsup and quantifiers - intermediate implication for the converse} that
	\begin{equation}
		\nonumber
		\exists B \! \in \! \Base, \forall \variable \! \in \! \aSet, \variable \! \in \! B \Rightarrow \Big \vert \mylimsup{N \to \infty} \anothernorm{\Process{\variable,N} - \mylim} \Big \vert_{\infty} \leqslant \eta' < \eta \mycomma
	\end{equation}
    which concludes the proof.
\end{proof}

\begin{Corollary}
	\label{corollary: esl with equivalent norms}
	If $\Process{}: \aSet \times \Nset \to \rv$ is a parametric discrete random process then, for any base $\Base$ in $\aSet$, any $\mylim \in \Rset^{\mydim}$ and any norm $\nu$ on $\Rset^{\mydim}$,
	\begin{onecol}
		\begin{equation}
			\label{eq: esl of a discrete process in our case (b)}
			\esssuplim{\Base}{N}\Process{\variable,N} = \mylim
			\, \, \, \myiff \, \, \limlimsupnu{\Base}{N}{\Process{\variable,N} - \mylim \, } = 0 \myfstop
            \nonumber
		\end{equation}
	\end{onecol}
\end{Corollary}
\begin{proof}
    By Definition \ref{Definition: lim of limsup}, $\esssuplim{\Base}{N} \Process{\variable,N} = \mylim$ is equivalent to
	\begin{equation}
		\label{eq: lim of limsup = 0}
		\limlimsup{\Base}{N}{\Process{\variable,N} - \mylim} = 0 \myfstop
	\end{equation}
    In the particular case of the Euclidean norm $\norm$, Proposition \ref{Proposition: Quanteur de lim of limsup} implies that \eqref{eq: lim of limsup = 0} is equivalent to \eqref{eq: limlimsup with quantifiers-0}. Thence the result by Proposition \ref{Proposition: Quanteur de lim of limsup}.
\end{proof}

Consider any $\Amat \in \GL$. We define a norm on $\Rset^\mydim$ by setting $\nu(\xvec) = \norm[\Amat \xvec \,]$ for each $\xvec \in \Rset^\mydim$. Therefore, we obtain the following statement from {Corollary} \ref{corollary: esl with equivalent norms}.

\begin{Proposition}
	\label{prop: esl with linear transform}
	If $\Process{}: \aSet \times \Nset \to \rv$ is a parametric discrete random process and $\Amat \in \GL$ then, for any base $\Base$ in $\aSet$ and any $\mylim \in \Rset^{\mydim}$,
	\begin{equation}
		%\label{eq: multidim ess sup lim}
		\esssuplim{\Base}{N} \Process{\variable,N} = \mylim \, \, \Leftrightarrow \, \, \esssuplim{\Base}{N} \Amat \Process{\variable,N} = \Amat \mylim \myfstop
		\nonumber
	\end{equation}
\end{Proposition}
The next two results are also direct consequences of Proposition \ref{Proposition: Quanteur de lim of limsup} with $\anothernormsh = \norm$.
%\vspace{0.1cm}
\begin{Proposition}
	\label{prop: esl = esl for each d}
	If $\Process{}: \! \aSet \times \Nset \to \rv$ is a parametric discrete random process and if, for any $\dimi \in \intd$, $\Process{}_\dimi: \aSet \times \Nset \to \fnrv{\Omega}{\Rset}$ is defined by:
	\[
	\begin{array}{llll}
		\Process{}_\dimi: & \aSet \times \Nset & \to & \fnrv{\Omega}{\Rset} 
		\vspace{0.1cm} \\
		& (\variable,N) & \mapsto & (\Process{\variable,N})_{\dimi}
	\end{array}
	\]
	then, for any base $\Base$ in $\aSet$ and any $\mylim \in \Rset^{\mydim}$,
	\begin{onecol}
		\begin{equation}
			%\label{eq: multidim ess sup lim}
			\esssuplim{\Base}{N} \Process{\variable,N} = \mylim \, \, \Leftrightarrow \, \, \forall \dimi \in \intd, \esssuplim{\Base}{N} 	\Process{}_\dimi(\variable,N) = \mylim_{\dimi} \myfstop
			\nonumber
		\end{equation}
	\end{onecol}
\end{Proposition}
\begin{Proposition}
	\label{prop: sums of ess sup lim}
	Consider any two parametric discrete random processes $\Processa{}, \Processb{}: \aSet \times \Nset \to \rv$.
	Given a base $\Base$ in $\aSet$, a pair $(\mylim,\mylim') \in \Rset^{\mydim} \times \Rset^{\mydim}$ and any ${(\alpha,\alpha')} \in \Rset$,
	\begin{equation}
		\begin{array}{lll}
			\hspace{-1cm} 
            \Big ( \, \big ( \, \esssuplim{\Base}{N} \Process{\variable,N} = \mylim \, \big ) \hspace{0.1cm} \mywedge \hspace{0.1cm} \big ( \, \esssuplim{\Base}{N} \Processb{\variable,N} = \mylim' \, \big ) \, \Big ) \vspace{0.2cm} \\
			\hspace{2cm} \Rightarrow \hspace{0.2cm} \esssuplim{\Base}{N} \left ( \alpha \Process{\variable,N} + {\alpha'} \Processb{\variable,N} \right ) = \alpha \mylim + {\alpha'} {\mylim'} \mycomma
		\end{array}
		\nonumber
	\end{equation}
	with
	$$\begin{array}{llll}
		\forall {\coeff} \in \Rset, {\coeff} \Process{\variable,N}: & \Omega & \to & \Rset^{\mydim} \\
		& \omega & \mapsto & {\coeff} \Process{\variable,N}(\omega) \myfstop
	\end{array}$$
\end{Proposition}

The following corollary of Proposition \ref{Proposition: Quanteur de lim of limsup} and \eqref{eq: lim over B(Delta)} is direct.

\begin{Corollary}[of Proposition \ref{Proposition: Quanteur de lim of limsup}]
	\label{Corollary: Quanteur de lim of limsup pour notre pb}
	Given any $(\Ctr,\Ctr_0) \in \Rset^\mydim \times \Rset^\mydim$, any $\ctrind \in \intk$, any parametric discrete random process  $\Process{}: \Setbis{\Ctr_0} \times \Nset \to \rv$,
	\begin{equation}
		\label{eq: the ess sup lim of a discrete random process}
		\esssuplim{\mydist{\CtrMap} \to \infty}{N}\Process{\CtrMap,N} = \Ctr
		\nonumber
	\end{equation} 
	if and only if,
	\begin{equation}
		\label{eq: limlimsup with quantifiers}
		\begin{array}{lll}
			\hspace{-0.7cm}
			\forall \eta \in \Open{0}{\infty}, \exists \, \radius \in \Open{0}{\infty}, \forall \CtrMap \in \Setbis{\Ctr_0}, \vspace{0.1cm} \\
			\hspace{0.75cm} 
            \mydist{\CtrMap} > \radius  \Rightarrow \exists \, \Omega_0 \in \tribu, %\vspace{0.1cm} \\
			\left \{		
			\begin{array}{lll} 
				\! \! \Pbb(\Omega_0) = 1 
				\vspace{0.1cm} \\ %\medskip \\ 
				\! \! \mywedge 
				\vspace{0.1cm} \\ %\medskip \\ 
				\! \! \forall \omega \in \Omega_0, \exists N_0 \in \Nset, 
				\vspace{0.1cm} \\ %\medskip \\
				\hspace{0.2cm} 
                \forall N \in \Nset, N > N_0 \Rightarrow \norm[\Process{\CtrMap,N}(\omega) - \Ctr \, ] < \eta \myfstop
			\end{array}
			\right.
		\end{array}
        \hspace{-1cm}
	\end{equation}
\end{Corollary}

\subsection{Proof of equivalence \texorpdfstring{$\myclubsuit$}{A} for diagonal positive-definite matrices}
\label{subsec: the case of diagonal covariance matrices}

\indent
From now on, consider a surjective sequence $\seq \! = \! (\seq_n)_{n \in \Nset}$ of integers valued in $\intk$, a function $w: \! \RightOpen{0}{\infty} \! \to \! \Open{0}{\infty}$ and two sequences $\MCov = (\MCov_n)_{n \in \Nset}$ and $\estMCov = (\estMCovn)_{n \in \Nset}$ of positive definite diagonal matrices satisfying \cref{Assumption: on alpha,Assumption: on w,Assumption: on C,Assumption: on lambda}. We set:
\begin{equation}
	\label{eq: deltamat = }	
	\forall n \in \Nset, \MCovn = \diag ( \, {\eigvalnone, \ldots, \eigvalndim} \, ) \myfstop
\end{equation}
We thus have:
\begin{equation}
	\label{eq: delta in interval}
	\forall n \in \Nset, \forall \dimi \in \intd, \sqrteigvalndimi \in [\sqrteigvalmin,\sqrteigvalmax] \myfstop
\end{equation}
Similarly, we put
\begin{equation}
	\forall n \in \Nset, \estMCovn = \diag (\,\esteigvalnone, \ldots, \esteigvalndim \, ) \myfstop
	\label{eq: estdeltamat = }
\end{equation}
and we have:
\begin{equation}
	\forall n \in \Nset, \forall \dimi \in \intd, \sqrtesteigvalndimi \in [\sqrtesteigvalmin,\sqrtesteigvalmax] \myfstop
	\label{eq: estdelta in estinterval }
\end{equation}
For any sequence $\Wvec = (\Wvec_n)_{n \in \Nset}$ of elements of $\rv$, any $N \in \Nset$, any $\Ctr \in \Rset^\mydim$ and any $\dimi \in \intd$, we set:
\begin{onecol}
	\begin{subequations}
		\label{eq: MyV and MyW-0}
		\begin{empheq}[left={} \empheqlbrace \hspace{-0.1cm} ]{align}
			& 
			\begin{array}{lll}
				\MyWTer(\Ctr,\Wvec) = \dfrac{1}{N} \displaystyle \sum_{n=1}^N w \big ( \Maha{\estMCovn}^2 (\Wvec_n - \Ctr) \big ) \sqrtesteigvalndimi^{\, -2} (\Wvec_n - \Ctr)_{\dimi} \mycomma
				\label{eq: MyV in Thm1-0} 
			\end{array}
			\\
			& 
			\begin{array}{lll}
				\MyVTer (\Ctr,\Wvec ) =
				\dfrac{1}{N} \displaystyle \sum_{n=1}^N w \big ( \Maha{\estMCovn}^2 (\Wvec_n - \Ctr) \big ) \sqrtesteigvalndimi^{\, -2} \myfstop
				\label{eq: MyW in Thm1-0}
			\end{array}
		\end{empheq}
	\end{subequations}
\end{onecol}
% \begin{twocol}
% 	\begin{subequations}
% 		\label{eq: MyV and MyW-0}
% 		\begin{empheq}[left={\hspace{-0.5cm}} \empheqlbrace \hspace{-0.2cm} ]{align}
% 			& 
% 			\begin{array}{lll}
% 				\MyWTer(\Ctr,\Wvec) 
% 				\vspace{0.1cm} \\
% 				\hspace{0.5cm}
% 				=  \dfrac{1}{N} \displaystyle \sum_{n=1}^N \hspace{-0.05cm} w \big ( \Maha{\estMCovn}^2 \hspace{-0.1cm} (\Wvec_n - \Ctr) \big ) \sqrtesteigvalndimi^{\, -2} (\Wvec_n - \Ctr)_{\dimi} 
% 				\label{eq: MyV in Thm1-0} 
% 			\end{array}
% 			\\
% 			& 
% 			\begin{array}{lll}
% 				\MyVTer (\Ctr,\Wvec ) =
% 				\dfrac{1}{N} \displaystyle \sum_{n=1}^N \hspace{-0.05cm} w \big ( \Maha{\estMCovn}^2 \hspace{-0.1cm} (\Wvec_n - \Ctr) \big ) \sqrtesteigvalndimi^{\, -2} 
% 				\hspace{-0.4cm}
% 				\label{eq: MyW in Thm1-0}
% 			\end{array}
% 		\end{empheq}
% 	\end{subequations}
% \end{twocol}
It follows from \eqref{eq: new h (a) - 0} \& \eqref{eq: MyV and MyW-0} that:
\begin{equation}
	\label{eq: new h (a) - 3}
	\TheFunc(\Ctr,\Wvec) - \Ctr = \left ( \dfrac{\MyWTer( \Ctr,\Wvec)}{\MyVTer( \Ctr,\Wvec)} \right )_{\! 1 \leqslant \dimi \leqslant \mydim} \myfstop
\end{equation}
According to \eqref{eq: esl centred}, when $N$ and $\mydist{\CtrMap}$ tend to $\infty$, the convergence in \esl~of $\TheFunc(\Ctr,\aObsSeq)$ to $\Ctr$ is equivalent to that of $\TheFunc(\Ctr,\aObsSeq) - \Ctr$ to $0$. Hence, \eqref{eq: new h (a) - 3} and Proposition \ref{prop: esl = esl for each d} show that, when $N$ and $\mydist{\CtrMap}$ tend to $\infty$, the convergence in \esl~of $\TheFunc(\Ctr,\aObsSeq)$ to $\Ctr$ boils down to that of each 
$\frac{\MyWTer( \Ctr,\aObsSeq)}{\MyVTer( \Ctr,\aObsSeq)}$ to $0$ for $1 \leqslant \dimi \leqslant \mydim$.
The proof thus articulates around the analyzis, given $\dimi \in \intd$, of $\MyWTer(\Ctr, \YObsSeqFull)$ and $\MyVTer( \Ctr, \YObsSeqFull )$, when $N$ and $\mydist{\CtrMap}$ tend to $\infty$. This analyzis is achieved through Lemmas \ref{step: approx a.s of V and W} to \ref{step: Expect of TNi}. These results are then combined to state Lemma \ref{step:ess.sup.lim tends to 0}. Corollary \ref{step:ess.sup.lim tends to 0 - corollary} of Lemma \ref{step:ess.sup.lim tends to 0} is then used to conclude the proof with the help of Lemma \ref{step:equivalence of limits} whose crucial role to achieve the proof can hardly be construed at this stage.

\vspace{0.1cm}
\indent
To proceed, we define
\begin{equation}
	\label{eq: indices of n in intn belonging to cluster k}
	\left \{
	\begin{array}{lll}
		\hspace{-0.1cm} \forall k \in \intk, \TheSet_{k} = \left \{ n \in \Nset : \seq_n = k) \right \} \mycomma \medskip \\
		\hspace{-0.1cm} \forall (k,N) \in \intk \times \NsetRestricted, \TheSet_{k,N} = \TheSet_{k} \cap \intn \myfstop
	\end{array}
	\right.
\end{equation}
and set also
%\vspace{-0.1cm}
\begin{equation}
	\hspace{-0.2cm} \forall (k,N) \! \in \! \intk \! \times \! \NsetRestricted, \thecoeff_{k,N} = \cardkN / N \geqslant \thecoeff_0 \mycomma
	\label{eq: inequality on alpha}
\end{equation}
%\begin{subequations}
%	\begin{empheq}[left={\forall (k,N) \in \intk \times \NsetRestricted, \,}\empheqlbrace]{align}
	%		& \, N_k = \cardkN \label{eq: N_k} \medskip \\
	%		& \, \thecoeff_{k,N} = {\cardkN}/{N} \geqslant \thecoeff_0 \label{eq: inequality on alpha}
	%	\end{empheq}	
%\end{subequations}
where the inequality in \eqref{eq: inequality on alpha} directly follows from \cref{Assumption: on alpha}.

\vspace{0.1cm}
\indent
Given any $\Ctr \in \Rset^{\mydim}$ and any sequence $\Wvec = (\Wvec_n)_{n \in \Nset}$ of $\mydim$-dimensional real random vectors, we also define, for any $k \in \intk$, any $N \in \NsetRestricted$ and any $\dimi \in \intd$,
\begin{onecol}
	\begin{subequations}
		\label{eq: S and T}
		\begin{empheq}[left={}\hspace{-0.2cm} \empheqlbrace]{align}
			& \MySBis(\Ctr,\Wvec) = \dfrac{1}{\cardkN} \displaystyle \sum_{n \in \TheSet_{k,N}} \hspace{-0.1cm} w \big ( \Maha{\estMCovn}^2(\Wvec_n - \Ctr) \big ) \, \sqrtesteigvalndimi^{ \, -2}
            \left ( \Wvec_n - \Ctr \right)_{\dimi} \mycomma
			\label{eq: S} \\
			& \MyTBis(\Ctr,\Wvec) = \dfrac{1}{\cardkN} \hspace{-0.25cm} \displaystyle \sum_{\quad n \in \TheSet_{k,N}} \hspace{-0.3cm} w \big ( \Maha{\estMCovn}^2(\Wvec_n - \Ctr) \big ) \, \sqrtesteigvalndimi^{\, -2} \myfstop
			\label{eq: T}
		\end{empheq}	
	\end{subequations}
\end{onecol}
% \begin{twocol}
% 	\begin{subequations}
% 		\label{eq: S and T}
% 		\begin{empheq}[left={}\hspace{-0.35cm} \empheqlbrace]{align}
% 			& 
% 			\begin{array}{lll}
% 				\hspace{-0.2cm}
% 				\MySBis(\Ctr, \Wvec) = \vspace{0.1cm} \\
% 				%				\hspace{0.1cm} 
% 				\dfrac{1}{\cardkN} \! \displaystyle \sum_{n \in \TheSet_{k,N}} \! \! \! \! w \big ( \Maha{\estMCovn}^2(\Wvec_n - \Ctr) \big ) \invsqrtesteigvalndimi \left ( \Wvec_n - \Ctr\right)_{\dimi}
% 			\end{array}
% 			\label{eq: S} \\
% 			& 
% 			\begin{array}{lll}
% 				%				\hspace{-0.2cm}
% 				\MyTBis(\Ctr,\Wvec) = \vspace{0.1cm} \\
% 				\hspace{0.1cm}
% 				\dfrac{1}{\cardkN} \! \displaystyle \sum_{n \in \TheSet_{k,N}} \! \! \! \!  w \big ( 	\Maha{\estMCovn}^2(\Wvec_n - \Ctr) \big ) \invsqrtesteigvalndimi 
% 			\end{array}
% 			\label{eq: T}
% 		\end{empheq}	
% 	\end{subequations}
% \end{twocol}
With this notation, we have:
\begin{subequations}
	\label{eq: MyV and MyW-4}
	\begin{empheq}[left={\forall \dimi \in \intd, } \empheqlbrace]{align}
		& \hspace{0.1cm} \MyWTer(\Ctr, \Wvec) = \displaystyle \sum_{k=1}^K \thecoeff_{k,N} \MySBis(\Ctr, \Wvec) \mycomma
		\label{eq: MyV and MyW-4a} \\
		& \hspace{0.1cm} \MyVTer(\Ctr, \Wvec) = \displaystyle \sum_{k=1}^K \thecoeff_{k,N} \MyTBis(\Ctr, \Wvec) \myfstop
		\label{eq: MyV and MyW-4b}
	\end{empheq}	
\end{subequations}

To ease the reading, we remind the reader that $\Set$, $\family$, and $\MVM$, are defined by Definitions \ref{def: sequence of centroids}, \ref{def: family F} and \ref{Definition: CtrMap}, respectively.

\begin{Lemma}
	\label{step: approx a.s of V and W}
	Given any $\YObsSeq \in \MVM$, any $\CtrMap \in \Set$, any $\Ctr \in \Rset^{\mydim}$ and any $\dimi \in \intd$, we have \text{\big ($\Pbb$-a.s\big )}:
	\begin{subequations}
		\begin{empheq}[left= {} \hspace{-0.45cm} \empheqlbrace]{align}
			& \displaystyle \lim_{N \to \infty} \left ( \MyWTer(\Ctr,\YObsSeqFull) - \Expect{\MyWTer(\Ctr,\YObsSeqFull)} \right ) = 0 \mycomma \nonumber \medskip \\
			& \displaystyle \lim_{N \to \infty} \left ( \MyVTer(\Ctr,\YObsSeqFull) - \Expect{\MyVTer(\Ctr,\YObsSeqFull)} \right ) = 0 \myfstop \nonumber
		\end{empheq}
	\end{subequations}
\end{Lemma}
\begin{proof}
	For any $\YObsSeq \in \MVM$, any $\CtrMap \in \Set$, any $\Ctr \in \Rset^{\mydim}$, any $\dimi \in \intd$, any $\omega \in \Omega$ and any $N \in \Nset$, it follows from \eqref{eq: MyV and MyW-4} that:
	\begin{subequations}
		\label{eq: upper-bound on Wn - E(Wn) and Vn - E(Vn)}
		\begin{empheq}[left={} \empheqlbrace]{align}
			& \begin{array}{lll}
				\big \vert \, \MyWTer(\Ctr,\YObsSeqFull) (\omega) - \Expect{\, \MyWTer(\Ctr,\YObsSeqFull) \, } \big \vert  \vspace{0.1cm} \\
				\hspace{1cm}
				\leqslant \displaystyle \sum_{k = 1}^{K} \thecoeff_{k,N} \big \vert \, \MySBis(\Ctr,\YObsSeqFull)(\omega) - \Expect{\, \MySBis(\Ctr,\YObsSeqFull) \, } \big \vert \mycomma
				\vspace{0.2cm}
			\end{array}
			\label{eq: upper-bound on Wn - E(Wn)} \\
			& \begin{array}{lll}
				\hspace{-0.1cm} 
				\big \vert \, \MyVTer(\Ctr, \YObsSeqFull) (\omega) - \Expect{\, \MyVTer(\Ctr,\YObsSeqFull) \, } \big \vert \vspace{0.1cm} \\
				\hspace{1cm}
				\leqslant \displaystyle \sum_{k = 1}^{K} \thecoeff_{k,N} \big \vert \, \MyTBis(\Ctr,\YObsSeqFull)(\omega) - \Expect{\, \MyTBis(\Ctr,\YObsSeqFull) \, } \big \vert \mycomma
			\end{array}
			\label{eq: upper-bound on Vn - E(Vn)}
		\end{empheq}	
	\end{subequations}
	According to \citet[p.~388, Theorem 2, Kolmogorov]{Shiryaev2006} or \citet[p.~388, Theorem 1, Cantelli]{Shiryaev2006}, and since \eqref{eq: inequality on alpha} implies that $\cardkN$ grows to $\infty$ when $N$ grows itself to $\infty$, we can write that, for any $k \in \intk$,
	\begin{equation}
		\label{eq: Kolmogorov}
		\displaystyle \lim_{N \to \infty} \left ( \MySBis(\Ctr,\YObsSeqFull) - \Expect{\MySBis(\Ctr,\YObsSeqFull)} \right ) = 0 \, \, \text{\big ($\Pbb$-a.s\big )} \myfstop
		\nonumber
	\end{equation}
	Thus, for any $k \in \intk$, there exists $\Omega'_{k} \in \tribu$ such that $\Pbb(\Omega'_{k}) = 1$ and, for any $\omega \in \Omega'_{k}$,
	\begin{equation}
		\label{eq: Kolmogorov in extenso}
		\displaystyle \lim_{N \to \infty} \left ( \MySBis(\Ctr,\YObsSeqFull)(\omega) -\Expect{\MySBis(\Ctr,\YObsSeqFull)} \right ) = 0 \myfstop
	\end{equation}
	Since $\Pbb(\Omega') = 1$ with $\Omega' = \bigcap_{k=1}^K \Omega'_{k}$, the foregoing establishes the existence of $\Omega' \in \tribu$ such that $\Pbb(\Omega') = 1$ and such that \eqref{eq: Kolmogorov in extenso} holds true for any $k \in \intk$ and any $\omega \in \Omega'$.
	Hence, for any $k \in \intk$, any $\omega \in \Omega'$ and any $\truc \in \Open{0}{\infty}$, there exists $\intmin_{k,\omega,\truc} \in \Nset$ such that:
	\begin{onecol}
		\begin{equation}
			\forall N \in \Nset, N > \intmin_{k,\omega,\truc} 
			\Rightarrow \left \vert \MySBis(\Ctr,\YObsSeqFull)(\omega) - \Expect{\MySBis (\Ctr,\YObsSeqFull)} \right \vert < \truc \myfstop
			\label{eq: upper-bound on each term of the bound on Wn - E(Wn)-0}
			%		\nonumber
		\end{equation}
	\end{onecol}
	From \eqref{eq: upper-bound on Wn - E(Wn)}, \eqref{eq: upper-bound on each term of the bound on Wn - E(Wn)-0} and the definition of $\thecoeff_{k,N}$ given by \eqref{eq: inequality on alpha}, we conclude to the existence of $\Omega' \in \tribu$ with $\Pbb(\Omega') = 1$ such that, for any $\omega \in \Omega'$ and any $\truc \in \Open{0}{\infty}$, there exists $N'_{\omega,\truc} = \max \big \{ \intmin_{k,\omega,\truc}: k \in \intk \big \} \in \Nset$ for which we have
	\begin{onecol}
		\begin{equation}
			\forall N \in \Nset, N > N'_{\omega,\truc} \Rightarrow \left \vert \MyWTer(\Ctr,\YObsSeqFull) (\omega) - \Expect{\MyWTer(\Ctr,\YObsSeqFull)} \right \vert < \truc \myfstop
			\label{eq: 1st convergence}
			%\nonumber
		\end{equation}
	\end{onecol}
	By a similar reasoning left to the reader, we establish the existence of $\Omega'' \in \tribu$ such that $\Pbb(\Omega'') = 1$ and such that, for any $\omega \in \Omega''$ and any $\truc \in \Open{0}{\infty}$, there exists $N''_{\omega,\truc} \in \Nset$ for which we have
	\begin{onecol}
		\begin{equation}
			\forall N \in \Nset, N > N''_{\omega,\truc} \Rightarrow \left \vert \MyVTer(\Ctr,\YObsSeqFull) (\omega) - \Expect{\MyVTer(\Ctr,\YObsSeqFull)} \right \vert < \truc \myfstop
			\label{eq: 2nd convergence}
			%\nonumber
		\end{equation}
	\end{onecol}
	\eqref{eq: 1st convergence} and \eqref{eq: 2nd convergence} establish the existence of $\Omega_0 = \Omega' \cap \Omega'' \in \tribu$ such that $\Pbb(\Omega_0) =1$ and such that, for any $\omega \in \Omega_0$ and any $\truc \in \Open{0}{\infty}$, there exists $N_0 = \max \{ N'_{\omega,\truc}, N''_{\omega,\truc} \} \in \Nset$ such that
	\begin{onecol}
		\begin{equation}
			\forall N \in \Nset, N > N_0 \Rightarrow 
			\left \{ 
			\begin{array}{lll}
				\hspace{-0.1cm} \left \vert \, \MyWTer(\Ctr,\YObsSeqFull) (\omega) - \Expect{\MyWTer(\Ctr,\YObsSeqFull)} \, \right \vert < \truc \mycomma
				\medskip \\
				\hspace{-0.1cm} \left \vert \, \MyVTer(\Ctr,\YObsSeqFull) (\omega) - \Expect{\MyVTer(\Ctr,\YObsSeqFull)} \, \right \vert < \truc \myfstop
			\end{array}
			\right.
			\nonumber
		\end{equation}
	\end{onecol}
	Thence the result.
\end{proof}
\indent
For all $\YObsSeq \in \MVM$, $\CtrMap \in \Set$, $\Ctr \in \Rset^{\mydim}$ and $\dimi \in \intd$, the result above provides us with the behavior of $\MyWTer(\Ctr,\YObsSeqFull) - \Exp{\MyWTer(\Ctr,\YObsSeqFull)}$ and  $\MyVTer(\Ctr,\YObsSeqFull) - \Exp{\MyVTer(\Ctr,\YObsSeqFull)}$ when $N$ tends to $\infty$. In what follows, we address the behaviour of $\Exp{\MyWTer(\Ctr,\YObsSeqFull)}$ and $\Exp{\MyVTer(\Ctr,\YObsSeqFull)}$ when $\mydist{\CtrMap} \to \infty$.

\vspace{0.1cm}
\indent
As of now, given $\Ctr \in \Rset^{\mydim}$ and $\radius \in \RightOpen{0}{\infty}$, $\ball(\Ctr,\radius) = \{ \xvec \in \Rset^{\mydim}: \| \xvec-\Ctr \| \leqslant \radius \}$ denotes the (closed) ball centred at $\Ctr$ with radius $\radius$. In addition, $\ball(\Ctr,\radius)^c$ designates the complementary of $\ball(\Ctr,\radius)$ in $\Rset^{\mydim}$. To proceed, we need the following reminders and preliminary results.

\vspace{0.1cm}
\noindent
1) According to \eqref{eq: definition of the initial mahanorm}, the Mahalanobis norms $\Maha{\MCov_n}$ and $\Maha{\estMCovn}$, respectively associated with the diagonal matrices $\MCov_n$ and $\estMCovn$, are given by
\begin{subequations}
	\begin{empheq}[left={\forall n \in \Nset, \forall \xvec \in \Rset^\mydim,} \empheqlbrace]{align}
		& \Maha{\MCov_n}(\xvec) = \norm[\white \xvec \,] \mycomma
		\label{eq: Mahalanobis norm associated with MCov_n - diagonal case}
		\\
		& \Maha{\estMCovn}(\xvec) = \norm[\invsqrtestMCovn\xvec \,] \myfstop
		\label{eq: estimated Mahalanobis norm - relation with euclidean - diagonal case}
	\end{empheq}
\end{subequations}
Since we have set \eqref{eq: deltamat = } \& \eqref{eq: estdeltamat = }, we also have:
\begin{subequations}
	\begin{empheq}[left={\forall n \in \Nset, \forall \dimi \in \intd, \forall \xvec \in \Rset^\mydim,} \empheqlbrace]{align}
		& \big ( \MCov_n^{-1/2} \xvec \big )_{\dimi} = {\sqrteigval_{n,\dimi}^{-1}} \, x_i \mycomma
		\label{eq: coordinate of MCov_n(-1/2) x}
		\\
		& \big (\invsqrtestMCovn\xvec \big )_{\dimi} = \invsqrtesteigvalndimi \, x_i \myfstop
		\label{eq: coordinate of estMCov_n(-1/2) x}
	\end{empheq}
\end{subequations}

\vspace{0.1cm}
\noindent
2) We derive from \eqref{eq: property of Squel}, \eqref{eq: deltamat = } \& \eqref{eq: estdeltamat = } that:
\begin{onecol}
	\begin{equation}
		\label{eq: dist of Zn(xi)}
		\begin{array}{lll}
			\forall \, \YObsSeq \in \MVM, \forall \CtrMap \in \Set, \forall \Ctr \in \Rset^{\mydim}, \forall n \in \Nset, 
			\vspace{0.1cm} \\
			\hspace{1cm} 
			\invsqrtestMCovn(\SquelY{\CtrMap}{n} - \Ctr) \thicksim \Ncal\left (\invsqrtestMCovn(\CtrMap(\seq_n) - \Ctr), \diag \left ( {\dfrac{\eigval_{n,1}}{\esteigvalnone}, \ldots, \dfrac{\eigval_{n,\mydim}}{\esteigvalndim}} \right )
			\right ) \myfstop
		\end{array}
	\end{equation}
\end{onecol}
\noindent
3) It is also convenient to introduce the following definition.
\begin{Definition}
	\label{def: Ndiagcov}
	Given any interval $[{\lambdamin^2},{\lambdamax^2}] \subset \Open{0}{\infty}$, $\fnGdiag{{\lambdamin^2}}{{\lambdamax^2}}$ is the set of all Gaussian-distributed $\Xvec \in \rv$ whose covariance matrix is diagonal with diagonal elements in $[{\lambdamin^2},{\lambdamax^2}]$.
\end{Definition}

\vspace{0.1cm}
With this definition, it follows from \eqref{eq: delta in interval}, \eqref{eq: estdelta in estinterval } \& \eqref{eq: dist of Zn(xi)} that
\begin{equation}
	\label{eq: dist of Zn(xi)-prop}
	\begin{array}{lll}
		\forall \, \YObsSeq \in \MVM, \forall \CtrMap \in \Set,
		\forall \Ctr \in \Rset^{\mydim}, \forall n \in \Nset, 
		\vspace{0.1cm} \\ 
		\hspace{3cm} 
		\invsqrtestMCovn(\SquelY{\CtrMap}{n} - \Ctr) \in  \fnGdiag{{\lambdamin^2}}{{\lambdamax^2}} \mycomma
	\end{array}
\end{equation}
where we henceforth put
\begin{equation}
	\label{eq: sigmamin and sigmamax}
	{\lambdamin^2} = {\eigvalmin \, / \, \esteigvalmax} \quad \& \quad {\lambdamax^2} = {\eigvalmax \, / \, \esteigvalmin} \myfstop
\end{equation}
Note also that, according to \eqref{eq: coordinate of estMCov_n(-1/2) x}, we have the next inequality:
\begin{equation}
	\label{eq: inequality on the mean}
	\hspace{-0.2cm}
	\forall n \in \Nset, \forall \xvec \in \Rset^\mydim, {\invsqrtesteigvalmax} \| \xvec \| \leqslant \| \invsqrtestMCovn\xvec \| \leqslant {\invsqrtesteigvalmin} \| \xvec \| \myfstop
\end{equation}

The next lemma provides approximations of $\Exp{\MySBis(\Ctr,\YObsSeqFull)}$ and $\Exp{\MyTBis(\Ctr,\YObsSeqFull)}$ defined by \eqref{eq: S and T}, in function of $\| \Ctr - \CtrMap(k) \|$. This behavior is rewritten in Corollary \ref{step: of behaviors of E[S] and E[T]} in function of $\mydist{\CtrMap}$.

%\vspace{0.1cm}
\begin{Lemma}
	\label{lemma: behaviors of E[S] and E[T]}
	For any $\machin \in \Open{0}{\infty}$, there exists $\radius \in \Open{0}{\infty}$ such that:
	\begin{equation}
		\begin{array}{lll}
			\label{eq: behavior of E[SNk] and T[SNk]}
			\forall (k,N) \in \intk \times \NsetRestricted, \forall \dimi \in \intd, \forall \, \YObsSeq \in \MVM, \vspace{0.2cm} \\
			\hspace{2cm} 
			\forall \CtrMap \in \Set, \forall \Ctr \in \ball(\CtrMap(k) , \radius)^c, 
			\left \{
			\begin{array}{llll}
				\! \! \! \left \vert \, \Expect{\MySBis(\Ctr,\YObsSeqFull)} \, \right \vert < \machin
                \vspace{0.2cm} \\
                \! \! \mywedge
                \vspace{0.2cm} \\
				\! \! \! \Expect{\MyTBis(\Ctr,\YObsSeqFull)} < \machin \myfstop
			\end{array}
			\right.
		\end{array}
	\end{equation}
\end{Lemma}
\begin{proof}
	Consider any $\machin \in \Open{0}{\infty}$. Since $w$ is continuous and $\displaystyle \lim\limits_{t \rightarrow \infty} {t w(t)} = 0$ by \cref{Assumption: on w}, the function $t \in \RightOpen{0}{\infty} \mapsto t w(t) \in \RightOpen{0}{\infty}$ is bounded. Therefore, $\displaystyle \lim\limits_{t \rightarrow \infty} w(t) = 0$ as well. We thus apply Lemma \ref{Lemma: asymptotic behavior-2} of Appendix \ref{App: Instrumental lemmas} to the functions $t \mapsto t {w(t^2)}$ and $t \mapsto {w(t^2)}$ to derive the existence of $\radius' \in \Open{0}{\infty}$ such that, for any $\Xvec \in \fnGdiag{{\lambdamin^2}}{{\lambdamax^2}}$ satisfying $\norm[ \, \Expect{\Xvec}] > \radius'$,
	\begin{equation}
		\left \{
		\begin{array}{lll}
			\! \! \left \| \Expect{w( \norm[\Xvec]^2) \Xvec} \right \|  \leqslant \Expect{w( \norm[\Xvec]^2) \| \Xvec \| } < \machin \sqrtesteigvalmin \mycomma
			\vspace{0.2cm} \\
			\! \! \Expect{w( \norm[\Xvec]^2)} < {\machin} \esteigvalmin \myfstop
		\end{array}
		\right.
		\label{eq: behavior of E[wc(Zkn)Zkn] and E[wc(Zkn)]_v1}
	\end{equation}
    where the slant inequality above is standard by Jensen's inequality applied to $\norm$.
    \vspace{0.1cm} \\
    \indent
	Set $\radius = \radius' \times \sqrtesteigvalmax$ and fix arbitrarily $ (k,N) \in \intk \times \NsetRestricted$, $\CtrMap \in \Set$, $\Ctr \in \ball(\CtrMap(k),\radius)^c$, $\YObsSeq \in \MVM$ and $n \in \TheSet_{k,N}$. We deduce from \eqref{eq: indices of n in intn belonging to cluster k}, \eqref{eq: dist of Zn(xi)} and the first inequality in \eqref{eq: inequality on the mean} that 
	\begin{equation}
		\label{eq: behavior of the transformed measurement vector}
		\hspace{-0.2cm} \big \| \Ebb [ \, \invsqrtestMCovn(\SquelY{\CtrMap}{n}-\Ctr) \, ] \big \| = \big \| \invsqrtestMCovn(\CtrMap(k)-\Ctr) \big \| > \radius'
	\end{equation}
	and thus, from \eqref{eq: dist of Zn(xi)-prop}, \eqref{eq: behavior of E[wc(Zkn)Zkn] and E[wc(Zkn)]_v1} \& \eqref{eq: behavior of the transformed measurement vector} that:
	\begin{onecol}
    \begin{subequations}
    	\begin{empheq}[left={} \empheqlbrace]{align}
		& \big \| \, \Expect{ \, w( \| \invsqrtestMCovn(\SquelY{\CtrMap}{n} - \Ctr) \|^2) \, \invsqrtestMCovn(\SquelY{\CtrMap}{n} - \Ctr) } \big \| < \machin \sqrtesteigvalmin \mycomma
		\label{eq: behavior of E[wc(Zkn)Zkn] and E[wc(Zkn)]_v2 (a)}
		\\
		& \Expect{ \, w( \| \invsqrtestMCovn(\SquelY{\CtrMap}{n} - \Ctr) \|^2) \, } <  {\machin} \esteigvalmin \myfstop
		\label{eq: behavior of E[wc(Zkn)Zkn] and E[wc(Zkn)]_v2 (b)}
    	\end{empheq}
        \label{eq: behavior of E[wc(Zkn)Zkn] and E[wc(Zkn)]_v2}
    \end{subequations}
	\end{onecol}
	% \begin{twocol}
	% 	\begin{equation}
	% 		\hspace{-0.1cm}
	% 		\left \{
	% 		\begin{array}{lll}
	% 			\hspace{-0.35cm} 
	% 			\begin{array}{lll}
	% 				\big \| \, \Expect{w( \| \invsqrtestMCovn(\SquelY{\CtrMap}{n} - \Ctr) \|^2) 
	% 					\vspace{0.1cm} \\ 
	% 					\hspace{2cm} 
	% 					\times (\invsqrtestMCovn(\SquelY{\CtrMap}{n} - \Ctr))} \big \| < \machin
	% 			\end{array}
	% 			\vspace{0.2cm} \\
	% 			\hspace{-0.1cm} \Expect{ w( \| \invsqrtestMCovn(\SquelY{\CtrMap}{n} \! - \! \Ctr) \|^2)} <  {\machin} \modif{\esteigvalmin} 
	% 		\end{array}
	% 		\right.
	% 		\label{eq: behavior of E[wc(Zkn)Zkn] and E[wc(Zkn)]_v2}
	% 		\vspace{0.1cm}	
	% 	\end{equation}
	% \end{twocol}
	For any $\dimi \in \intd$, \eqref{eq: coordinate of estMCov_n(-1/2) x} induces that
	$$\big ( \invsqrtestMCovn(\SquelY{\CtrMap}{n} - \Ctr) \big )_\dimi = \invsqrtesteigvalndimi
    \big ( \SquelY{\CtrMap}{n} - \Ctr \big )_\dimi$$
    and thereby that:
    \begin{align}
    \big \vert \, \Ebb \, \big [ \, w( \| & \invsqrtestMCovn (\SquelY{\CtrMap}{n} - \Ctr) \|^2) \, \invsqrtesteigvalndimi \big ( \SquelY{\CtrMap}{n} - \Ctr \big )_\dimi \, ] \, \big \vert \vspace{0.1cm} \nonumber \\
    & = 
    \big \vert \,  \Expect{\, w( \| \invsqrtestMCovn(\SquelY{\CtrMap}{n} - \Ctr) \|^2) \big ( \invsqrtestMCovn(\SquelY{\CtrMap}{n} - \Ctr) \big )_\dimi  \, } \, \big \vert\vspace{0.1cm} \nonumber \\
    & \leqslant  \big \| \, \Expect{ \, w( \| \invsqrtestMCovn(\SquelY{\CtrMap}{n} - \Ctr) \|^2) \, \invsqrtestMCovn(\SquelY{\CtrMap}{n} - \Ctr) } \big \| \myfstop
    \nonumber
    \end{align}
	It results from the inequality above, \eqref{eq: behavior of E[wc(Zkn)Zkn] and E[wc(Zkn)]_v2} and \eqref{eq: estdelta in estinterval } that, for all $\dimi \in \intd$,
	\begin{onecol}
		\begin{equation}
			\label{eq: behavior of E[S] and E[T]}
%			\forall \dimi \in \intd,
			\left \{
			\hspace{-0.1cm} 
			\begin{array}{lll}
            \begin{array}{lll}
				\Big \vert \, \Ebb \big [ \, w \big( \| \invsqrtestMCovn(\SquelY{\CtrMap}{n} - \Ctr) \|^2 \, \big ) \, \sqrtesteigvalndimi^{\, -2} \, \big ( \SquelY{\CtrMap}{n} - \Ctr \big )_\dimi \, \big ] \, \Big \vert \vspace{0.15cm} \\
                \hspace{2cm} = \sqrtesteigvalndimi^{\, -1} \big \vert \, \Ebb \big [ \, w \big( \| \invsqrtestMCovn(\SquelY{\CtrMap}{n} - \Ctr) \|^2 \, \big ) \, \sqrtesteigvalndimi^{\, -1} \, \big ( \SquelY{\CtrMap}{n} - \Ctr \big )_\dimi \, \big ] \, \big \vert
                \vspace{0.2cm} \\
                \hspace{2cm} \leqslant \sqrtesteigvalndimi^{\, -1} \, \big \| \, \Expect{ \, w( \| \invsqrtestMCovn(\SquelY{\CtrMap}{n} - \Ctr) \|^2) \, \invsqrtestMCovn(\SquelY{\CtrMap}{n} - \Ctr) \big \| }
                \vspace{0.2cm} \\
                \hspace{2cm} < \sqrtesteigvalndimi^{\, -1} \, \sqrtesteigvalmin \, \machin
                \vspace{0.2cm} \\
                \hspace{2cm} < \machin
                \end{array}
				\vspace{0.2cm} \\
				\Expect{ \, w \big ( \, \| \invsqrtestMCovn(\SquelY{\CtrMap}{n} \! - \! \Ctr) \|^2 \big ) \, } \sqrtesteigvalndimi^{\, -2} < \machin \myfstop
			\end{array}
			\right.
			\nonumber
		\end{equation}
	\end{onecol}
	% \begin{twocol}
	% 	\begin{equation}
	% 		\label{eq: behavior of E[S] and E[T]}
	% 		%	\forall n \! \in \! \TheSet_{k,N},
	% 		\left \{
	% 		\hspace{-0.1cm} 
	% 		\begin{array}{lll}
	% 			\Big \vert \, \Ebb \big [ w \big ( \, \| \invsqrtestMCovn(\SquelY{\CtrMap}{n}-\Ctr) \|^2 \, \big ) \vspace{0.1cm} \\
	% 			\hspace{2.5cm} 
	% 			\times \, \invsqrtestDiagmatndimi \big ( \SquelY{\CtrMap}{n} \! - \! \Ctr \big )_\dimi \, \big ] \, \Big \vert < \machin
	% 			\vspace{0.1cm} \\
	% 			\Expect{ w \big ( \, \| \invsqrtestMCovn(\SquelY{\CtrMap}{n} \! - \! \Ctr) \|^2 \big ) } \modif{\sqrtesteigvalndimi^{\, -2}} < \machin
	% 		\end{array}
	% 		\right.
	% 		\nonumber
	% 	\end{equation}
	% \end{twocol}
	Whence the result by \eqref{eq: estimated Mahalanobis norm - relation with euclidean - diagonal case} \& \eqref{eq: S and T}, where we take $\Wvec = \YObsSeqFull$ and $W_n = \SquelY{\CtrMap}{n}$.
\end{proof}

For what follows, we recall that $\Setbis{\Ctr_0}$, $\MVMbis$, and $\mydist{\CtrMap}$, are defined by Definitions \ref{definition: subset of injective maps}, \ref{definition: definition of the subset of interest}, and \ref{definition: the distance}, respectively.

\begin{Corollary}
	\label{step: of behaviors of E[S] and E[T]} For any $\machin \in \Open{0}{\infty}$ and any pair $(\Ctr_0,\Ctr) \in \Rset^{\mydim} \times \Rset^{\mydim}$, there exists $\radiusb \in \Open{0}{\infty}$ such that
	\begin{onecol}
		\begin{equation}
			\label{eq: corollary of behavior of E[SNk] and T[SNk]}
			\begin{array}{lll}
				\forall\dimi \in \intd, \forall \ctrind \in \intk, \forall \, \YObsSeq \in \MVMbis, \forall \CtrMap \in \Setbis{\Ctr_0}, 
                \vspace{0.25cm} \\
				\hspace{1cm}\mydist{\CtrMap} > \radiusb \Rightarrow
				\forall (k,N) \in \big ( \intk\setminus \{\ctrind\} \big ) \times \NsetRestricted, 
				\left \{
				\begin{array}{llll}
					\hspace{-0.2cm} \left \vert \Expect{\MySBis(\Ctr,\YObsSeqFull)} \right \vert < \machin \vspace{0.1cm} \\
                    \mywedge
                    \vspace{0.1cm} \\					\hspace{-0.15cm} \Expect{\MyTBis(\Ctr,\YObsSeqFull)} < \machin \myfstop
				\end{array}
				\right.
			\end{array}
			\nonumber
		\end{equation}
	\end{onecol}
	% \begin{twocol}
	% 	\begin{equation}
	% 		\label{eq: corollary of behavior of E[SNk] and T[SNk]}
	% 		\begin{array}{llll}
	% 			\hspace{-0.44cm} \forall \ctrind \in \intk, \forall \, \YObsSeq \in \MVMbis,  \vspace{0.1cm} \\
	% 			%			\hspace{0.4cm} 
	% 			\forall \CtrMap \in \Setbis{\Ctr_0}, \mydist{\CtrMap} > \radiusb \Rightarrow
	% 			\vspace{0.1cm} 
	% 			\\
	% 			\hspace{0.5cm} 
	% 			\forall (k,N) \! \in \! \big ( \intk \! \setminus \! \{\ctrind\} \big ) \! \times \! \NsetRestricted, 
	% 			\! \vspace{0.1cm} \\
	% 			\hspace{1cm} 
	% 			\left \{
	% 			\begin{array}{llll}
	% 				\hspace{-0.2cm} \left \vert \Expect{\MySBis(\Ctr,\YObsSeqFull)} \right \vert \! < \! \machin \vspace{0.1cm} \\
	% 				\hspace{-0.2cm} \Expect{\MyTBis(\Ctr,\YObsSeqFull)} \! < \! \machin
	% 			\end{array}
	% 			\right.
	% 		\end{array}
	% 		\nonumber
	% 	\end{equation}
	% \end{twocol}
\end{Corollary}
\begin{proof}
	Let $\machin \in \Open{0}{\infty}$. Lemma \ref{lemma: behaviors of E[S] and E[T]} implies the existence of $\radius \in \Open{0}{\infty}$ such that \eqref{eq: behavior of E[SNk] and T[SNk]} holds true. Given any $(\Ctr_0,\Ctr) \! \in \! \Rset^{\mydim} \times \Rset^{\mydim}$, set $\radiusb = \radius + \| \Ctr-\Ctr_0 \|$. Consider arbitrary $\ctrind \in \intk$, $\YObsSeq \in \MVMbis$ and $\CtrMap \in \Setbis{\Ctr_0}$. If $\mydist{\CtrMap} > \radiusb$ then, for any $k \in \intk \setminus \{\ctrind\}$, $\norm[\CtrMap(k) - \Ctr_0] > \radiusb$, which implies that $\norm[\CtrMap(k) - \Ctr] \geqslant \norm[\CtrMap(k) - \Ctr_0] - \norm[\Ctr-\Ctr_0] > r$. Thence the result as a consequence of \eqref{eq: behavior of E[SNk] and T[SNk]}.
\end{proof}
\indent
For $\Ctr \in \Rset^{\mydim}$, $N \in \NsetRestricted$, $\YObsSeq \in \MVMbis$ and $\CtrMap \in \Setbis{\Ctr_0}$, it follows from \eqref{eq: MyV and MyW-4} that
\begin{equation}
	\left \{ 
	\begin{array}{lll} 
		\! \! \Exp{\MyWTer(\Ctr,\YObsSeqFull)} = \sum_{k = 1}^{K} \thecoeff_{k,N} \, \Exp{\MySBis(\Ctr,\YObsSeqFull)} \mycomma
		\vspace{0.1cm} \\ 
		\! \! \Exp{\MyVTer(\Ctr,\YObsSeqFull)} = \sum_{k = 1}^{K} \thecoeff_{k,N} \, \Exp{\MyTBis(\Ctr,\YObsSeqFull)} \myfstop
	\end{array}
	\right.
	\nonumber
\end{equation}
{According to Lemma \ref{step: approx a.s of V and W} and Corollary \ref{step: of behaviors of E[S] and E[T]}, the equalities above} suggest that, when $N$ and $\mydist{\CtrMap}$ are both large enough, $\MyWTer(\Ctr,\YObsSeqFull)$ and $\MyVTer(\Ctr,\YObsSeqFull)$ can be approximated by $\thecoeff_{\ctrind,N} \Exp{\Sbm_{\ctrind,N,i}(\Ctr,\YObsSeqFull)}$ and $\thecoeff_{\ctrind,N} \Exp{T_{\ctrind,N,i}(\Ctr,\YObsSeqFull)}$, respectively. This heuristic is formalized by Corollary \ref{step:ess.sup.lim tends to 0 - corollary} below, according to which 
$$\dfrac{\MyWTer(\Ctr,\YObsSeqFull) - \thecoeff_{\ctrind,N} \Expect{S_{\ctrind,N,i}(\Ctr,\YObsSeqFull)}}{\MyVTer(\Ctr,\YObsSeqFull)}$$ vanishes when $N$ and $\mydist{\CtrMap}$ tend to $\infty$. To establish Lemma \ref{step:ess.sup.lim tends to 0}, the next result will be instrumental. In this lemma and the rest of the proof, we set
\begin{equation}
	\label{eq: the expect}
	\begin{array}{lll}
		\theexpect(\radius) = \dfrac{1}{(2 \pi)^{\mydim/2} {\lambdamax^\mydim}} \displaystyle \int_{\Rset^\mydim} w \left ( \left ( \norm[\xvec] + \radius \right )^2 \right ) e^{- \norm[\xvec]^2/2 {\lambdamin^2}} \, \dxvec \myfstop
	\end{array}
%	\nonumber
\end{equation}
Note that $\theexpect$ does not increase since $w$ does not.

\begin{Lemma}
	\label{step: Expect of TNi}
	Given $\YObsSeq \in \MVM$, $\CtrMap \in \Set$, $\radius \in \RightOpen{0}{\infty}$, $\Ctr \in \ball(\CtrMap(k),\radius)$, $(k,N) \in \intk \times \NsetRestricted$ and $\dimi \in \intd$,
	\begin{equation}
		\label{eq: lower bound on expectation of T}
		\Expect{\MyTBis(\Ctr,\YObsSeqFull)} \geqslant 
		\theexpect \Big ( \invsqrtesteigvalmin \radius \Big ) \, \investeigvalmax > 0 \myfstop
	\end{equation}
\end{Lemma}
\begin{proof} 
	Fix arbitrarily $\YObsSeq \in \MVM$, $\CtrMap \in \Set$, $(k,N) \in \intk \times \NsetRestricted$, $n \in \TheSet_{k,N}$. For any $\Ctr \in \Rset^{\mydim}$, it results from \eqref{eq: indices of n in intn belonging to cluster k} \& \eqref{eq: dist of Zn(xi)}, that
	\begin{equation}
		\begin{array}{lll}
			\Expect{w \big ( \norm[ \invsqrtestMCovn( \SquelY{\CtrMap}{n} - \Ctr)]^2 \big )} \vspace{0.2cm} \\
			\hspace{1cm} = \dfrac{1}{(2 \pi)^{\mydim/2} \, \prod_{\dimi = 1}^\mydim {\sqrteig_{n,\dimi}}} 
			\displaystyle \int w \big ( \, \norm[\xvec + \invsqrtestMCovn(\CtrMap(k) \! - \! \Ctr) ]^2 \, \big ) \, e^{-\frac{1}{2} \sum_{\dimi = 1}^\mydim  x_\dimi^2 / {\eig_{n,\dimi}}} \, \dxvec 
		\end{array}
		\label{eq: towards the lower bound for the expectation}
	\end{equation}
	with $\xvec = (x_1, \ldots, x_\mydim)^\transpose$ and ${\eig_{n,i} = \eigval_{n,\dimi}\, / \, \esteigvalndimi}$ for all $\dimi \in \intd$. Because of \eqref{eq: dist of Zn(xi)-prop}, we have $\prod_{\dimi = 1}^\mydim \sqrteig_{n,\dimi} \leqslant \lambdamax^{\mydim}$ and, for all $\xvec \in \Rset^{\mydim}$, $e^{-\frac{1}{2} \sum_{\dimi = 1}^\mydim  x_\dimi^2 / \eig_{n,\dimi}} \geqslant e^{-\frac{1}{2} \| \xvec \|^2 / \lambdamin^2}$.
%    $$\forall \xvec \in \Rset^{\mydim}, e^{-\frac{1}{2} \sum_{\dimi = 1}^\mydim  x_\dimi^2 / {\eig_{n,\dimi}}} \geqslant e^{-\frac{1}{2} \| \xvec \|^2 / \lambdamin^2} \myfstop$$ 
    In addition, the triangle inequality and the non-increasingness of $w$ induce that, for all $\xvec \in \Rset^{\mydim}$, $$\forall \xvec \in \Rset^{\mydim}, w \big ( \, \norm[\xvec + \invsqrtestMCovn(\CtrMap(k) - \Ctr) ]^2 \, \big ) \geqslant w \left ( \big ( \, \| \xvec \| +  \| \invsqrtestMCovn(\CtrMap(k) - \Ctr \big ) \| \, )^2 \right ) \myfstop$$ Injecting the three inequalities above into \eqref{eq: towards the lower bound for the expectation} yields
	\begin{onecol}
		\begin{equation}
			\label{eq: inequality on expectation with w}
			\Expect{w \big ( \norm[ \invsqrtestMCovn( \SquelY{\CtrMap}{n} - \Ctr)]^2 \big )} \geqslant
			\theexpect \left ( \norm[ \invsqrtestMCovn( \CtrMap(k) - \Ctr ) ] \right ) \myfstop
		\end{equation}
	\end{onecol}
	% \begin{twocol}
	% 	\begin{equation}
	% 		\label{eq: inequality on expectation with w}
	% 		\begin{array}{l}
	% 			\hspace{-0.25cm} 
	% 			\Expect{w \big ( \norm[ \invsqrtestMCovn( \SquelY{\CtrMap}{n} - \Ctr)]^2 \big )} \vspace{0.1cm} \\
	% 			\hspace{1.5cm} \geqslant 
	% 			\theexpect \left ( \norm[ \invsqrtestMCovn( \CtrMap(k) - \Ctr ) ] \right )
	% 		\end{array}
	% 	\end{equation}
	% \end{twocol}
    Take any $\Ctr \in \ball(\CtrMap(k),\radius)$. The second inequality in \eqref{eq: inequality on the mean} implies that
    $$
    \norm[ \invsqrtestMCovn( \CtrMap(k) - \Ctr ) ] \leqslant \invsqrtesteigvalmin \radius \myfstop
    $$
	Since $\theexpect$ does not increase, we derive from the inequality above and \eqref{eq: inequality on expectation with w} that
	\begin{equation}
		\label{eq: almost done with the inequaility in this lemma}
		\Expect{w \big ( \norm[ \invsqrtestMCovn( \SquelY{\CtrMap}{n} - \Ctr)]^2 \big )} \geqslant \theexpect \Big ( {\invsqrtesteigvalmin} \radius \Big ) \myfstop
	\end{equation}
	From \eqref{eq: T} with $\Wvec = \YObsSeqFull$ and $\Wvec_n = \SquelY{\CtrMap}{n}$, we have
	\begin{onecol}
		\begin{equation}
			\Expect{\MyTBis(\Ctr,\YObsSeqFull)} = 
			\dfrac{1}{\cardkN} \hspace{-0.3cm}
            \displaystyle \sum_{\quad n \in \TheSet_{k,N}} \hspace{-0.2cm} \Expect{w \big ( \Maha{\estMCovn}^2 (\SquelY{\CtrMap}{n} - \Ctr) \big )} \,\investeigvalndimi
			\nonumber
		\end{equation}
	\end{onecol}
	for any $\dimi \in \intd$. The slack inequality in \eqref{eq: lower bound on expectation of T} follows by injecting \eqref{eq: estimated Mahalanobis norm - relation with euclidean - diagonal case} and \eqref{eq: almost done with the inequaility in this lemma} into the equality above and taking that ${\sqrtesteigval_{n,\dimi}^{-2}} \geqslant {\sqrtesteigval_{\mymax}^{-2}}$ for all $\dimi \in \intd$ into account. The strict inequality in \eqref{eq: lower bound on expectation of T} results from the definition of $\theexpect$.
\end{proof}
\vspace{-0.5cm}
\begin{Lemma}
	\label{step:ess.sup.lim tends to 0}
	Given $\dimi \in \intd$, $\ctrind \in \intk$, $(\Ctr_0, \Ctr) \in \Rset^{\mydim} \times \Rset^{\mydim}$, $\YObsSeq \in \MVMbis$ and $\eta \in \Open{0}{\infty}$, there exists $\theradius \in \Open{0}{\infty}$ such that, for all $\CtrMap \in \Setbis{\Ctr_0}$:
	\begin{equation}
		\hspace{-3cm}
		\begin{array}{lll}
%			\exists \, \theradius \in \Open{0}{\infty}, \forall \, \CtrMap \in \Setbis{\Ctr_0}, 
%			\vspace{0.2cm} \\ 
			\hspace{-1cm} 
            \mydist{\CtrMap} > \theradius \Rightarrow \exists \, \Omega_0 \in \tribu, \Pbb(\Omega_0) = 1, \forall \omega \in \Omega_0, \exists N_0 \in \Nset, \forall N \in \Nset, 
		\end{array}
		\vspace{-0.5cm}
		\nonumber
	\end{equation}
	\begin{subequations}
		\label{eq: full predicate before limsup-2}
		\begin{empheq}[left={\hspace{-0.5cm} N \! > \! N_0 \Rightarrow } \empheqlbrace]{align}
			&\begin{array}{lll}
				\hspace{-0.2cm} \left \vert \MyWTer(\Ctr,\YObsSeqFull)(\omega) \! - \! \thecoeff_{\ctrind,N} \Expect{S_{\ctrind,N,i}(\Ctr,\YObsSeqFull)} \right \vert %\vspace{0.2cm} \\
				%\hspace{5cm} 
                \! < \! \dfrac{\eta}{1 + \eta} \, \thecoeff_0 \tau( \norm[\Ctr - \Ctr_0] )
			\end{array}
            \hspace{-0.8cm}
			\vspace{0.1cm} 
			\label{eq: full predicate before limsup-2a} \\
			& \, \mywedge \vspace{0.1cm} \nonumber \\
			& \MyVTer(\Ctr,\YObsSeqFull)(\omega) > \dfrac{\thecoeff_0 \, \tau( \norm[\Ctr - \Ctr_0] )}{1+\eta} > 0
			\label{eq: full predicate before limsup-2b}
		\end{empheq}
	\end{subequations}
	where
	\begin{equation}
		\label{eq: def of delta}
		\forall \radius \in \RightOpen{0}{\infty}, \tau( \radius ) = \theexpect \big ( {\invsqrtesteigvalmin} \, \radius \big ) {\investeigvalmax} > 0
        \myfstop
	\end{equation} 
\end{Lemma}
\begin{proof}
	The proof proceeds by combining predicates drawn from the lemmas above. From now on, we fix arbitrarily $\dimi \in \intd$, $\ctrind \in \intk$, $(\Ctr_0, \Ctr) \in \Rset^{\mydim} \times \Rset^{\mydim}$ and $\eta \in \Open{0}{\infty}$.

    \vspace{0.2cm}
	First, \eqref{eq: MyV and MyW-4} induces, with $\Wvec = \YObsSeqFull$, that
    \begin{Predicate}
    \label{predicate:p1}
	\begin{onecol}
		\begin{equation}
			\label{eq: third basic upper-bound for Preliminary result}
			\begin{array}{l}
				\begin{array}{llll}
					\forall \, \YObsSeq \in \MVMbis, \forall \, \CtrMap \in \Setbis{\Ctr_0}, \forall \omega \in \Omega, \forall N \in \Nset, N \geqslant K \implies \vspace{0.2cm} \\
					\hspace{0.4cm}
					\left \{
					\begin{array}{llll}
						\hspace{-0.3cm}
						\begin{array}{lll}
							\! \big \vert \, \MyWTer(\Ctr,\YObsSeqFull)(\omega) - \thecoeff_{\ctrind,N} 	\Expect{S_{\ctrind,N,\dimi}(\Ctr,\YObsSeqFull)} \big \vert \vspace{0.2cm} \\
							\begin{array}{lll} 
								\leqslant \big \vert \MyWTer(\Ctr,\YObsSeqFull)(\omega) - \Ebb \big [ \MyWTer(\Ctr,\YObsSeqFull) \big ] \big \vert + \sum_{k = 1, k \neq \ctrind}^K \thecoeff_{k,N} \big \vert \Expect{\MySBis(\Ctr,\YObsSeqFull)} \big \vert
							\end{array}			
						\end{array} 
						\vspace{0.2cm} \\
                        \hspace{-0.1cm} \mywedge
                        \vspace{0.2cm} \\
						\hspace{-0.3cm}
						\begin{array}{lll}
							\! \left \vert \MyVTer(\Ctr,\YObsSeqFull)(\omega) - \thecoeff_{\ctrind,N} 	\Expect{T_{\ctrind,N,i}(\Ctr,\YObsSeqFull)} \right \vert \vspace{0.2cm} \\
%							\quad 
							\begin{array}{lll} 
								\leqslant \big \vert \MyVTer(\Ctr,\YObsSeqFull)(\omega) - 	\Expect{\MyVTer(\Ctr,\YObsSeqFull)} \big \vert + \sum_{k = 1, k \neq \ctrind}^K \thecoeff_{k,N} \Big \vert \Expect{\MyTBis(\Ctr,\YObsSeqFull)} \Big \vert \myfstop
							\end{array}
						\end{array}
					\end{array}
					\right.
				\end{array}
				\hspace{-0.2cm}
			\end{array}
			\nonumber
		\end{equation}
	\end{onecol}
    \end{Predicate}
	Setting
	\begin{equation}
		\label{eq: definition of varepsilon}
		\myeps = \dfrac{1}{2} \dfrac{\eta}{1+\eta} \, \thecoeff_0 \, \tau( \norm[\Ctr - \Ctr_0] ), %\dfrac{1}{2} \dfrac{\eta}{1+\eta} \thecoeff_0 \delta(\| \CCtr - \CtrMapBis(i) \|)
	\end{equation}
	Lemma \ref{step: approx a.s of V and W} induces that:
    \begin{Predicate}
    \label{predicate:p2}
	\begin{onecol}
		\begin{equation}
			\label{eq: fifth basic upper-bound for Preliminary result}
			\begin{array}{l}
				\begin{array}{lll}
					\forall \, \YObsSeq \in \MVMbis, \forall \, \CtrMap \in \Setbis{\Ctr_0}, \exists \, \Omega_0 \in \tribu, 
					\vspace{0.1cm} \\
					\hspace{0.25cm} 
					\left \{
					\begin{array}{llll}
						\! \! \Pbb(\Omega_0) = 1, 
                        \medskip \\
						\! \! \mywedge \\
						\! \! \forall \omega \! \in \! \Omega_0, \exists N'_0 \! \in \! \Nset, \forall N \! \in \! \Nset, N \! > \! N'_0 \Rightarrow \! 
						\left \{
						\! \! \! 
						\begin{array}{llll}
							\begin{array}{lll}
								\! \! \big \vert \, \MyWTer(\Ctr,\YObsSeqFull)(\omega) - \Expect{\MyWTer(\Ctr,\YObsSeqFull)} \big \vert 
								< \myeps
							\end{array}
                            \vspace{0.1cm} \\
    						\mywedge
                            \vspace{0.1cm} \\
							\begin{array}{lll}
								\! \! \big \vert \, \MyVTer(\Ctr,\YObsSeqFull)(\omega) - \Expect{\MyVTer(\Ctr,\YObsSeqFull)} \big \vert 
								%						\\
								%						\qquad 
								< \myeps \vspace{0.1cm}
							\end{array}
						\end{array}
						\right.
					\end{array}
					\right.
				\end{array}
			\end{array}
			\nonumber
		\end{equation}
	\end{onecol}
        \end{Predicate}

	By combining \cref{predicate:p1,predicate:p2}, where we take $N_0 = \max(N'_0,K)$, we obtain:
    \begin{Predicate}
        \label{predicate:p3}
	\begin{onecol}
		\begin{equation}
			\label{eq: intermediate predicate}
			\begin{array}{l}
	
				\forall \, \YObsSeq \in \MVMbis, \forall \, \CtrMap \in \Setbis{\Ctr_0}, \exists \, \Omega_0 \in \tribu, \vspace{0.1cm} \\
				\hspace{0.4cm} \left \{
				\begin{array}{lll}
					\hspace{-0.1cm} 
					\Pbb(\Omega_0) = 1 \vspace{0.1cm} \\
					\mywedge \vspace{0.1cm} \\
					\hspace{-0.1cm}
					\forall \omega \in \Omega_0, \exists N_0 \in \Nset, \forall N \in \Nset, \vspace{0.2cm} \\
					\hspace{0.3cm} 
					N > N_0 \Rightarrow \hspace{-0.2cm}
					\begin{array}{llll}	
						\left \{
						\begin{array}{llll}
							\begin{array}{lll}
								\hspace{-0.3cm}
								\big \vert \, \MyWTer(\Ctr,\YObsSeqFull)(\omega) \! - \! \thecoeff_{\ctrind,N} 		\Expect{S_{\ctrind,N,i}(\Ctr,\YObsSeqFull)} \big \vert 
								\vspace{0.2cm} \\ 
								\hspace{2cm} < \myeps + \! \sum_{k = 1, k \neq \ctrind}^K \thecoeff_{k,N} \, \Big \vert \, \Expect{\Sbm_{k,N,i}(\Ctr,\YObsSeqFull)} \, \Big \vert
								\vspace{0.2cm} \\
                                \hspace{-0.3cm} \mywedge
								\vspace{0.2cm} \\
								\hspace{-0.3cm}
								\big \vert \, \MyVTer(\Ctr,\YObsSeqFull)(\omega) - \thecoeff_{\ctrind,N} 				\Expect{T_{\ctrind,N,i}(\Ctr,\YObsSeqFull)} \, \big \vert 
								\vspace{0.2cm} \\
								\hspace{2cm} < \myeps + \! \sum_{k = 1, k \neq \ctrind}^K \thecoeff_{k,N} \Big \vert \,  \Expect{\MyTBis(\Ctr,\YObsSeqFull)} \, \Big \vert \myfstop 
							\end{array}
						\end{array}
						\right.
					\end{array}
				\end{array}
				\right.
			\end{array}
			\nonumber
		\end{equation}
	\end{onecol}
    \end{Predicate}
	On the other hand, it follows from Corollary \ref{step: of behaviors of E[S] and E[T]} that
    \begin{Predicate}
        \label{predicate:p4}
    
	\begin{onecol}
		\begin{equation}
			\label{eq: second basic upper-bound for Preliminary result}
			\begin{array}{l}
				\begin{array}{llll}
					\exists \, \theradius \in \Open{0}{\infty}, \forall \, \YObsSeq \in \MVMbis, \forall \, \CtrMap \in \Setbis{\Ctr_0}, \vspace{0.1cm} \\
					\hspace{0.1cm}
					\begin{array}{lll}
						\mydist{\CtrMap} > \theradius
						\Rightarrow \forall (k,N) \in \left ( \intk \setminus \{\ctrind\} \right ) \times \NsetRestricted, 
						\left \{ \! \! 
						\begin{array}{llll}
							\big \vert \, \Expect{\MySBis(\Ctr,\YObsSeqFull)} \, \big \vert 
							< \myeps 
                            \vspace{0.2cm} \\
                            \mywedge
                            \vspace{0.2cm} \\
							%0 < 
                            \Expect{\MyTBis(\Ctr,\YObsSeqFull)} < \myeps \myfstop 
						\end{array}
						\right.
					\end{array}
				\end{array}
			\end{array}
			\nonumber
		\end{equation}
	\end{onecol}
    \end{Predicate}
	It results from \cref{predicate:p3,predicate:p4} and the definition of the coefficients  $\thecoeff_{k,N}$ given by \eqref{eq: inequality on alpha} that:
    \begin{Predicate}
        \label{predicate:p5}
    
	\begin{onecol}
		\begin{equation}
            \hspace{-1cm}
			\begin{array}{l}
				\begin{array}{llll}
%					\hspace{1cm} 
					\exists \, \theradius \in \Open{0}{\infty}, \forall \, \YObsSeq \in \MVMbis, 
					\forall \, \CtrMap \in \Setbis{\Ctr_0}, \exists \, \Omega_0 \in \tribu, \vspace{0.2cm} \\ 
					\hspace{0.5cm} 
					\left \{
					\begin{array}{lll}
						\! \! \Pbb(\Omega_0) = 1 \vspace{0.1cm} \\ %\medskip \\
						\! \! \mywedge \vspace{0.1cm} \\ %\medskip \\
						\! \! \forall \, \omega \in \Omega_0, \exists N_0 \in \Nset, \forall N \in \Nset, \vspace{0.2cm} \\
						\hspace{1cm}
						\big ( N > N_0 \big ) \wedge \big ( \mydist{\CtrMap} > \theradius \big ) 
						\implies \vspace{0.2cm} \\
						\hspace{2.25cm} 
						\left \{
						\begin{array}{lll}
							\hspace{-0.35cm} 
							\begin{array}{lll}
								\begin{array}{lll}
									\hspace{-0.1cm} 
									\big \vert \, \MyWTer (\Ctr,\YObsSeqFull)(\omega) - \thecoeff_{\ctrind,N} 	\Expect{S_{\ctrind,N,i}(\Ctr,\YObsSeqFull)} \big \vert < 2 \myeps
								\end{array}
							\end{array}
							\vspace{0.2cm} \\
                            \mywedge
							\vspace{0.2cm} \\
							\begin{array}{l}
								\hspace{-0.35cm} 
								\begin{array}{lll}
									\hspace{-0.1cm}
									\big \vert \, \MyVTer(\Ctr, \YObsSeqFull)(\omega) - \thecoeff_{\ctrind,N} 	\Expect{T_{\ctrind,N,i}(\Ctr, \YObsSeqFull)} \big \vert < 2 \myeps \myfstop
								\end{array}
							\end{array}
						\end{array}
						\right.
					\end{array}
					\right.
				\end{array}
			\end{array}
			\nonumber
		\end{equation}
	\end{onecol}
    \end{Predicate}
	On the one hand, $\mydist{\CtrMap}$ is unequivocally determined by $\CtrMap$, $\ctrind$ and $\Ctr_0$. On the other hand, according to Corollary \ref{step: of behaviors of E[S] and E[T]} and the definition of $\myeps$ given by \eqref{eq: definition of varepsilon}, $\theradius$ depends on $\eta, \Ctr_0, \Ctr$ only. Therefore, \cref{predicate:p5} implies
    \begin{Predicate}
    \label{predicate:p6}
        
	\begin{onecol}
	\begin{equation}
		\begin{array}{l}
			\begin{array}{llll}
				\hspace{-0.35cm} 
				\forall \, \YObsSeq \in \MVMbis, \exists \, \theradius \in \Open{0}{\infty}, \forall \, \CtrMap \in \Setbis{\Ctr_0}, 
				\vspace{0.2cm} \\
				\mydist{\CtrMap} > \theradius \Rightarrow \exists \, \Omega_0 \in \tribu, %\vspace{0.2 cm} \\ 
				\left \{
				\begin{array}{lll}
					\! \! \Pbb(\Omega_0) = 1 \vspace{0.1cm} \\ 
					\! \! \mywedge \vspace{0.1cm} \\ 
					\! \! \forall \omega \in \Omega_0, \exists N_0 \in \Nset, \forall N \in \Nset, N > N_0 \Rightarrow \vspace{0.2cm} \\
					\hspace{0.1cm} 
                    \left \{
					\begin{array}{lll}
						\hspace{-0.35cm} 
						\begin{array}{lll}
							\begin{array}{lll}
								\hspace{-0.15cm}
								\big \vert \, \MyWTer(\Ctr,\YObsSeqFull)(\omega) - \thecoeff_{\ctrind,N} 	\Expect{S_{\ctrind,N,i}(\Ctr,\YObsSeqFull)} \, \big \vert < 2 \myeps
							\end{array}
						\end{array}
						\vspace{0.2cm} \\
                        \mywedge
						\vspace{0.2cm} \\
    					\hspace{-0.15cm} 
                        \MyVTer(\Ctr, \YObsSeqFull)(\omega) > \thecoeff_{\ctrind,N} \Expect{T_{\ctrind,N,i}(\Ctr, \YObsSeqFull)} - 2 \myeps
                        \myfstop
						% \hspace{-0.15cm}
						% \big \vert \, \MyVTer(\Ctr, \YObsSeqFull)(\omega) - \thecoeff_{\ctrind,N} \Expect{T_{\ctrind,N,i}(\Ctr, \YObsSeqFull)} \, \big \vert < 2 \myeps \myfstop
					\end{array}
					\right.
				\end{array}
				\right.
			\end{array}
		\end{array}
		\nonumber
	\end{equation}
	\end{onecol}
        \end{Predicate}

	We have
	\begin{equation}
		\label{eq: very easy and basic inequalities}
		\left \{
		\begin{array}{lll}
			\thecoeff_{\ctrind,N} \geqslant \thecoeff_0 & \text{[from \eqref{eq: inequality on alpha}]} \mycomma \vspace{0.1cm} \\
			\thecoeff_0 \, \tau( \norm[\Ctr - \Ctr_0] ) - 2 \myeps = \dfrac{\thecoeff_0 \, \tau( \norm[\Ctr - \Ctr_0] )}{1+\eta} > 0 & \text{[from \eqref{eq: definition of varepsilon}]} \myfstop
		\end{array}
		\right.
	\end{equation}
	On the other hand, given $\YObsSeq \in \MVMbis$ and $\CtrMap \in \Setbis{\Ctr_0}$, for which we thus have $\CtrMap(\ctrind) = \Ctr_0$, we derive from \eqref{eq: def of delta} and Lemma \ref{step: Expect of TNi}, with $\radius = \norm[\Ctr - \Ctr_0]$, that:
	\begin{equation}
		\label{eq: another basic inequality}
		\forall N \in \NsetRestricted, \Expect{T_{\ctrind,N,i}(\Ctr,\YObsSeqFull)} \geqslant \tau( \norm[\Ctr - \Ctr_0] ) > 0
        \myfstop
	\end{equation}
	The result thus follows by injecting \eqref{eq: definition of varepsilon}, \eqref{eq: very easy and basic inequalities} and \eqref{eq: another basic inequality} into \cref{predicate:p6}.
\end{proof}

\begin{Corollary}[of Lemma \ref{step:ess.sup.lim tends to 0}]
	\label{step:ess.sup.lim tends to 0 - corollary}
	Given $\dimi \in \intd$, $\ctrind \in \intk$, $(\Ctr_0, \Ctr) \in \Rset^{\mydim} \times \Rset^{\mydim}$ and $\YObsSeq \in \MVMbis$,
	\begin{onecol}
		\begin{equation}
			\begin{array}{lll}
				\esslimsup{\mydist{\CtrMap} \to \infty}{N}
				\Big ( \MyWTer(\Ctr,\YObsSeqFull) - \thecoeff_{\ctrind,N} \Expect{S_{\ctrind,N,i}(\Ctr,\YObsSeqFull)} \Big ) \vspace{0.2cm} \\
				\hspace{1cm} 
				= \esslimsup{\mydist{\CtrMap} \to \infty}{N} \dfrac{\MyWTer(\Ctr,\YObsSeqFull) - \thecoeff_{\ctrind,N} \Expect{S_{\ctrind,N,i}(\Ctr,\YObsSeqFull)}}{\MyVTer(\Ctr,\YObsSeqFull)}
				\vspace{0.2cm} \\
				\hspace{1cm} 
				= 0	\myfstop
			\end{array}
			\nonumber
		\end{equation}
	\end{onecol}
\end{Corollary}
\begin{proof}
	Pick any $\dimi \! \in \! \intd$, $\ctrind \! \in \! \intk$, $(\Ctr_0, \Ctr) \! \in \! \Rset^{\mydim} \! \times \! \Rset^{\mydim}$, and $\YObsSeq \! \in \! \MVMbis$.
	For any $(\CtrMap,N) \in \Setbis{\Ctr_0} \times \NsetRestricted$, define $\Fbm_1(\CtrMap,N)$ by 
	$$\Fbm_1(\CtrMap,N) = \MyWTer(\Ctr,\YObsSeqFull) - \thecoeff_{\ctrind,N} 	\Expect{S_{\ctrind,N,i}(\Ctr,\YObsSeqFull)} \myfstop$$
	Consider any $\eta' \in \Open{0}{\thecoeff_0 \tau(\norm[\Ctr-\Ctr_0])}$, with $\tau(\norm[\Ctr-\Ctr_0]))$ given by \eqref{eq: def of delta}, and set $$\eta = \frac{\eta'}{\thecoeff_0 \tau(\norm[\Ctr-\Ctr_0])-\eta'} \myfstop$$
	From \eqref{eq: full predicate before limsup-2a} in Lemma \ref{step:ess.sup.lim tends to 0}, we obtain that 
	\begin{onecol}
		\begin{equation}
			\label{eq: els F1 = 0}
			\hspace{-0.6cm}
			\begin{array}{lll}
				\exists \, \theradius \in \Open{0}{\infty}, \forall \, \CtrMap \in \Setbis{\Ctr_0}, 
%				\vspace{0.2cm} \\ 
%				\hspace{0.6cm} 
				\mydist{\CtrMap} > \theradius \, \Rightarrow \, \exists \, \Omega_0 \in \tribu, \Pbb(\Omega_0) = 1, 
				\forall \omega \in \Omega_0, \exists N_0 \in \Nset, 
				\vspace{0.2cm} \\
				\hspace{7.5cm} {\forall N \in\Nset, N > N_0 \, \Rightarrow} \, \vert \Fbm_1(\CtrMap,N)(\omega) \vert < \eta' \myfstop
			\end{array}
			\hspace{-0.3cm}
		\end{equation}
	\end{onecol}
	Therefore, according to Corollary \ref{Corollary: Quanteur de lim of limsup pour notre pb}, we obtain $$\esslimsup{\mydist{\CtrMap} \to \infty}{N} \Fbm_1(\CtrMap,N) = 0$$
    and thus, the first equality.
    \vspace{0.25cm} \\
    \indent
	Now, given any $(\CtrMap,N) \in \Setbis{\Ctr_0} \times \NsetRestricted$, set     
    $$\Fbm_2(\CtrMap,N) = \dfrac{\MyWTer(\Ctr,\YObsSeqFull) - \thecoeff_{\ctrind,N} \Expect{S_{\ctrind,N,i}(\Ctr,\YObsSeqFull)}}{\MyVTer(\Ctr,\YObsSeqFull)} \myfstop$$
    This parametric discrete random process is well defined because \eqref{eq: T}, \eqref{eq: MyV and MyW-4b} and the fact that $w(t) > 0$ for all $t \in \RightOpen{0}{\infty}$ by \cref{Assumption: on w}, imply that $\MyVTer(\Ctr,\YObsSeqFull) > 0$. Given any $\eta \in \Open{0}{\infty}$, we derive from \eqref{eq: full predicate before limsup-2} that
	\begin{equation}
		\label{eq: els F2 = 0}
		\hspace{-0.35cm}
		\begin{array}{lll}
			\exists \, \theradius \in \Open{0}{\infty}, \forall \, \CtrMap \in \Setbis{\Ctr_0}, \mydist{\CtrMap} > \theradius \Rightarrow 
			\vspace{0.2cm} \\ 
			\hspace{0.5cm} \exists \, \Omega_0 \in \tribu, \Pbb(\Omega_0) = 1, 
			\forall \, \omega \in \Omega_0, \exists N_0 \in \Nset, \forall N \in \Nset, N > N_0 \Rightarrow \vert \, \Fbm_2(\CtrMap,N)(\omega) \, \vert < \eta \myfstop
		\end{array}
        \hspace{-0.3cm}
	\end{equation}
	Therefore, according to Corollary \ref{Corollary: Quanteur de lim of limsup pour notre pb}, 
    $$\esslimsup{\mydist{\CtrMap} \to \infty}{N} \Fbm_2(\CtrMap,N) = 0 \myfstop$$
    Whence the second equality.
\end{proof}
\indent
For further use below, note that, given $\ctrind \! \in \! \intk$, $N \! \in \! \NsetRestricted$, $i \! \in \! \intd$, $\Ctr \! \in \! \Rset^{\mydim}$, $\YObsSeq \! \in \! \MVM$ and $\CtrMap \! \in \! \Set$, it results from \eqref{eq: S}, \eqref{eq: estimated Mahalanobis norm - relation with euclidean - diagonal case} \& \eqref{eq: coordinate of estMCov_n(-1/2) x} that:
\begin{onecol}
	\begin{equation}
		\label{eq: Sbm(i,N)}
        \hspace{-0.1cm}
		S_{\ctrind,N,i}(\Ctr,\YObsSeqFull) 
		\! = \! \dfrac{1}{\cardkoN} \! \! \displaystyle \sum_{n \in \TheSet_{\ctrind,N}} \hspace{-0.25cm} \! w \Big ( \displaystyle \sum_{\dimib=1}^\mydim \sqrtesteigvalndimib^{\, -2} ( \SquelY{\CtrMap}{n}-\Ctr )_{\dimib}^{2} \, \Big  ) \, \sqrtesteigvalndimi^{\, -2} \, ( \SquelY{\CtrMap}{n} - \Ctr )_\dimi
        \myfstop
    \end{equation}
\end{onecol}
Also, if $\ctrind \! \in \! \intk$, $N \! \in \! \NsetRestricted$, $i \! \in \! \intd$, $(\Ctr_0, \Ctr) \! \in \! \Rset^{\mydim} \times \Rset^{\mydim}$, $\YObsSeq \! \in \! \MVMbis$ and $\CtrMap \! \in \! \Setbis{\Ctr_0}$, then $\CtrMap(\seq_n) \! = \! \CtrMap(\ctrind) \! = \! \Ctr_0$ for all $n \! \in \! \TheSet_{\ctrind,N}$ and thus, according to \eqref{eq: property of Squel} and \eqref{eq: deltamat = }:
\begin{equation}
    \label{eq: the law of a coordinate}
    \forall \, n \in \TheSet_{\ctrind,N}, \forall \, \dimi \in \intd, (\Squel{\CtrMap}{n} - \Ctr)_\dimi \thicksim \Ncal \left ( (\ctr_0 - \ctr)_\dimi, \eigvalndimi \right ) \myfstop
\end{equation}
\begin{Lemma}
	\label{step:equivalence of limits}
	Given $\ctrind \in \intk$, $(\Ctr_0, \Ctr) \in \Rset^{\mydim} \times \Rset^{\mydim}$, $\YObsSeq \in \MVMbis$ and $\CtrMap \in \Setbis{\Ctr_0}$,
	\begin{equation}
		\left ( \, \forall \, i \in \intd, \displaystyle \lim_{N \to \infty} \Expect{S_{\ctrind,N,i}(\Ctr,\YObsSeqFull)} = 0 \, \right ) \, \, \, \Leftrightarrow \, \, \, \Ctr = \Ctr_0 %\Ctr_i = (\Ctr_0)_i \myfstop
		\nonumber
	\end{equation}
\end{Lemma}
\begin{proof}
	Fix arbitrarily $\ctrind \in \intk$, $(\Ctr_0, \Ctr) \in \Rset^{\mydim} \times \Rset^{\mydim}$, $\YObsSeq \in \MVMbis$, $\CtrMap \in \Setbis{\Ctr_0}$ and $i \in \intd$. Let $N \in \NsetRestricted$. Set $M = \cardkoN$ and write $\TheSet_{\ctrind,N} = \{n_1, n_2, \ldots, n_M \}$ by ordering the elements of this set. For each $m \in \intM$ and each $\dimi \in \intd$, put 
    \begin{equation}
         \label{eq: Xmi}
         X_{m,\dimi} = \invsqrteigvalnmdimi \left ( \SquelY{\CtrMap}{n_m} - \Ctr \right )_\dimi \myfstop
         \nonumber
    \end{equation}
    We thus derive from \eqref{eq: the law of a coordinate} that
    \begin{equation}
         \label{eq: Distribution of Xmi}
         X_{m,\dimi} \thicksim 
         \Ncal \left ( \invsqrteigvalnmdimi  (\ctr_0 - \ctr)_\dimi, 1 \right ) \myfstop
    \end{equation}
	If we set $\Xvec_m = (X_{m,1}, \ldots, X_{m,\mydim})^\transpose$, we thus have
	\begin{equation}
		\label{eq: distribution is what we want}
		\forall m \in \intM, \Xvec_m \thicksim \Ncal(\Coeff_m \meanvec,\Id) \, \text{with} \left \{ \begin{array}{lll} \hspace{-0.1cm} \Coeff_m = \MCovnm^{\, -1/2} \mycomma \vspace{0.1cm} \\ \hspace{-0.1cm} \meanvec = \Ctr_0 - \Ctr \myfstop \end{array} \right.
		\hspace{-0.4cm}
	\end{equation}
    Thence, according to \eqref{eq: Sbm(i,N)}, $S_{\ctrind,N,i}(\Ctr,\YObsSeqFull)$ rewrites as
    $$
    S_{\ctrind,N,i}(\Ctr,\YObsSeqFull) 
		= \dfrac{1}{M} \displaystyle \sum_{m = 1}^M \sqrtesteigvalnmdimi^{\, -2} \, \sqrteigvalnmdimi \, w \Big ( \displaystyle \sum_{\dimib=1}^\mydim \sqrtesteigvalnmdimib^{\, -2} \, \sqrteigvalnmdimib^2 \, X_{m,\dimib}^{2} \, \Big  ) \, X_{m,\dimi} \myfstop
    $$
    Therefore, we have:
%	It thus follows from the foregoing and \eqref{eq: Sbm(i,N)} that 
	\begin{equation}
		\label{eq: sum rewritten}
		\Expect{S_{\ctrind,N,i}(\Ctr,\YObsSeqFull)} = \dfrac{1}{M} \sum_{m=1}^M \Expect{\fbm_m(\Xvec_m) X_{m,\dimi}}
	\end{equation}
	with
	\begin{equation}
	    \label{eq: expression of fm}
        \forall \, \xvec = (x_1, \ldots, x_\mydim)^\transpose \in \Rset^\mydim, \fbm_{m}(\xvec) = \sqrtesteigvalnmdimi^{\, -2} \, \sqrteigvalnmdimi \, w \Big ( \displaystyle \sum_{\dimib=1}^\mydim \sqrtesteigvalnmdimib^{\, -2} \, \sqrteigvalnmdimib^{\, 2} \, x_{\dimib}^{\, 2} \, \Big  )
        \myfstop
        \nonumber
	\end{equation}
	Since \cref{Assumption: on alpha} induces the equivalence between an arbitrarily large $N$ and an arbitrarily large $M$, it follows from \eqref{eq: sum rewritten} that
	\begin{onecol}
		\begin{equation}
			\label{eq: equivalence between the limits}
			\displaystyle \lim_{N \to \infty} \Expect{S_{\ctrind,N,i}(\Ctr,\YObsSeqFull)} = 0 \, \Leftrightarrow \, \displaystyle \lim_{M \to \infty} \dfrac{1}{M} \sum_{m=1}^M \Expect{\fbm_m(\Xvec_m) X_{m,\dimi}} = 0 \myfstop
		\end{equation}
	\end{onecol}
According to \eqref{eq: deltamat = } and \eqref{eq: delta in interval}, for all $m \in \intM$, $\Coeff_m$ is diagonal {with diagonal elements in} $[\invsqrteigvalmax, \invsqrteigvalmin] \subset \Open{0}{\infty}$. In addition, $w(t)$ does not increase with $t$ {and it follows from \cref{Assumption: on w} that} $w(t)$ is finitely upper-bounded. {Without loss of generality, suppose that this upper-bound is $1$.} Therefore, it follows from \eqref{eq: expression of fm} that, for any $m \in \intM$, $0 < \fbm^* \leqslant \fbm_{m} \leqslant \sqrtesteigvalmin^{-2} \, \sqrteigvalmax$ with
$$\forall \xvec \in \Rset^{\mydim}, \fbm^*(\xvec) = \sqrtesteigvalmax^{\, -2} \, \sqrteigvalmin \, w \left ( \sqrtesteigvalmin^{\, -2} \, \sqrteigvalmax^{\, 2} \, \norm[\xvec]^2 \right ) \myfstop$$
For any $m \in \intM$, $\fbm_{m}(\xvec)$ is, in the sense of Definition \ref{def: fn even in each coord} in Appendix \ref{App: Instrumental lemmas}, even in each coordinate of $\xvec \in \Rset^{\mydim}$. The assumptions of Lemma \ref{Lemma:The last very useful lemma} are thus satisfied and the result follows from \eqref{eq: distribution is what we want}, \eqref{eq: sum rewritten}, \eqref{eq: equivalence between the limits} \& \eqref{eq: equivalence between the limit of the sum and a null mean} in Lemma \ref{Lemma:The last very useful lemma}.
\end{proof}
\indent
We now achieve the proof of $\myclubsuit$ in Theorem \ref{Theorem: fixed points} for diagonal positive definite matrices. As of now, we fix arbitrarily $\ctrind \in \intk$, $(\Ctr_0,\Ctr) \in \Rset^{\mydim} \times \Rset^{\mydim}$ and $\YObsSeq \in \MVMbis$. 

\vspace{0.1cm}
We begin with the direct implication. Suppose that $\Ctr = \Ctr_0$. Consider any $\CtrMap \in \Setbis{\Ctr_0}$ and any $\dimi \in \intd$. Lemma \ref{step:equivalence of limits} implies that $$\displaystyle \lim_{N \to \infty} \Expect{S_{\ctrind,N,i}(\Ctr,\YObsSeqFull)} = 0 \myfstop$$ 
%Consider any $\dimi \in \intd$. 
We thus have:
\begin{onecol}
	\begin{equation}
		\label{eq: predicate on lim of E[Skni]}
		\forall \, \myeps \in \Open{0}{\thecoeff_0 \tau(0)}, \exists \, N'_0 \in \Nset, {\forall N \, \in \Nset, N > N'_0 \Rightarrow} \, 0 \leqslant \Expect{S_{\ctrind,N,i}(\Ctr,\YObsSeqFull)} < \myeps^2 \myfstop
	\end{equation}
\end{onecol}
On the other hand, set $\eta = \frac{\thecoeff_0 \tau(0) - \myeps}{\myeps}$. Since $\eta \in \Open{0}{\infty}$, it follows from \eqref{eq: full predicate before limsup-2b} in Lemma \ref{step:ess.sup.lim tends to 0} and our choice for $\eta$ that:
\begin{onecol}
	\begin{equation}
		\label{eq: full predicate before limsup-3}
        \hspace{-0.2cm}
		\begin{array}{lll}
			\exists \, \theradius \in \Open{0}{\infty}, \forall \, \CtrMap \in \Setbis{\Ctr_0}, \vspace{0.2cm} \\
			\hspace{0.9cm} \mydist{\CtrMap} > \theradius \Rightarrow \exists \, \Omega_0 \in \tribu, 
            \left \{
            \begin{array}{lll}
            \hspace{-0.1cm} \Pbb(\Omega_0) = 1, \forall \omega \in \Omega_0, \exists N''_0 \in \Nset, \vspace{0.1cm} \\
            \hspace{-0.1cm} 
            \forall N \in \Nset, %\vspace{0.1cm} \\ 
%			\hspace{0.5cm} 
            %\hspace{-0.2cm} 
            N > N''_0 \, \Rightarrow \, \, \MyVTer(\Ctr,\YObsSeqFull)(\omega) > \myeps > 0 \myfstop
            \end{array}
            \right.
		\end{array}
        \hspace{-0.1cm}
	\end{equation}
\end{onecol}
Define $\Fbm(\CtrMap,N)$ for any $(\CtrMap,N) \in \Setbis{\Ctr_0} \times \Nset$ by setting
$$\Fbm(\CtrMap,N) = \dfrac{\Expect{S_{\ctrind,N,i}(\Ctr,\YObsSeqFull)}}{\MyVTer(\Ctr,\YObsSeqFull)} \myfstop$$
By combining \eqref{eq: predicate on lim of E[Skni]} \& \eqref{eq: full predicate before limsup-3}, setting $N_0 = \max\{N'_0,N''_0\}$ and taking into account that, by definition of $\thecoeff_{\ctrind,N}$ given in \eqref{eq: inequality on alpha}, $\thecoeff_{\ctrind,N} \leqslant 1$ for all $N \in \Nset$, we obtain:
\begin{onecol}
	\begin{equation}
		\label{eq: full predicate before limsup-4}
		\begin{array}{lll}
			\forall \, \myeps \in \Open{0}{\thecoeff_0 \tau(0)} , \exists \, \theradius \in \Open{0}{\infty}, \forall \, \CtrMap \in \Setbis{\Ctr_0}, 
			\vspace{0.2cm} \\ 
			\hspace{1.5cm} \mydist{\CtrMap} > \theradius \Rightarrow  \exists \, \Omega_0 \in \tribu, \Pbb(\Omega_0) = 1, \forall \omega \in \Omega_0, \exists N_0 \in \Nset, \vspace{0.2cm} \\
			\hspace{5cm} {\forall N \in \Nset, N > N_0 \Rightarrow} \, 0 \leqslant \thecoeff_{\ctrind,N} \Fbm(\CtrMap,N)(\omega) < \myeps \myfstop
		\end{array}
		\hspace{-0.5cm}
		\nonumber
	\end{equation}
\end{onecol}
It follows from the predicate above and Corollary \ref{Corollary: Quanteur de lim of limsup pour notre pb} of Proposition \ref{Proposition: Quanteur de lim of limsup} that $$\esslimsup{\mydist{\CtrMap} \to \infty}{N} \thecoeff_{\ctrind,N} \Fbm(\CtrMap,N) = 0 \myfstop$$ 
Since $\dimi$ is arbitrary in $\intd$, we thus have
\begin{onecol}
	\begin{equation}
		\label{Eq: Encore une limsup}
		\forall \dimi \in \intd, \esssuplim{\mydist{\CtrMap} \to \infty}{N} \dfrac{\thecoeff_{\ctrind,N} \, \Expect{S_{\ctrind,N,i}(\Ctr,\YObsSeqFull)}}{\MyVTer(\Ctr,\YObsSeqFull)} = 0 \myfstop
	\end{equation}
\end{onecol}
Consequently, by Proposition \ref{prop: sums of ess sup lim} and the second equality in Corollary \ref{step:ess.sup.lim tends to 0 - corollary} of Lemma \ref{step:ess.sup.lim tends to 0}, \eqref{Eq: Encore une limsup} is equivalent to
\begin{equation}
	\label{eq: the limit for each i}
	\forall \dimi \in \intd, \esssuplim{\mydist{\CtrMap} \to \infty}{N} \,  \dfrac{\MyWTer(\Ctr,\YObsSeqFull)}{\MyVTer(\Ctr,\YObsSeqFull)} = 0 \myfstop
\end{equation}
We thus derive from \eqref{eq: new h (a) - 3} and \eqref{eq: esl centred} that \eqref{eq: the limit for each i} is equivalent to
\begin{equation}
	\label{eq: premisse}
	\esssuplim{\mydist{\CtrMap} \to \infty}{N} \TheFunc(\Ctr,\aObsSeq) = \Ctr \myfstop
	%	\nonumber
\end{equation}
Conversely, suppose that \eqref{eq: premisse} holds true. By the equivalences established above, we thus have \eqref{Eq: Encore une limsup}. By assumption, $w$ is upper-bounded and non-negative. Without loss of generality, suppose that this upper-bound is $1$ as in the proof of Lemma \ref{step:equivalence of limits}. Therefore, according to \eqref{eq: MyW in Thm1-0} and since $w > 0$ by assumption\footnote{Note that $w > 0$ almost everywhere with respect to Lebesgue measure would suffice.} $$\forall \omega \in \Omega, 0 < \MyVTer(\Ctr, \YObsSeqFull) (\omega) \leqslant {\investeigvalmin} \myfstop$$
As a consequence, since $\thecoeff_{\ctrind,N} \geqslant \thecoeff_0$ by \eqref{eq: inequality on alpha}, we derive from the foregoing that
\begin{onecol}
	$$
	\forall \omega \in \Omega, \left \vert \Expect{S_{\ctrind,N,i}(\Ctr,\YObsSeqFull)} \right \vert 
	\leqslant \dfrac{\investeigvalmin}{\thecoeff_0} \dfrac{\thecoeff_{\ctrind,N} \left \vert \Expect{S_{\ctrind,N,i}(\Ctr,\YObsSeqFull)} \right \vert} { \MyVTer(\Ctr,\YObsSeqFull) (\omega)} %\hspace{0.4cm} \left (\text{$\Pbb$-a.s} \right )
	$$
\end{onecol}
and thus, that
\begin{onecol}
	\begin{equation}
		\label{eq: inequalities on limsup}
		\forall \omega \in \Omega, \displaystyle \mylimsup{N \to \infty} \left \vert \, \Expect{S_{\ctrind,N,i}(\Ctr,\YObsSeqFull)} \, \right \vert \leqslant \dfrac{\investeigvalmin}{\thecoeff_0} \displaystyle \mylimsup{N \to \infty} \dfrac{\thecoeff_{\ctrind,N} \left \vert \Expect{S_{\ctrind,N,i}(\Ctr,\YObsSeqFull)} \right \vert} { \MyVTer(\Ctr,\YObsSeqFull)(\omega)} \myfstop
	\end{equation}
\end{onecol}
According to \eqref{eq: elementary property of the infinity norm}, there exists $\Omega_0 \in \tribu$ with $\Pbb(\Omega_0) = 1$ such that
\begin{onecol}
	\begin{equation}
		\label{eq: bound for the limsup}
		\forall \omega \in \Omega_0, \mylimsup{N \to \infty} \dfrac{\thecoeff_{\ctrind,N} \left \vert \Expect{S_{\ctrind,N,i}(\Ctr,\YObsSeqFull)} \right \vert}{ \MyVTer(\Ctr,\YObsSeqFull)(\omega)} \leqslant \left \vert \mylimsup{N \to \infty} \dfrac{\thecoeff_{\ctrind,N} \left \vert \Expect{S_{\ctrind,N,i}(\Ctr,\YObsSeqFull)} \right \vert}{ \MyVTer(\Ctr,\YObsSeqFull)} \right \vert_{\infty} \myfstop
	\end{equation}
\end{onecol}
Thence, \eqref{eq: inequalities on limsup} and \eqref{eq: bound for the limsup} imply that
\begin{onecol}
	\begin{equation}
		\label{eq: inequalities on limsup-}
%		\nonumber
		\mylimsup{N \to \infty} \left \vert \Expect{S_{\ctrind,N,i}(\Ctr,\YObsSeqFull)} \right \vert \leqslant \dfrac{{\investeigvalmin}}{\thecoeff_0} \left \vert \mylimsup{N \to \infty} \dfrac{\thecoeff_{\ctrind,N} \left \vert \Expect{S_{\ctrind,N,i}(\Ctr,\YObsSeqFull)} \right \vert}{ \MyVTer(\Ctr,\YObsSeqFull)} \right \vert_{\infty} \myfstop
	\end{equation}
\end{onecol}
On the other hand, from \eqref{Eq: Encore une limsup} and \eqref{eq: esl of a discrete process in our case}, we have
$$\lim_{\mydist{\CtrMap} \to \infty} \left \vert \mylimsup{N \to \infty} \dfrac{\thecoeff_{\ctrind,N} \left \vert \Expect{S_{\ctrind,N,i}(\Ctr,\YObsSeqFull)} \right \vert}{ \MyVTer(\Ctr,\YObsSeqFull)} \right \vert_{\infty} = 0 \myfstop$$
It thus follows from \eqref{eq: inequalities on limsup-} that
\begin{equation}
	\label{eq: final limit-1}
	\lim_{\mydist{\CtrMap} \to \infty} \, \left ( \, \mylimsup{N \to \infty} \left \vert \, \Expect{S_{\ctrind,N,i}(\Ctr,\YObsSeqFull)} \, \right \vert \, \right ) = 0 \myfstop
\end{equation}
Since $\YObsSeq \in \MVMbis$ and $\CtrMap \in \Setbis{\Ctr_0}$, \eqref{eq: Sbm(i,N)} and \eqref{eq: the law of a coordinate} induce that $\Expect{S_{\ctrind,N,i}(\Ctr,\YObsSeqFull)}$ depends only on $\Ctr_0 - \Ctr$, $\MCovn$ and $\estMCovn$ with $n \in \TheSet_{\ctrind,N}$. Hence, the value of $\Expect{S_{\ctrind,N,i}(\Ctr,\YObsSeqFull)}$ does not depend on $\CtrMap$ and thus on $\mydist{\CtrMap}$ defined by \eqref{eq: definition of distmin}.
As a consequence, the value of $\mylimsup{N \to \infty} \left \vert \, \Expect{S_{\ctrind,N,i}(\Ctr,\YObsSeqFull)} \, \right \vert$, which always exists in $[0,\infty]$, does not depend on $\mydist{\CtrMap}$ either. Thereby, \eqref{eq: final limit-1} implies that:
\begin{onecol}
	\begin{equation}
		\label{eq: equality between esl and limsup}
		\mylimsup{N \to \infty} \left \vert \, \Expect{S_{\ctrind,N,i}(\Ctr,\YObsSeqFull)} \, \right \vert 
		= \displaystyle \lim_{\mydist{\CtrMap} \to \infty} \left ( \, \mylimsup{N \to \infty} \left \vert \, \Expect{S_{\ctrind,N,i}(\Ctr,\YObsSeqFull)} \, \right \vert \, \right )
		= 0 \myfstop
	\end{equation}
\end{onecol}
Therefore, $\displaystyle \lim_{N \to \infty} \Expect{S_{\ctrind,N,\dimi}(\Ctr,\YObsSeqFull)} = 0$ and since $\dimi$ is arbitrary in $\intd$, Lemma \ref{step:equivalence of limits} implies that $\Ctr = \Ctr_0$. 
\subsection{Proof of Equivalence (\textbf{A}) for non-diagonal \spdms}
\label{subsec:The case of any sequences of covariance matrices}
We begin with a preliminary lemma that makes it possible to move from the diagonal case addressed above to the general case of non-diagonal \spdms.
\begin{Lemma}
	\label{lemma: lemma useful in practice}
	If $\MCov  = (\MCovn)_{n \in \Nset}$ and $\estMCov = (\estMCovn)_{n \in \Nset}$ are two sequences of \spdms~satisfying \cref{Assumption: on C,Assumption: on lambda}, then, for any sequence $\Wvec = (\Wvec_n)_{n \in \Nset}$ of elements of $\rv$ and any $\Ctr \in \Rset^\mydim$,	
	\begin{onecol}
		\begin{equation}
			\label{eq: equivalence of els-0}
			%\forall \, \Ctr \in \Rset^{\mydim}, \forall \, \Wvec = (\Wvec_{n})_{n \in \Nset} \in \family,
            \TheFunc(\Ctr,\Wvec) = \Rmat \, \TheFuncb \left (\Rmat^\transpose \Ctr,\fn{\Rmat^\transpose} ( 	\Wvec ) \right )
			\nonumber
		\end{equation}
	\end{onecol}
	where $\estDiagmat = (\estDiagmatn)_{n \in \Nset}$ and $\Rmat$ are defined according to \eqref{eq: eigen decompositions}.
\end{Lemma}	
\begin{proof}
	Let $\Wvec = (\Wvec_n)_{n \in \Nset}$ be any sequence of elements of $\rv$ and $\Ctr \in \Rset^\mydim$. First, according to \eqref{eq: new h (a) - 0} applied to $\estDiagmat$, we have:
	\begin{onecol}
		\begin{equation}
			\label{eq: new h (a) - case diagonal}
			\TheFuncb(\Ctr,\Wvec) 
			= \left ( \sum_{n=1}^N w \big ( \Maha{\estDiagmatn}^2 (\Wvec_n - \Ctr) \big ) \investDiagmatn \right )^{\! -1} \! 
			\times \, \sum_{n=1}^N w \big ( \Maha{\estDiagmatn}^2 ( \Wvec_n - \Ctr) \big ) \investDiagmatn \Wvec_n \myfstop
		\end{equation}
	\end{onecol}
	Second, by injecting \eqref{eq: eigen decomposition of estMCov} \& \eqref{eq: estimated Mahalanobis norm - relation with euclidean} into \eqref{eq: new h (a) - 1}, we obtain
	\begin{equation}
		\label{eq: we rewrite G}
		\begin{array}{lll}
			\hspace{-0.15cm} 
			\TheFunc(\Ctr,\Wvec)
			\vspace{0.2cm} 
			\\
			= \Big ( \Rmat \displaystyle \sum_{n=1}^N w \big ( \Maha{\estDiagmatn}^2 \! \big ( \Rmat^\transpose (\Wvec_n-\Ctr) \big ) \big ) \investDiagmatn \Rmat^\transpose \, \Big )^{\! -1} \! \! \times \! \Rmat \displaystyle \sum_{n=1}^N w \big ( \Maha{\estDiagmatn}^2 \! \big ( \Rmat^\transpose (\Wvec_n-\Ctr) \big ) \big ) \investDiagmatn \Rmat^\transpose \Wvec_n
			\vspace{0.2cm} \\
			= \Rmat \Big ( \displaystyle \sum_{n=1}^N \! w \big ( \Maha{\estDiagmatn}^2 \! \big ( \Rmat^\transpose \Wvec_n \! - \! \Rmat^\transpose \Ctr \big ) \big ) \investDiagmatn \Big )^{\! -1} \! \Rmat^\transpose \! \! \times \! \Rmat \displaystyle \sum_{n=1}^N \! w \big ( \Maha{\estDiagmatn}^2 \! \big ( \Rmat^\transpose \Wvec_n \! - \! \Rmat^\transpose \Ctr \big ) \big ) \investDiagmatn \! \Rmat^\transpose \Wvec_n
			%\vspace{0.1cm} 
			\\
			= \Rmat \Big  ( \displaystyle \sum_{n=1}^N w \big ( \Maha{\estDiagmatn}^2 \! \big ( \Rmat^\transpose \Wvec_n - \Rmat^\transpose \Ctr \big ) \big ) \investDiagmatn \Big )^{\! -1} \! \times \, \displaystyle \sum_{n=1}^N w \big ( \Maha{\estDiagmatn}^2 \! \big ( \Rmat^\transpose \Wvec_n - \Rmat^\transpose \Ctr \big ) \big ) \investDiagmatn \Rmat^\transpose \Wvec_n \mycomma
		\end{array}
	\end{equation}    
    since $\Rmat$ is orthogonal and, because $w(t) > 0$ for all $t \in \RightOpen{0}{\infty}$, $$\forall \omega \in \Omega, \sum_{n=1}^N w \left ( \Maha{\estDiagmatn}^2 \big ( \Rmat^\transpose \Wvec_n (\omega)- \Rmat^\transpose \Ctr \big ) \right ) \investDiagmatn \in \GL \myfstop$$
    Whence the result as a consequence of \eqref{eq: definition of W bar}, \eqref{eq: new h (a) - case diagonal} and the last equality in \eqref{eq: we rewrite G}.
\end{proof}
\indent
Suppose that \cref{Assumption: on alpha,Assumption: on w,Assumption: on C,Assumption: on lambda} hold for a sequence $\seq = (\seq_n)_{n \in \Nset}$ of integers in $\intk$, a function $w: \RightOpen{0}{\infty} \to \Open{0}{\infty}$ and two sequences of \spdms~$\MCov = (\MCovn)_{n \in \Nset}$ and $\estMCov = (\estMCovn)_{n \in \Nset}$. Consider any $\YObsSeq \in \MVMbis$, any $\CtrMap \in \aSetbis$ and any $\Ctr \in \Rset^\mydim$. 
\medskip \\
\indent
We derive from Lemma \ref{lemma: lemma useful in practice} that
\begin{onecol}
\begin{equation}
	\label{eq: equivalence of els}
		\TheFunc(\Ctr, \aObsSeq) - \Ctr = \Rmat \, \left (\TheFuncb(\Rmat^\transpose \Ctr,\fn{\Rmat^\transpose} (\aObsSeq) ) - \Rmat^\transpose \Ctr \right ) \mycomma
\end{equation}
\end{onecol}
where $\estDiagmat = (\estDiagmatn)_{n \in \Nset}$, the matrices $\estDiagmatn$ and $\Rmat$ being defined by \eqref{eq: eigen decompositions}.

\vspace{0.1cm}
In addition, choose any $\Ctr_0 \in \Rset^\mydim$ and any $\ctrind \in \intk$. Introduce the function $\psitfn: \bSetbis \to \aSetbis$ , denoted shortly $\psit$, assigning to a given $\CtrMapBis \in \bSetbis$ the sequence $\psit(\CtrMapBis) \in \aSetbis$ defined by:
\begin{equation}
	\label{eq: definition of psit}
	\begin{array}{cccc}
		\psit(\CtrMapBis):
		& \intk & \to & \Rset^{\mydim} \\
		& k & \mapsto & \Rmat \, \CtrMapBis(k) \myfstop
	\end{array}
\end{equation}
Since $\Rmat \in \GL$, $\psit$ is a bijection and its inverse is $\invpsitfn: \aSetbis \to \bSetbis$ --- in short, $\invpsit$. According to \eqref{eq: definition of psit}, $\invpsit$ assigns to a given $\CtrMap \in \aSetbis$ the sequence:
\begin{equation}
	\label{eq: definition of invpsit}
	\begin{array}{cccc}
		\invpsit(\CtrMap):
		& \intk & \to & \Rset^{\mydim}  \\
		& k & \mapsto & \Rmat^\transpose \CtrMap(k) \myfstop
	\end{array}
\end{equation}
According to the definition of $\fn{\Rmat^\transpose}$ given by \eqref{eq: definition of W bar}, we construct the composite map 
\begin{equation}
	\label{eq: composite map}
	\hspace{-0.2cm}
	\begin{array}{cccccc}
		\fn{\Rmat^\transpose} \! \circ \! \Obs \! \circ \! \psit: 
		& \bSetbis \! & \! \to \! & \! \familyfn{\Delta} 
		\vspace{0.1cm} \\
		& \CtrMapBis \! & \! \mapsto \! & \! \left ( \Rmat^\transpose \Squel{\psit(\CtrMapBis)}{n} \right )_{n \in \Nset} \myfstop
		\vspace{0.1cm}
	\end{array}
\end{equation}
Let $\CtrMapBis \! \in \! \bSetbis$. We have $\psit(\CtrMapBis) \! \in \! \aSetbis$ and $\YObsSeq(\psit(\CtrMapBis)) \! \in \! \familyfn{\MCov}$. Therefore, since we have $\YObsSeq \! \in \! \MVMbis$, it follows that $\CtrSeq(\YObsSeq(\psit(\CtrMapBis))) \! = \! \psit(\CtrMapBis)$, where we recall that $\CtrSeq$ is defined by Definition \ref{Definition: CtrMap}. \eqref{eq: distribution of Rmat'Wn} thus implies that
\begin{equation}
	\label{eq: distribution of Rmat' Y R}
	\begin{array}{lll}
		\forall n \in \Nset, \Rmat^\transpose \Squel{\psit(\CtrMapBis)}{n} 
		& \thicksim & \Ncal \! \left ( \Rmat^\transpose \big ( \psit(\CtrMapBis) \big ) (\seq_n) , \Diagmatn \right )
		\vspace{0.15cm} \\
		& \thicksim & \Ncal \! \left ( \CtrMapBis (\seq_n) , \Diagmatn \right ) 
	\end{array}
\end{equation}
since, by definition of $\psit$ and because $\Rmat$ is orthogonal, \eqref{eq: definition of psit} induces that
$$\forall k \in \intk, \Rmat^\transpose \big ( \psit(\CtrMapBis) \big ) (k) = \Rmat^\transpose \Rmat \, \CtrMapBis(k) = \CtrMapBis(k) \myfstop$$
Thus, for any $\YObsSeq \in \MVMbis$ and any $\Ctr_0 \in \Rset^\mydim$,
\begin{equation}
	\label{eq: the composite map belongs to the correct set}
	\fn{\Rmat^\transpose} \circ \Obs \circ \psit \in \MVMter \myfstop
\end{equation}
For any $\CtrMap \in \aSetbis$, it follows from Definition \ref{definition: the distance} and \eqref{eq: definition of invpsit} that $\mydistfn{}{\Rmat^\transpose \! \Ctr_0} \circ \invpsit = \mydist{}$. Since $\invpsit$ is the inverse of $\psit$, the above equality induces that 
\begin{equation}
	\label{eq: change of distance}
 	\mydistfn{}{\Rmat^\transpose \! \Ctr_0} = \mydist{} \circ \psit \myfstop
     \nonumber
\end{equation}
From this equality and \eqref{eq: change of variable} with $H = \mydist{}$ and $\gamma^{-1} = \psit$, we derive that, for any function $f: \Setbis{\Ctr_0} \to \Rset$, 
\begin{equation}
	\label{eq: change of variable applied to our case}
	\displaystyle \lim_{\mydist{\CtrMap} \to \infty} f(\CtrMap) = \displaystyle \lim_{\mydistfn{\CtrMapBis}{\Rmat^\transpose \! \Ctr_0} \to \infty} f \big ( \psit(\CtrMapBis) \big ) \myfstop
%	\nonumber
\end{equation}
Since $\Rmat$ is orthogonal, it follows from \eqref{eq: esl of a discrete process in our case}, \eqref{eq: equivalence of els} and \eqref{eq: change of variable applied to our case} that $$\esssuplim{\mydist{\CtrMap} \to \infty}{N} \TheFunc(\Ctr,\aObsSeq) = \Ctr$$
if and only if
\begin{onecol}
	\begin{equation}
	\label{eq: equivalence 2}
        \limlimsupbasic{\mydistfn{\CtrMapBis}{\Rmat^\transpose \Ctr_0} \to \infty}{N}{\, \TheFuncb(\Rmat ^\transpose \Ctr, \fn{\Rmat^\transpose} ( \YObsSeq(\psit(\CtrMapBis) ) ) - \Rmat^\transpose \Ctr} 
		= 0 \myfstop
	\end{equation}
\end{onecol}
According to \eqref{eq: the composite map belongs to the correct set}, for any $\CtrMapBis \in \bSetbis$, $\fn{\Rmat^\transpose} \circ \YObsSeq \circ \psit(\CtrMapBis) \in \familyfn{\Diagmat}$. Therefore, the covariance matrices of the elements of $\fn{\Rmat^\transpose} \circ \YObsSeq \circ \psit(\CtrMapBis)$ are all diagonal. By Equivalence $\myclubsuit$ established in Section \ref{subsec: the case of diagonal covariance matrices} for diagonal covariance matrices, the equality in \eqref{eq: equivalence 2} is satisfied if and only if $\Rmat^\transpose \Ctr = \Rmat^\transpose \Ctr_0$, and hence, if and only if $\Ctr = \Ctr_0$ because $\Rmat \in \GL$.

\subsection{Proof of Equivalence (\textbf{B})}
\label{subsec: proof of spade}

Suppose that a sequence $\seq = (\seq_n)_{n \in \Nset}$ of integers valued in $\intk$, a function $w: \RightOpen{0}{\infty} \to \Open{0}{\infty}$ and two sequences $\MCov = (\MCov)_{n \in \Nset}$ and $\estMCov = (\estMCov)_{n \in \Nset}$ of \spdms~satisfy \cref{Assumption: on alpha,Assumption: on w,Assumption: on C,Assumption: on lambda}. 
\medskip \\
\indent
Given any $\ctrind \! \in \! \intk$, any pair $(\Ctr_0,\Ctr) \! \in \Rset^{\mydim} \! \times \! \Rset^{\mydim}$ and any $\YObsSeq \! \in \! \MVMbis$, it follows from \eqref{eq: esl centred} that 
\begin{onecol}
	\begin{equation}
		\label{eq: equiv1}
		\esssuplim{\mydist{\CtrMap} \to \infty}{N} \TheFunc(\Ctr,\aObsSeq) = \Ctr 
		\quad \Leftrightarrow \quad \esssuplim{\mydist{\CtrMap} \to \infty}{N} \left ( \TheFunc(\Ctr,\aObsSeq) - \Ctr \right ) = 0 \myfstop
		\nonumber
	\end{equation}
\end{onecol}
According to \eqref{eq: equivalence of els} \& Proposition \ref{prop: esl with linear transform}, which applies since $\Rmat \in \GL$, we derive from the equivalence above that 
\begin{onecol}
	\begin{equation}
        \begin{array}{lll}
		\esssuplim{\mydist{\CtrMap} \to \infty}{N} \TheFunc(\Ctr,\aObsSeq) = \Ctr \\
        \hspace{2cm}
		\Leftrightarrow \quad 
		\esssuplim{\mydist{\CtrMap} \to \infty}{N} \left (\TheFuncb(\Rmat^\transpose \Ctr,\fn{\Rmat^\transpose} (\aObsSeq) ) - \Rmat^\transpose \Ctr \right ) = 0
        \end{array}
		\nonumber
	\end{equation}
\end{onecol}
and thus, by applying \eqref{eq: esl centred} again, we obtain that 
\begin{onecol}
	\begin{equation}
        \begin{array}{lll}
		\esssuplim{\mydist{\CtrMap} \to \infty}{N} \TheFunc(\Ctr,\aObsSeq) = \Ctr 
        \\
		\hspace{2cm} \Leftrightarrow \quad
		\esssuplim{\mydist{\CtrMap} \to \infty}{N} \TheFuncb(\Rmat^\transpose \Ctr,\fn{\Rmat^\transpose} (\aObsSeq) ) = \Rmat^\transpose \Ctr \myfstop
        \end{array}
		\nonumber
	\end{equation}
\end{onecol}
Whence the result as a consequence of $\myclubsuit$.

%%%%%%%%%%%%%%%%%%%%%%%%%%%%%%%%%%%%%%%%%%%%%%
%%%%%%%%%%%%%%%%%%%%%%%%%%%%%%%%%%%%%%%%%%%%%%
%%%%%%%%%%%%%%%%%%%%%%%%%%%%%%%%%%%%%%%%%%%%%%
%%%%%%%%%%%%%%%%%%%%%%%%%%%%%%%%%%%%%%%%%%%%%%
\section{Instrumental lemmas}
\label{App: Instrumental lemmas}
In what follows, $\normpdf$ denotes the probability density function of $\Ncal(0,1)$ and we write (a.e) for ``almost everywhere'' with respect to the Lebesgue measure on $\Rset$. As usual, $L^1(\Rset)$ denotes the set of all integrable functions $f: \Rset \to \Rset$ with respect to Lebesgue measure and the $\sgn$ function is defined for all $x \in \Rset$ by setting:
$$\sgn(x) = 
\left \{
\begin{array}{lll}
	1 & \text{if} & x > 0 \mycomma
	%\vspace{0.1cm} 
	\\
	-1 & \text{if} & x < 0 \mycomma
	%\vspace{0.1cm} 
	\\
	0 & \text{if} & x = 0 \myfstop
\end{array}
\right.
$$

\begin{Definition}
	\label{def: fn even in each coord}
	We say that a function $\fbm: \Rset^{\mydim} \to \RightOpen{0}{\infty}$ is even in each coordinate of $\xvec = (\xvec_1, \ldots, \xvec_m)^\transpose \in \Rset^{\mydim}$ if, for any $i \in \llbracket 1, \mydim \, \rrbracket$, $\fbm(\ldots, -x_i, \ldots ) = \fbm(\ldots, x_i, \ldots)$ when the coordinates other than $x_i$ remain unchanged.
\end{Definition}

\vspace{0.1cm}
To state and prove Lemmas \ref{Lemma: Impaire} and \ref{Lemma:The last very useful lemma} below, we introduce the following functions:
\begin{onecol}
\begin{equation}
	\label{eq:function_uab}
	\forall (a,b) \in \RightOpen{0}{\infty} \times \RightOpen{0}{\infty}, \forall x \in \RightOpen{0}{\infty}, u_{a,b}(x) = \dfrac{1}{\sqrt{2 \pi}} e^{-x^2 / 2} e^{-b^2/2} \left (e^{ax} - e^{-ax} \right ) \myfstop
\end{equation}
\end{onecol}
\begin{Lemma}
	\label{Lemma: Impaire}
		If $G: \Rset \to \Rset$ is odd and non-negative for any $x \geqslant 0$ and if, for any $\mu \in \Rset$, $G \times  \normpdf(\cdot-\mu) \in L^1(\Rset)$ then, for any $\mu \in \Rset$ and any $(a,b) \in \RightOpen{0}{\infty} \times \RightOpen{0}{\infty}$ such that $0 \leqslant a \leqslant \abs{\mu} \leqslant b < \infty$,
	\begin{equation}
		\label{eq: int of G}
		\displaystyle \int_0^\infty G(x) u_{a,b}(x) \dx \leqslant \sgn (\mu) \displaystyle \int_{-\infty}^\infty G(x) \normpdf \left ( x-\mu \right ) \dx \myfstop
		\nonumber
	\end{equation}
\end{Lemma}
\begin{proof}
Suppose that $G: \Rset \to \Rset$ is odd and non-negative for any $x \geqslant 0$. Fix $\mu \in \Rset$. Readily, we have
\begin{onecol}
	\begin{equation}
		\label{eq: integral of G times p}
			\begin{array}{lll}
				\displaystyle \int_{-\infty}^\infty G(x) \normpdf(x-\mu) \dx 
				& = & \dfrac{1}{\sqrt{2 \pi}} \displaystyle \int_0^\infty G(x) \left ( e^{-\frac{1}{2}(x - \mu)^2} - e^{-\frac{1}{2}(x + \mu)^2} \right ) \dx
				\Big ) \dx \vspace{0.25cm} \\
				% & = & \displaystyle \int_0^\infty G(x) \Big ( \normpdf(x - \mu) - \normpdf(x + \mu) \Big ) \dx \vspace{0.2cm} \\
				& = & \displaystyle \int_0^\infty G(x) u_{\mu,\mu}(x) \dx	\myfstop
			\end{array}
	\end{equation}
\end{onecol}
In addition,
\begin{onecol}
\begin{equation}
	\label{eq: umumu2}
	\forall x \in \Rset, \sgn(\mu) u_{\mu,\mu}(x) = \frac{1}{\sqrt{2 \pi}} e^{-x^2 / 2} e^{-\mu^2/2} \left (e^{ \vert \mu \vert x} - e^{- \vert \mu \vert x} \right ) \myfstop
\end{equation}
\end{onecol}
If $(a,b) \in \RightOpen{0}{\infty} \times \RightOpen{0}{\infty}$ verifies $0 \leqslant a \leqslant \abs{\mu} \leqslant b < \infty$, it straightforwardly results from \eqref{eq: umumu2} and \eqref{eq:function_uab} that
\begin{equation}
	\label{eq: inequality in lemma impaire}
	\forall x \geqslant 0, \sgn(\mu) {u_{\mu,\mu}(x)} \geqslant u_{a,b}(x) \geqslant 0 \myfstop
\end{equation}
The result thus derives from \eqref{eq: integral of G times p} and \eqref{eq: inequality in lemma impaire}, since $G(x) \geqslant 0$ if $x \geqslant 0$. 
\end{proof}
\begin{Lemma}
	\label{Lemma: Paire}
	If $F: \Rset \to \RightOpen{0}{\infty}$ is even and $F \times \normpdf(\cdot - \mu) \in L^1(\Rset)$ for any $\mu \in \Rset$ then, for any $\mu \in \Rset$ and any pair $(a,b) \in \RightOpen{0}{\infty} \times \RightOpen{0}{\infty}$ such that $0 \leqslant a \leqslant \abs{\mu} \leqslant b < \infty$,
	\[
	\displaystyle \int_{-\infty}^\infty \hspace{-0.2cm} F(x) \normpdf \left ( x-\mu \right ) \dx \geqslant e^{-\frac{1}{2} (b^2 - a^2)} \hspace{-0.1cm} \displaystyle \int_0^\infty \hspace{-0.2cm} F(x) \normpdf \left ( x - a \right ) \dx \myfstop
	\]
\end{Lemma}
\begin{proof}
Since $F$ is even, we have:
\begin{onecol}
	\begin{equation}
		\label{eq: equality in lemma paire}
		\displaystyle \int_{-\infty}^\infty F(x) \normpdf\left ( x-\mu \right ) \dx = \dfrac{1}{\sqrt{2 \pi}} \displaystyle \int_0^\infty F(x) e^{- x^2/2} e^{-\mu^2/2} \left ( e^{\mu x} + e^{-\mu x} \right ) \dx \myfstop
	\end{equation}
\end{onecol}
If $(a,b) \in \RightOpen{0}{\infty} \times \RightOpen{0}{\infty}$ verifies $0 \leqslant a \leqslant \abs{\mu} \leqslant b < \infty$, we have the following inequalities:
	\begin{equation}
		\label{eq: easy facts for lemma paire}
		\forall x \geqslant 0, 
		\left \{
		\begin{array}{lll}
			e^{\mu x} + e^{-\mu x} \geqslant e^{ \abs{\mu} x } \geqslant e^{a x} \mycomma \vspace{0.1cm} \\
			e^{-\mu^2/2} \geqslant e^{-b^2 /2} \myfstop
		\end{array}
		\right.
		\nonumber
	\end{equation}
	Thence the result by injecting the inequalities above into \eqref{eq: equality in lemma paire}, since $F \geqslant 0$.
\end{proof}
\begin{Lemma}
	\label{Lemma:The last very useful lemma} %$ $ \\
	%	\noindent
	Suppose that $[\constmin,\constmax] \subset \Open{0}{\infty}$ and consider any $\meanvec = (\mean_1 \ldots, \mean_\mydim)^\transpose \in \Rset^{\mydim}$.	Let $(\Xvec_m)_{m \in \Nset}$ be a sequence of $\mydim$-dimensional real Gaussian random vectors such that $$\forall m \in \Nset, \Xvec_m = (X_{m,1}, \ldots, X_{m,\mydim})^\transpose \thicksim \Ncal(\Coeff_m \meanvec, \Identity{\mydim})$$ where, for any $m \in \Nset$, $\Coeff_m$ is a diagonal matrix {whose diagonal elements belong to} $[\constmin,\constmax]$. If $(\fbm_m)_{m \in \Nset}$ is a sequence of positive and bounded functions $\fbm_m: \Rset^{\mydim} \to \Open{0}{\infty}$ that are even in each coordinate and if there exists a positive function $\fbm^*: \Rset^{\mydim} \to \Open{0}{\infty}$ such that $\forall m \in \Nset, \displaystyle \fbm^* \leqslant \fbm_m$ (a.e), %(\ed{where (a.e) stands for almost everywhere}), 
	then 
	\begin{onecol}
		\begin{equation}
			\label{eq: equivalence between the limit of the sum and a null mean}
			\left ( \forall i \in \intd, \displaystyle \lim_{M \to \infty} \dfrac{1}{M} \displaystyle \sum_{{m = 1}}^M 	\Expect{\fbm_m(\Xvec_m) X_{m,i}} = 0 \right ) \, \, \Leftrightarrow \, \, \mean = 0 \myfstop
		\end{equation}
	\end{onecol}
\end{Lemma}
\begin{proof} 
We prove the direct implication only, the converse being straightforwardly true because each $\fbm_m$ is even in its $i^{\text{th}}$ variable. Choose arbitrarily any $m \in \Nset$. To fix ideas, consider the first coordinate $X_{m,1}$ of $\Xvec_m$. This choice is convenient to alleviate notation without incurring any loss of generality. For any given $m \in \Nset$, we put $\Coeff_m = \diag(\const_{m,1},\ldots,\const_{m,\mydim})$. Since $\fbm_m$ is bounded, $\Expect{\, \vert \, \fbm_m(\Xvec_m) X_{m,1} \, \vert \,} < \infty$. Therefore, we have
\begin{onecol}
	\begin{equation}
		\label{eq: general expression of E[f(Xn)Xn1]}
		\Expect{f_m (\Xvec_m) X_{m,1}} 
		= \displaystyle \int_{-\infty}^\infty  x_1 \, g_m(x_1) \normpdf \left ( x_1 - \const_{m,1} \mean_1 \right ) \dx_1 < \infty		
	\end{equation}
\end{onecol}
with
\begin{subequations}
	\label{eq: definition of gm}
	\begin{empheq}[left={\forall x_1 \in \Rset, g_m(x_1) =} 
    \empheqlbrace]{align}
        & \, \fbm_m(x_1) \mycomma & \text{if $\mydim = 1$} \mycomma \label{eq: definition of gm-a} 
		\\
		& \, \displaystyle \int_{\Rset^{{\mydim}-1}} \fbm_m(x_1, \ldots, x_{\mydim}) \prod_{i=2}^{\mydim} \normpdf \left ( x_i - \const_{m,i} \mean_i \right ) \dx_i \mycomma & \text{if $\mydim \geqslant 2$} \mycomma
        \label{eq: definition of gm-b}
	\end{empheq}
\end{subequations}
where \eqref{eq: definition of gm-b} is obtained via Fubini's Theorem. The function $G:x_1 \in \Rset \mapsto x_1 \, g_m(x_1) \in \Rset$ is odd because $\fbm_{m}$ is even in each of its coordinate. Furthermore, $G(x_1) \geqslant 0$ if $x_1 \geqslant 0$. We also have $0 \leqslant \constmin \abs{\mean_1} \leqslant \const_{m,1} \abs{\mean_1} \leqslant \constmax \abs{\mean_1} < \infty$. In addition, $$\int_{-\infty}^\infty G(x_1) \normpdf(x_1-\const_{m,1} \mean_1) \dx < \infty$$ by \eqref{eq: general expression of E[f(Xn)Xn1]}. Thence, Lemma \ref{Lemma: Impaire} applies to $G$ and yields that
	\begin{onecol}
		\begin{equation}
			\label{eq: basic equality to bound E[f(Xn)Xn1]}
			\displaystyle \int_0^\infty x_1 g_m(x_1) u(x_1,\mean_1) \dx_1 \leqslant \text{sign}(\const_{m,1} \mean_1) \displaystyle \int_{-\infty}^\infty x_1 \, g_m(x_1) \normpdf(x_1 - \const_{m,1} \mean_1) \dx_1 
		\end{equation}
	\end{onecol}
	where $u(x_1,\mean_1) = u_{\constmin|\mean_1|,\constmax|\mean_1|}(x_1)$ is calculated according to \eqref{eq:function_uab}. Since $\const_{m,1} > 0$, $\text{sign}(\const_{m,1} \mean_1) = \text{sign}(\mean_1)$. We thus derive from \eqref{eq: basic equality to bound E[f(Xn)Xn1]} and \eqref{eq: general expression of E[f(Xn)Xn1]} that
	\begin{equation}
		\label{eq: first equality to bound E[f(Xn)Xn1]}
		\hspace{-0.5cm} \displaystyle \int_0^\infty \hspace{-0.25cm} x_1 g_m(x_1) u(x_1,\mean_1) \dx_1 %\\ 
		\leqslant 
		\text{sign}(\mean_1) \, \Expect{{f_m} (\Xvec_{m} ) X_{m,1}}
		\hspace{-0.4cm} 
	\end{equation}
	Set $\theconst^2 = \frac{1}{2} \left ( \constmax^2 - \constmin^2 \right )$ and suppose that $\mydim \geqslant 2$. 
    %Given any $x_1 \in \Rset$, any $j \in {\llbracket 2 , \mydim \, \rrbracket}$, and, if $j \leqslant \mydim-1$, any real values $x_{j+1}, \ldots, x_{\mydim}$, we now prove that:     
    We now prove that, for any $j \in {\llbracket 2 , \mydim \, \rrbracket}$, any $x_1 \in \Rset$ and, if $j \leqslant \mydim-1$, any real values $x_{j+1}, \ldots, x_{\mydim}$,
	\begin{equation}
		\label{eq: recurrence}
		\hspace{-0.35cm}
		\begin{array}{lll}
			\displaystyle \int_{\Rset^{j-1}} \fbm_m(x_1, x_2, \ldots, x_{\mydim}) \prod_{i = 2}^{j} \normpdf(x_i - \const_{m,i} \mean_i) \dx_i \vspace{0.1cm} \\ 
			\hspace{2cm} \geqslant \, e^{-\theconst^2 \sum_{i=2}^{j} \mean_i^2} \displaystyle \int_{\mathbb{R}_{+}^{j-1}} \fbm_m(x_1, x_2, \ldots, x_{\mydim}) \prod_{i=2}^{j} \normpdf(x_i - \constmin \abs{\mean_i}) \dx_i \myfstop
		\end{array}
	\end{equation}
	\noindent
    We proceed by recurrence and begin with $j = 2$. Consider any $x_1 \in \Rset$ and, if $\mydim \geqslant 3$, any $x_3, \ldots, x_\mydim$ in $\Rset$. We then define $F_2: \Rset \to \RightOpen{0}{\infty}$ for every $x_2 \in \Rset$ by setting $$F_2(x_2) = \fbm_m(x_1, x_2, \ldots, x_{\mydim}) \myfstop$$
	The function $F_2$ is even and 
	$$\displaystyle \int_{-\infty}^\infty F_2(x_2) \normpdf(x_2 - \const_{m,2} \mean_2) \dx_2 < \infty$$ 
	because $\fbm_m$ is bounded. Since $0 \leqslant \constmin \abs{\mean_2} \leqslant \const_{m,2} \abs{\mean_2} \leqslant \constmax \abs{\mean_2} < \infty$, Lemma \ref{Lemma: Paire} applied to $F_2$ induces that
    \begin{onecol}
		\begin{equation}
			\label{eq: recurrence (j=2)}
            \begin{array}{lll}
			\displaystyle \int_{-\infty}^\infty \fbm_m(x_1, x_2, \ldots, x_{\mydim}) \normpdf(x_2 - \const_{m,2} \mean_2) \dx_2 \\
            \hspace{2cm} \geqslant \, e^{-\theconst^2 \mean_2^2} \displaystyle \int_0^\infty \fbm_m(x_1, x_2, \ldots, x_{\mydim}) \normpdf(x_2 - \constmin \abs{\mean_2} ) \dx_2
            \end{array}
			\nonumber
		\end{equation}
	\end{onecol}
	and \eqref{eq: recurrence} is thus verified for $j = 2$. If $\mydim = 2$, the recurrence stops here.
	\medskip \\
	\indent
	Suppose that $\mydim \geqslant 3$ and that \eqref{eq: recurrence} is satisfied for a given $j \in {\llbracket 2 , \mydim-1\, \rrbracket}$. Given any $x_1 \! \in \! \Rset^{\mydim}$ and any given real values $x_{j+1}, \ldots, x_{\mydim}$, the Fubini theorem implies:
	\begin{onecol}
		\begin{equation}
			\label{eq: recurrence (j)-0}
			\begin{array}{lll}
				\displaystyle \int_{\Rset^{j}} \fbm_m(x_1, \ldots, x_{\mydim}) \prod_{i=2}^{j+1} \normpdf(x_i - \const_{m,i} \mean_i ) \dx_i \vspace{0.1cm} \\ %\medskip \\
				= \displaystyle \int_{-\infty}^\infty\bigg (  \int_{\Rset^{j-1}} \! \fbm_m(x_1, \ldots, x_{\mydim}) \prod_{i=2}^{j} \normpdf(x_i - \const_{m,i} \mean_i ) \dx_i \bigg )
				\times \normpdf(x_{j+1} - \const_{m,j+1} \mean_{j+1} ) \dx_{j+1} \myfstop
			\end{array}
		\nonumber
		\end{equation}
	\end{onecol}
Therefore, according to \eqref{eq: recurrence}, we obtain:
\begin{equation}
	\label{eq: recurrence (j)-1}
	\begin{array}{lll}
		\displaystyle \int_{\Rset^{j}} \fbm_m(x_1, \ldots, x_{\mydim}) \prod_{i=2}^{j+1} \normpdf(x_i - \const_{m,i} \mean_i ) \dx_i
		\vspace{0.15cm} \\ 			
		\hspace{1cm} \geqslant e^{-\theconst^2 \sum_{i=2}^j \mean_i^2} 
		\displaystyle \int_{-\infty}^\infty \hspace{-0.1cm} \bigg (\displaystyle \int_{\mathbb{R}_{+}^{j-1}} \fbm_m(x_1, \ldots, x_{\mydim}) \prod_{i=2}^j \normpdf \left (x_i - \constmin \abs{\mean_i} \right ) \dx_i \bigg ) \vspace{0.15cm} \\
		\hspace{8cm} \times \, \normpdf(x_{j+1} - \const_{m,j+1} \mean_{j+1} ) \dx_{j+1} \mycomma
	\end{array}
\end{equation}
%\indent
Define the map $F_{j+1}: \Rset \to \RightOpen{0}{\infty}$ by setting:
	\begin{onecol}
		\[
		\forall x_{j+1} \in \Rset, F_{j+1}(x_{j+1}) = e^{-\theconst^2 \sum_{i=2}^j \mean_i^2} \int_{\mathbb{R}_{+}^{j-1}} \fbm_m(x_1, \ldots, x_{\mydim}) \prod_{i=2}^j \normpdf(x_i - \constmin \abs{\mean_i} ) \dx_i \myfstop
		\]
	\end{onecol}
	This function is even and $\fbm_m$ being bounded, \eqref{eq: recurrence (j)-1} implies that $F_{j+1} \normpdf(\cdot- \const_{m,j+1} \mean_{j+1}) \in L^1(\Rset)$. As $0 \leqslant \constmin \abs{\mean_{j+1}} \leqslant \const_{m,j+1} \abs{\mean_{j+1}} \leqslant  \constmax  \abs{\mean_{j+1}} < \infty$, we apply Lemma \ref{Lemma: Paire} to $F_{j+1}$, which leads to:
\begin{onecol}
\begin{equation}
 	\label{eq: recurrence (j)-2}
 	\hspace{-0.1cm}
 	\begin{array}{lll}
 		e^{-\theconst^2 \sum_{i=2}^j \mean_i^2} %\vspace{0.1cm} \\
 		%\times 
 		\displaystyle \int_{-\infty}^\infty
 		\hspace{-0.1cm} 
 		\bigg ( \int_{\mathbb{R}_{+}^{j-1}} \hspace{-0.1cm} 
 		\fbm_m(x_1, \ldots, x_{\mydim}) \prod_{i=2}^j \! \normpdf(x_i \! - \! \constmin \abs{\mean_i}) \dx_i \! \bigg ) \vspace{0.1cm} \\
 		\hspace{7cm} \times \, \normpdf(x_{j+1} - \const_{m,j+1} \mean_{j+1} ) \dx_{j+1} \vspace{0.1cm} \\
 		\hspace{1cm} 
 		\geqslant \, e^{-\theconst^2 \sum_{i=2}^{j+1} \mean_i^2}
 		\displaystyle \int_{\mathbb{R}_{+}^{j}} \hspace{-0.1cm} \fbm_m(x_1, \ldots, x_{\mydim}) \prod_{i=2}^{j+1} \normpdf(x_i - \constmin \abs{\mean_i} ) \dx_i \myfstop
 	\end{array}
 	\vspace{0.1cm}
\end{equation}
\end{onecol}
The recurrence follows by combining \eqref{eq: recurrence (j)-1} and \eqref{eq: recurrence (j)-2}.
\vspace{0.1cm} \\
\indent
When $j = \mydim$ and $\mydim \geqslant 2$, \eqref{eq: recurrence}, \eqref{eq: definition of gm-b} {and the inequality $\sum_{i=2}^j \mean_i^2 \leqslant \norm[\meanvec]^2$} imply that:
\begin{onecol}
\begin{equation}
	\label{eq: lower bound on the recurrence-0}
	\forall x_1 \in \Rset^{\mydim}, g_m(x_1) \geqslant e^{-\theconst^2 \norm[\meanvec]^2} \displaystyle \int_{\mathbb{R}_{+}^{{\mydim}-1}} \fbm_m(x_1, \ldots, x_{\mydim}) \prod_{i=2}^{{\mydim}} \normpdf(x_i - \constmin \abs{\mean_i} ) \dx_i \myfstop
	% \nonumber
\end{equation}
\end{onecol}
Since $\fbm_m(x_1, x_2, \ldots, x_{\mydim}) \geqslant \fbm^*(x_1, x_2, \ldots, x_{\mydim}) > 0$ for almost every $(x_1, \ldots, x_{\mydim}) \in \Rset^{\mydim}$, \eqref{eq: lower bound on the recurrence-0} and \eqref{eq: definition of gm} induce the following lower bound:
\begin{equation}
\label{eq: lower bound on the recurrence}
\forall x_1 \in \Rset^{\mydim}, g_m(x_1) \geqslant \integral 
\end{equation}
with
\begin{onecol}
\begin{equation}
	\label{eq: definition of integral}
	\integral = %(x_1,\mean_2, \ldots, \mean_{\mydim}) = 
    \left \{
    \begin{array}{lll}
    \hspace{-0.1cm} 
    \fbm^*(x_1) \mycomma & \text{if $\mydim = 1$} \mycomma
    \vspace{0.2cm} \\
    \hspace{-0.1cm} 
    e^{-\theconst^2 \norm[\meanvec]^2} \displaystyle \int_{\mathbb{R}_{+}^{\mydim-1}} \fbm^*(x_1, \ldots, x_{\mydim)} \, \prod_{i=2}^{{\mydim}} \normpdf(x_i- \constmin \abs{\mean_i}) \dx_i \mycomma & \text{if $\mydim \geqslant 2$} \myfstop
    \end{array}
    \right.
%	\nonumber
\end{equation}
\end{onecol}
We now derive from \eqref{eq: first equality to bound E[f(Xn)Xn1]}, \eqref{eq: lower bound on the recurrence} and \eqref{eq: definition of integral} that
$$
\forall m \in \Nset, 
0 \leqslant \displaystyle \int_0^\infty x_1 \integral %(x_1,\mean_2, \ldots, \mean_{\mydim}) 
u(x_1,\mean_1) \dx_1 
\leqslant \text{sign}(\mean_1) \Expect{f_m (\Xvec_m) X_{m,1}}
$$
and thus that
\begin{onecol}
\begin{equation}
	\label{eq: the final lower bound on sumE[f(Xn)Xn1]}
	0 \leqslant %e^{-\theconst^2 \norm[\meanvec]^2} 
    \displaystyle \int_0^\infty x_1 \integral %(x_1,\mean_2, \ldots, \mean_{\mydim}) 
    u(x_1,\mean_1) \dx_1
	\leqslant \text{sign}(\mean_1) \frac{1}{M} \displaystyle \sum_{m=1}^M \Expect{f_m (\Xvec_m) X_{m,1}} \myfstop
	\nonumber
\end{equation}
\end{onecol}
Suppose that $\lim\limits_{M \to \infty} \frac{1}{M} \sum_{m=1}^M \Expect{f_m (\Xvec_m) X_{m,1}} = 0$. The inequalities above then imply that $$\int_0^\infty x_1 \integral u(x_1,\mean_1) \dx_1 = 0 $$ and thus that $x_1 \integral u(x_1,\mean_1) = 0$ for almost every $x_1 \in \RightOpen{0}{\infty}$. Because $x_1 \integral > 0$ for any $x_1 > 0$, $x_1 \integral u(x_1,\mean_1) = 0$ for almost every $x_1 \in \RightOpen{0}{\infty}$ implies that $u(x_1,\mean_1) = 0$ for almost every $x_1 \in \Open{0}{\infty}$ and thus, that $\mean_1 = 0$ because, as shown by a straightforward computation, for any given $x_1 \neq 0$, $u(x_1,\mean_1) = 0$ if and only if $\mean_1 = 0$. We can reiterate the same reasoning for any other $i =2, \ldots, \mydim$ to conclude. 
\end{proof}
\indent
In \eqref{eq: example of a base via a function}, take the Euclidean norm $\norm[\cdot]$ in $\Rset^\mydim$ for $H$. Given a function $f: \Rset^\mydim \to \Rset$, we derive from \eqref{eq: lim over B(H)} that
\begin{onecol}
\begin{equation}
	\label{eq: limit when the norm tends to infty}
	\displaystyle \lim_{\norm[\xvec] \to \infty} f(\xvec) = 0 
	\Leftrightarrow \big ( \, \forall \eta \in \Open{0}{\infty}, \exists \, r \in \RightOpen{0}{\infty}, \forall \xvec \in \Rset^{\mydim}, \norm[\xvec] > r \Rightarrow \vert f(\xvec) \vert < \eta \, \big ) \myfstop
\end{equation}
\end{onecol}
Similarly, $\Lone$ being the set of all $\Xvec \in \rv$ such that $\Exp{ \norm[\Xvec] } < \infty$, consider 
$$
\begin{array}{ccccc}
H: 
& \Lone & \to & \RightOpen{0}{\infty} \vspace{0.1cm} \\
& \Xvec & \mapsto & \| \, \Ebb \, [ \Xvec ] \, \| 
\end{array}
$$
in \eqref{eq: example of a base via a function}. Given $f: \Lone \to \RightOpen{0}{\infty}$, it follows from \eqref{eq: lim over B(H)} that
\begin{equation}
	\label{eq: limit when the norm tends to infty for random vector}
	\begin{array}{lll}
		\hspace{-0.4cm} 
		\displaystyle \lim_{\| \, \Ebb \, [ \Xvec ] \, \| \to \infty} \fbm(\Xvec) = 0 \vspace{0.2cm} \\
		\Leftrightarrow \big ( \forall \eta \in \Open{0}{\infty}, \exists \, r \in \RightOpen{0}{\infty}, \forall \Xvec \in \rv, 
		\| \, \Ebb \, [ \Xvec ] \, \| > r \Rightarrow \vert f(\xvec) \vert < \eta \big ) \myfstop
	\end{array}
\end{equation}
\begin{Lemma}
	\label{Lemma: asymptotic behavior-1}
	Let $\Xvec(\xivec) \thicksim \Ncal(\xivec,\sigma^2 \Identity{{\mydim}})$ with $\sigma > 0$ and $g: \RightOpen{0}{\infty} \to \RightOpen{0}{\infty}$ be a non-null continuous function. If $\displaystyle \lim\limits_{t \rightarrow \infty} {g(t)} = 0$ then $$\displaystyle \lim\limits_{ \| \xivec \| \rightarrow \infty} \Expect{g ( \| \Xvec(\xivec) \|)} = 0 \myfstop$$
\end{Lemma}

%\vspace{0.1cm}
\begin{proof} 
This straightforward application of the Lebesgue dominated convergence theorem is provided for self-completeness of the paper. 

\vspace{0.1cm}
Let $g: \RightOpen{0}{\infty} \to \RightOpen{0}{\infty}$ be a non-null continuous function and suppose that $\Xvec(\xivec) \thicksim \Ncal(\xivec,\sigma^2 \Identity{{\mydim}})$ with $\sigma > 0$. We have:
\begin{onecol}
\begin{equation}
	\label{eq: limit of E[w(X)2] in appendix}
	\begin{array}{lll}
		\Expect{g ( \| \Xvec(\xivec) \| ) } 
		& = & \dfrac{1}{(2 \pi)^{{\mydim}/2} \sigma^{\mydim}} \displaystyle \int_{\Rset^{\mydim}} g \big ( \| \xvec \| \big ) e^{-\| \xbm- \xibm \| / 2 \sigma^2} \dxvec \vspace{0.25cm} \\
		& = & \dfrac{1}{(2 \pi)^{{\mydim}/2} \sigma^{\mydim}} \displaystyle \int_{\Rset^{\mydim}} g \big (
		\| \yvec + \xivec \| \big ) e^{-\| \yvec \| / 2 \sigma^2} \dyvec \myfstop
	\end{array}
\end{equation}
\end{onecol}
Since $g$ is continuous and $\displaystyle \lim\limits_{t \rightarrow \infty} {g(t)} = 0$, $g$ is bounded and so are the values $g \left ( \| \yvec + \xivec \| \right )$ for $(\yvec,\xivec) \in \Rset^{\mydim} \times \Rset^{\mydim}$. In addition, given any $\yvec \in \Rset^{\mydim}$, since $\displaystyle \lim\limits_{t \rightarrow \infty} {g(t)} = 0$, it results from \eqref{eq: limit when the norm tends to infty} that $\displaystyle \lim\limits_{ \| \xivec \| \rightarrow \infty} g \left ( \| \yvec + \xivec \| \right ) = 0$. The result thus follows from \eqref{eq: limit of E[w(X)2] in appendix} and the Lebesgue dominated convergence theorem.
\end{proof}
\indent
We now use Definition \ref{def: Ndiagcov} introduced in Appendix \ref{App: the proof} to state our next result.
\begin{Lemma}
\label{Lemma: asymptotic behavior-2}
If $[\constmin,\constmax] \! \subset \! \Open{0}{\infty}$ and $\myf: \! \RightOpen{0}{\infty} \! \to \! \RightOpen{0}{\infty}$ is a continuous function such that $\displaystyle \lim\limits_{t \rightarrow \infty} \myf(t) = 0$, then $$\displaystyle \lim\limits_{ \| \Ebb [{\Xvec}] \| \rightarrow \infty} \Expect{\myf ( \| \Xvec \| ) } = 0$$
uniformly in $\Xvec \in \fnGdiag{{\constmin^2}}{{\constmax^2}}$ in that:
\begin{onecol}
\begin{equation}
	\label{eq: crucial uniform convergences}
	\begin{array}{c}
		\forall \, \varepsilon > 0, \exists \, \radius > 0, \forall \Xvec\in \fnGdiag{{\constmin^2}}{{\constmax^2}},  \| \, \Ebb[{\Xvec}] \, \| > \radius \Rightarrow \, \Expect{ \, \myf ( \, \| \Xvec \| \, ) \, } < \varepsilon \myfstop
	\end{array}
\end{equation}
\end{onecol}
\end{Lemma}
\begin{proof}
Given any $\meanvecc \in \Rset^{\mydim}$, set
\begin{onecol}
$$
F(\meanvecc) 
= \dfrac{1}{(2 \pi)^{\mydim/2} \constmin^\mydim} \displaystyle \int_{\Rset^{\mydim}} \myf \big ( \| \xvec \| \big ) e^{-\| \xvec - \meanvecc \|^2 / 2 \constmax^2} \dxvec 
= \dfrac{\constmax^{\mydim}}{\constmin^{\mydim}} \, \Expect{ \myf ( \| \Zvec (\meanvecc) \| ) }
$$
\end{onecol}
where $\Zvec (\meanvecc)$ is any $\mydim$-dimensional real random vector with distribution $\Ncal(\meanvecc, \constmax^2 \, \Identity{{\mydim}})$. It follows from Lemma \ref{Lemma: asymptotic behavior-1} that $\lim\limits_{\| \meanvecc \| \to \infty} \Expect{ \myf ( \| \Zvec (\meanvecc) \| ) } = 0$ and thus that
\begin{equation}
	\label{eq: limit of F}
    \lim\limits_{\| \meanvecc \| \to \infty} F(\meanvecc) = 0 \myfstop
\end{equation} 
Consider any $\Xvec \in \fnGdiag{{\constmin^2}}{\constmax^2}$. Let $\const^2_1, \ldots, \const^2_\mydim$ be the diagonal elements of the covariance matrix of $\Xvec$ and set $\meanvecc = (\meanveccc_1, \ldots, \meanveccc_\mydim)^\transpose = \Ebb [ \Xvec]$. We thus have:
\begin{onecol}
\[
\Exp{ \myf ( \norm[\Xvec]) } = \dfrac{1}{(2 \pi)^{\mydim/2} \, \prod_{\dimi = 1}^\mydim \! \const_{\dimi}} \displaystyle \int_{\Rset^{\mydim}} \myf \big ( \| \xvec \| \big ) e^{-\frac{1}{2} \sum_{\dimi = 1}^\mydim \! {(\xvec_\dimi-\meanveccc_\dimi)^2}/{\const^2_{\dimi}}} \, \dxvec \myfstop
\]
\end{onecol}
Since $\prod_{\dimi = 1}^\mydim \! \const_{\dimi} \geqslant \constmin^{\mydim}$ and $\sum_{\dimi = 1}^\mydim \! {(x_\dimi-\meanvecc_\dimi)^2}/{\const^2_{\dimi}} \! \geqslant \! \norm[\xvec - \meanvecc]^2 / {\constmax^2}$ for any $\xvec \! \in \! \Rset^\mydim$, we have
$\Expect{\myf ( \norm[\Xvec] ) } \leqslant F(\meanvecc)$. Thence the result as a consequence of \eqref{eq: limit of F}.
\end{proof}

%%%%%%%%%%%%%%%%%%%%%%%%%%%%%%%%%%%%%%%%%%%%%%
%%%%%%%%%%%%%%%%%%%%%%%%%%%%%%%%%%%%%%%%%%%%%%
%%%%%%%%%%%%%%%%%%%%%%%%%%%%%%%%%%%%%%%%%%%%%%
%%%%%%%%%%%%%%%%%%%%%%%%%%%%%%%%%%%%%%%%%%%%%%

\section{Proof of Proposition \ref{Proposition:the iterative method converges}}
\label{App: proof of the convergence}
This proof, up to \eqref{eq: absolute value of the difference of elementary risks} below, relies on the same type of arguments as in the Appendix of \citet{comaniciu2002mean}. To begin with, through the same steps as those employed to derive \eqref{eq: we rewrite G}, it follows from \eqref{eq: our general function} that:
\begin{equation}
	\label{eq: MyNewfunc_app: general-prelim}
	\hspace{-0.3cm}
	\begin{array}{lll}
		\forall \, (\ell,N) \in \Nset \times \Nset, 
		\medskip \\
		\hspace{0.5cm}
        \begin{array}{lll}
		\Ctr^{(\ell+1)} 
        & = & \Thefuncbasic(\Ctr^{(\ell)}) \medskip \\
		& = & \Rmat \left ( \displaystyle \sum_{n=1}^N w \big ( \Maha{\estDiagmatn}^2 (\Rmat^\transpose \obs_n - \Rmat^\transpose \Ctr^{(\ell)}) \big ) \investDiagmatn \right )^{\! -1} 
        \medskip \\
        & & \hspace{2cm}
		\times \displaystyle \sum_{n=1}^N w \big ( \Maha{\estDiagmatn}^2 (\Rmat^\transpose \obs_n - \Rmat^\transpose \Ctr^{(\ell)}) \big ) \investDiagmatn \Rmat^\transpose \obs_n \myfstop
        \end{array}
	\end{array}
\end{equation}
For any $\ell \in \Nset$, set $\CCtr^{(\ell)} = \Rmat^\transpose \Ctr^{(\ell)}$ and, in addition, for any $n \in \Nset$, $\oobs_n = \Rmat^\transpose \obs_n$. We can thus rewrite \eqref{eq: MyNewfunc_app: general-prelim} as:
\begin{equation}
	\label{eq: MyNewfunc_app: general}
	\hspace{-0.3cm}
	\begin{array}{lll}
		\forall \, {(\ell,N) \in \Nset \times \Nset}, 
		\medskip \\
		\hspace{0.3cm}
        \begin{array}{lll}
		\CCtr^{(\ell+1)}
		& = & \Rmat^\transpose \Ctr^{(\ell+1)} 
        \medskip \\
        & = & \left ( \displaystyle \sum_{n=1}^N w \big ( \Maha{\estDiagmatn}^2 (\oobs_n-\CCtr^{(\ell)}) \big ) \investDiagmatn \right )^{\! -1} \hspace{-0.1cm}
		\times \displaystyle \sum_{n=1}^N w \big ( \Maha{\estDiagmatn}^2 (\oobs_n-\CCtr^{(\ell)}) \big ) \investDiagmatn \oobs_n \myfstop
        \end{array}
	\end{array}
\end{equation}
For any $(\ell,N) \in \Nset \times {\Nset}$, set $$\elrisk(\ell) = \sum_{n=1}^N \risk(\Maha{\estDiagmatn}^2 ( \CCtr^{(\ell)} - \oobs_n )) \mycomma$$ where $\risk$ has derivative $w$ according to \eqref{eq:definition of w}. 
Since $w$ is non-increasing and differentiable, we have $\risk'' = w' \leqslant 0$. It follows that $\risk$ increases --- because $w$ is valued in $\Open{0}{\infty}$ --- and concave. The concavity of $\risk$, \eqref{eq:definition of w} and \eqref{eq: definition of the initial mahanorm} imply that, for any $(\ell,n) \in \Nset \times {\intn}$,
\begin{equation}
	\begin{array}{lll}
		\risk(\Maha{\estDiagmatn}^2 ( \CCtr^{(\ell+1)} - \oobs_n )) - \risk(\Maha{\estDiagmatn}^2 ( \CCtr^{(\ell)} - \oobs_n ))  \vspace{0.2cm} \\ 
		\qquad \leqslant w \left ( \Maha{\estDiagmatn}^2 ( \CCtr^{(\ell)} \! - \! \oobs_n ) \right )
		\times \big ( \Maha{\estDiagmatn}^2 ( \CCtr^{(\ell+1)} - \oobs_n ) - \Maha{\estDiagmatn}^2 ( \CCtr^{(\ell)} - \oobs_n ) \big ) \vspace{0.2cm} \\
		\qquad \leqslant 2 w \left ( \Maha{\estDiagmatn}^2 ( \CCtr^{(\ell)} \! - \! \oobs_n ) \right )  \big ( \meanvec^{(\ell)} \! - \! \oobs_n \big )^\transp \investDiagmatn \Diff^{(\ell)}
	\end{array}
	\nonumber
\end{equation}
with
\begin{subequations}
	\label{eq: def of meanvec and Diff}
	\begin{empheq}[left={\forall \ell \in \Nset,} \empheqlbrace]{align}
		& \meanvec^{(\ell)} = \frac{1}{2} \left ( \CCtr^{(\ell+1)} + \CCtr^{(\ell)}\right ) \mycomma
		\label{eq: def of meanvec}
		\\
		& \Diff^{(\ell)} = \CCtr^{(\ell+1)} - \CCtr^{(\ell)} \mycomma
		\label{eq: def of Diff}
	\end{empheq}
\end{subequations}
since
\begin{align}
\Maha{\estDiagmatn}^2 ( \CCtr^{(\ell+1)} & - \oobs_n ) - \Maha{\estDiagmatn}^2 ( \CCtr^{(\ell)} - \oobs_n ) 
\nonumber
\medskip \\
& = 
\Big ( \CCtr^{(\ell+1)} - \oobs_n + \CCtr^{(\ell)} - \oobs_n \Big )^\transpose \investDiagmatn \Big ( \CCtr^{(\ell+1)} - \oobs_n - \big ( \CCtr^{(\ell} - \oobs_n \big ) \Big )
\nonumber
\medskip \\
& = 2 \left ( \meanvec^{(\ell)} - \oobs_n \right ) ^\transpose \investDiagmatn \Diff^{(\ell)} \myfstop
\nonumber
\end{align}
We derive from the foregoing that, for any $(\ell,N) \in \Nset \times {\Nset}$,
\begin{onecol}
	\begin{equation}
		\label{eq: upper-bound on the diff of risks}
		\elrisk(\ell+1) - \elrisk(\ell) 
		\leqslant 2 \displaystyle \sum_{n=1}^N w \big ( \Maha{\estDiagmatn}^2 ( \CCtr^{(\ell)} - \oobs_n ) \big ) \big ( \meanvec^{(\ell)} - \oobs_n \big )^\transp \investDiagmatn \Diff^{(\ell)} \myfstop
	\end{equation}
\end{onecol}
According to \eqref{eq: MyNewfunc_app: general}, given $(\ell,N) \in \Nset \times \Nset$, we have:
\begin{onecol}
	\begin{equation}
	    \sum_{n=1}^N w \big (\Maha{\estDiagmatn}^2 ( \CCtr^{(\ell)} - \oobs_n ) \big ) \investDiagmatn \oobs_n = \sum_{n=1}^N w \big ( \,\Maha{\estDiagmatn}^2 ( \CCtr^{(\ell)} - \oobs_n ) \big ) \investDiagmatn \CCtr^{(\ell+1)} \myfstop
        \label{eq: equality 1}
	\end{equation}
\end{onecol}
For any given $(\ell,N) \in \Nset \times \Nset$, we derive from \eqref{eq: def of meanvec and Diff} and \eqref{eq: equality 1} that
\begin{onecol}
	$$\sum_{n=1}^N w \big (\Maha{\estDiagmatn}^2 ( \CCtr^{(\ell)} - \oobs_n ) \big ) \investDiagmatn \big ( \meanvec^{(\ell)} - \oobs_n \big ) = - \frac{1}{2} \sum_{n=1}^N w \big (\Maha{\estDiagmatn}^2 ( \CCtr^{(\ell)} - \oobs_n ) \big ) \investDiagmatn \Diff^{(\ell)}$$
\end{onecol}
and thus, according to \eqref{eq: definition of the initial mahanorm} for the definition of the Mahalanobis norm, that
\begin{align}
\displaystyle \sum_{n=1}^N w & \big ( \Maha{\estDiagmatn}^2 ( \CCtr^{(\ell)} - \oobs_n ) \big ) \big ( \meanvec^{(\ell)} - \oobs_n \big )^\transpose \investDiagmatn \Diff^{(\ell)}
\nonumber
\medskip \\
& =  
\displaystyle \sum_{n=1}^N w \big (\Maha{\estDiagmatn}^2 ( \CCtr^{(\ell)} - \oobs_n ) \big ) (\Diff^{(\ell)})^\transpose \investDiagmatn \big ( \meanvec^{(\ell)} - \oobs_n \big )
\nonumber
\medskip 
\\
& =  
- \dfrac{1}{2} \displaystyle \sum_{n=1}^N w \big (\Maha{\estDiagmatn}^2 ( \CCtr^{(\ell)} - \oobs_n ) \big ) \Maha{\estDiagmatn}^2(\Diff^{(\ell)})
\nonumber
\end{align}
By injecting the last equality into above into \eqref{eq: upper-bound on the diff of risks}, we obtain:
\begin{onecol}
	\begin{equation}
		\label{eq: bound on the difference between the elementary risks}
        \forall (\ell,N) \in \Nset \times \Nset,
		\elrisk(\ell + 1) - \elrisk(\ell) \leqslant - \displaystyle \sum_{n=1}^N w \big ( \Maha{\estDiagmatn}^2 ( \CCtr^{(\ell)} - \oobs_n ) \big ) \Maha{\estDiagmatn}^2(\Diff^{(\ell)}) \leqslant 0 \myfstop
	\end{equation}
\end{onecol}
Therefore, the sequence $\left ( \elrisk(\ell) \right )_{\ell \in \Nset}$ does not increase. Since $\elrisk(\ell) \geqslant 0$ for any $\ell \in \Nset$, this sequence has a limit. 

Given any $(\ell,L,N) \in \Nset \times  {\Nset} \times {\Nset}$, \eqref{eq: bound on the difference between the elementary risks} induces now that
\begin{onecol}
	\begin{equation}
		\label{eq: absolute value of the difference of elementary risks}
		\vert \elrisk(\ell+L) - \elrisk(\ell) \vert \geqslant \sum_{m=1}^{L} \left ( \, \sum_{n=1}^N w \big ( \Maha{\estDiagmatn}^2(\CCtr^{(\ell + m - 1)} - \oobs_n) \big ) \right ) \Maha{\estDiagmatn}^2 (\Diff^{(\ell+m-1)}) \myfstop
	\end{equation}
\end{onecol}
We now set
\begin{equation}
	\forall n \in \Nset, \estDiagmatn = \diag (\,\esteigvalnone, \ldots, \esteigvalndim \, ) \mycomma
	\label{eq: estdeltamat = appendix}
%	\nonumber
\end{equation}
where, for any given $n \in \Nset$, the positive real values $\esteigvalnone, \ldots, \esteigvalndim$ are the eigenvalues of $\estMCovn$ and thus verify:
\begin{equation}
	\forall n \in \Nset, \forall \dimi \in \intd, \sqrtesteigvalndimi \in [\sqrtesteigvalmin,\sqrtesteigvalmax] \myfstop
	\label{eq: estdelta in estinterval appendix}
\end{equation}
Hence, we    have:
\begin{onecol}
	\begin{equation}
		\label{eq: def of Mahalanobis norm for diag}
			\forall n \in \Nset, \forall \xvec = (x_1, \ldots, x_\mydim)^\transpose \in \Rset^\mydim, 
			\Maha{\estDiagmatn} (\xvec) = \sqrt{\textstyle \displaystyle \sum_{\dimi = 1}^\mydim \investeigvalndimi \, x^2_\dimi} \mycomma
	\end{equation}
\end{onecol}
from which we derive that
\begin{onecol}
	\begin{equation}
		\label{eq: inequality on operator norm-2}
		\forall {(\ell,n) \in \Nset \times \Nset}, 
		\Maha{\estDiagmatn}(\CCtr^{(\ell)}-\oobs_n) \leqslant \invsqrtesteigvalmin \times \norm[\CCtr^{(\ell)}-\oobs_n] \myfstop
		%	\nonumber
	\end{equation}
\end{onecol}
For any $\ell \in \Nset$, set $\CCtr^{(\ell)} = \big ( \cctr^{(\ell)}_1, \ldots, \cctr^{(\ell)}_\mydim \big )^\transpose$ and, for any $n \in \Nset$, $\oobs_n = (z_{n,1}, \ldots, z_{n,\mydim})^\transpose$. We obtain from \eqref{eq: MyNewfunc_app: general} and \eqref{eq: def of Mahalanobis norm for diag} that:
\begin{onecol}
\begin{equation}
	\label{eq: MyNewfunc_app}
	\forall (\ell,N) \in \Nset \times {\Nset}, \forall \dimi \in \intd, 
		\cctr_\dimi^{(\ell+1)} =
		\dfrac{\textstyle \sum_{n = 1}^N w \left ( \textstyle \sum_{\dimib=1}^\mydim \sqrtesteigvalndimib^{-2} \, \big (\oobs_{n} - \CCtr^{(\ell)} \big)_{\! \dimib}^{\! 2} \right ) \investeigvalndimi \, z_{n,i}}{\textstyle \sum_{n = 1}^N w \left ( \textstyle \sum_{\dimib=1}^\mydim \sqrtesteigvalndimib^{-2} \, \big ( \oobs_{n}-\CCtr^{(\ell)} \big)_{\! \dimib}^{\! 2} \right ) \investeigvalndimi} \myfstop
\end{equation}
\end{onecol}
Therefore, for any $\ell \in \intNrestrict$, any $\dimi \in \intd$, $\cctr_\dimi^{(\ell)}$ is a barycenter of the $N$ real values $z_{1,\dimi}, \ldots, z_{N,\dimi}$. Given any vector space $E$ endowed with a norm $\nu$, if $x^* = \textstyle \sum_{n=1}^N \alpha_n x_n$ is the barycenter of $N$ vectors $x_1, \ldots, x_N$ of $E$ respectively assigned to $N$ positive real coefficients $\alpha_1, \ldots, \alpha_N$ such that $\sum_{n=1}^N \alpha_n = 1$ then, for any $j \in \llbracket 1, N \rrbracket$,
\begin{equation}
	\label{eq: max for barycenter-0}
	\nu(x^*-x_j) = \nu \Big ( \sum_{r=1}^N \alpha_r (x_r- x_j) \Big ) \leqslant \sum_{r=1}^N \alpha_r \nu (x_r - x_j),
	\nonumber
\end{equation}
from which we derive that:
\begin{onecol}
	\begin{equation}
		\label{eq: max for barycenter}
		\nu(x^*-x_j) \leqslant \mymax \big \{ \nu(x_r-x_s): (r,s) \in \llbracket 1, N \rrbracket \times \llbracket 1, N \rrbracket \big \} \myfstop
	\end{equation}
\end{onecol}
Hence, from \eqref{eq: MyNewfunc_app} \& \eqref{eq: max for barycenter-0} with $E = \Rset$ and $\nu = \vert \cdot \vert$, we obtain that, for any $N \in {\Nset}$, any $n \in \intn$, any $\dimi \in \intd$ and any $\ell \in \intNrestrict$:
\begin{equation}
	\label{eq:barycenter application}
	\vert \cctr_\dimi^{(\ell)} - z_{n,\dimi} \vert \leqslant \max_{(r,s) \in \intn \times \intn } \{ \vert \, z_{r,\dimi} - z_{s,\dimi} \, \vert \} \myfstop
	\nonumber
\end{equation}
As a consequence, given  $\ell \in \intNrestrict$, $n \in \intn$ and $\dimi \in \intd$,
\begin{equation}
	\label{eq: upper bound on the square norm}
	\norm[\cctr^{(\ell)} - \oobs_n]^2 = \sum_{\dimi=1}^\mydim (\cctr_\dimi^{(\ell)} - z_{n,\dimi})^2 \leqslant \mydim \times \themax^2
\end{equation}
where $\themax = \max_{1 \leqslant \dimi \leqslant \mydim}  \left ( \max_{(r,s) \in \intn \times \intn } \{ \vert \, z_{r,\dimi} - z_{s,\dimi} \, \vert \} \right )$. 
Since $w$ does not increase, \eqref{eq: inequality on operator norm-2} \& \eqref{eq: upper bound on the square norm} imply that
\begin{onecol}
\begin{equation}
	\label{eq: lower bound on w for convergence}
		\forall {(\ell,n) \in \intNrestrict \times \Nset}, 
		w(\Maha{\estDiagmatn}^2 \! (\CCtr^{(\ell)} - \oobs_n) ) \geqslant w({\sqrtesteigvalmin^{-2}} \times \mydim \times \themax^2) \myfstop
	\nonumber
\end{equation} 
\end{onecol}
Thence the following lower bound 
$$\forall (\ell,N,m) \in \intNrestrict \times {\Nset} \times \Nset, \sum_{n=1}^N w(\Maha{\estDiagmatn}^2 (\CCtr^{(\ell + m)} - \oobs_n) ) \geqslant A$$ 
with $A = N \times w({\sqrtesteigvalmin^{-2}} \times \mydim \times \themax^2)$. Injecting the inequality above into \eqref{eq: absolute value of the difference of elementary risks} gives:
\begin{onecol}
\begin{equation}
	\label{eq: lower bound on the absolute value of the difference between the elementary risks}
		\forall (\ell,N,L) \in \intNfn{3} \times {\Nset} \times {\Nset}, 
		\vert \elrisk(\ell+L) - \elrisk(\ell) \vert \geqslant A \displaystyle \sum_{m=1}^{L} \Maha{\estDiagmatn}^2 \! (\Diff^{(\ell+m-1)}) \myfstop
\end{equation}
\end{onecol}
Fix $\varepsilon \in \Open{0}{\infty}$. Since the sequence $\left ( \elrisk(\ell) \right )_{\ell \in \Nset}$ has a limit, it is Cauchy. 
Therefore, there exists $\ell_0(\varepsilon) \in \intNfn{3}$ such that
\begin{onecol}
	\begin{equation}
		\label{eq: towards the stopping criterion}
		\forall {(\ell,L) \in \Nset \times \Nset}, \ell \geqslant \ell_0(\varepsilon) \Rightarrow \vert \elrisk(\ell+L) - \elrisk(\ell) \vert \leqslant A \varepsilon^2 / \eigvalmax \myfstop
	\end{equation}
\end{onecol}
According to \eqref{eq: estdeltamat = appendix} and \eqref{eq: def of Mahalanobis norm for diag}, we have 
\begin{equation}
    \label{eq: inequalities on mahalanobis norm}
    \forall \xvec \in \Rset^\mydim, \forall n \in \Nset, \Maha{\estDiagmatn}(\xvec) = \norm[\estDiagmatn^{-1/2} \xvec \,] \geqslant {\invsqrtesteigvalmax} \| \xvec \| \myfstop
\end{equation}
Thereby, it results from \eqref{eq: inequalities on mahalanobis norm}, \eqref{eq: lower bound on the absolute value of the difference between the elementary risks} and \eqref{eq: towards the stopping criterion} that, for any integer $\ell \geqslant \ell_0(\varepsilon)$ and any $L \in {\Nset}$, 
\begin{equation}
    \label{eq: final inequality for the stoppping criterion}
    \sum_{m=1}^{L} \norm[\Diff^{(\ell+m-1)}]^2 \leqslant \eigvalmax \sum_{m=1}^{L} \Maha{\estDiagmatn}^2 \big ( \Diff^{(\ell+m-1)} \big ) \leqslant \eigvalmax \, \vert \elrisk(\ell+L) - \elrisk(\ell) \vert \, / \, A \leqslant \varepsilon^2 \myfstop
\end{equation}
According to \eqref{eq: def of Diff}, 
\begin{equation}
    \label{eq: equality a}
    \forall \ell \in \Nset, \CCtr^{(\ell+L)} - \CCtr^{(\ell)} = \sum_{m=1}^{L} \Diff^{(\ell + m-1)} \myfstop
\end{equation} 
Straightforwardly, we also have:
\begin{equation}
    \label{eq: equality b}
    \forall L \in \Nset, \forall r \in \llbracket 1, L \rrbracket, \norm[\Diff^{(\ell + r -1)}] \leqslant \sqrt{\sum_{m=1}^{L} \norm[\Diff^{(\ell+m-1)}]^2} \myfstop
\end{equation} 
As a consequence, for any $L \in {\Nset}$ and any $\ell \geqslant \ell_0(\varepsilon)$, 
$$
\begin{array}{llllll}
	\norm[\Ctr^{(\ell+L)} - \Ctr^{(\ell)}]^2 
	& = \hspace{-0.2cm} & \hspace{-0.2cm} \norm[\CCtr^{(\ell+L)} - \CCtr^{(\ell)}]^2 \hspace{1cm} & \text{[$\Rmat$ is orthogonal]} \vspace{0.1cm} \\ 
	& \leqslant & \norm[\displaystyle \sum_{m=1}^{L} \Diff^{(\ell+m-1)} ]^2 \hspace{1cm} & \text{[By \eqref{eq: equality a}]}
	\vspace{0.1cm} \\
	& \hspace{-0.2cm} \leqslant \hspace{-0.2cm} & \hspace{-0.2cm} \Big ( \displaystyle\sum_{m=1}^{L}  \norm[ \Diff^{(\ell+m-1)} ] \Big )^2 
	\vspace{0.1cm} \\
	& \hspace{-0.2cm} \leqslant \hspace{-0.2cm} & \hspace{-0.2cm} L^2 \displaystyle \sum_{m=1}^{L} \norm[ \Diff^{(\ell+m-1)} ]^2 & \text{[By \eqref{eq: equality b}]}
	\vspace{0.2cm} \\
	& \hspace{-0.2cm} \leqslant \hspace{-0.2cm} & \hspace{-0.2cm} L^2 \varepsilon^2 & [\text{By \eqref{eq: final inequality for the stoppping criterion}}] \myfstop
\end{array}
$$

\section{p-value of Wald's test for Gaussian mean testing}
\label{Sec: RDT pvalue}

Suppose that $\Xvec \thicksim \Ncal(\xivec,\Smat)$, with $\xivec \in \Rset^{\dimc}$ and $\Smat$ is a $\dimc \times \dimc$ \spdm. The critical region of the Wald test $\Topt$ given by \eqref{Eq:Thresholding test from above} is defined as
\begin{onecol}
	\begin{equation}
		\label{eq: def of Salpha}
		\Scal_\level = \mybig \{ \xvec \in \Rset^{\dimc}: \Topt(\xvec) = 1 \mybig \} = \mybig \{ \xvec \in \Rset^{\dimc}:  \mu(\level) < \thenorm( \xvec ) \mybig \} \mycomma
	\end{equation}
\end{onecol}
where $\mu(\level)$ is defined by \eqref{Eq: Wald threshold}. It straightforwardly follows from \eqref{Eq: Wald threshold} and the properties of $\Fchi$ that the function $\level \in \intervalzeroone \mapsto \mu(\level) \in \Open{0}{\infty}$ is strictly decreasing. Therefore, for two levels $0 < \level < \level' < 1$, we have $\Scal_{\level} \subset \Scal_{\level'}$. Consider the function $\TestFamily_{\Smat}$ that assigns to each given $\alpha \in \Open{0}{1}$ the Wald test $\Topt$. According to \citet[p.~63, Sec.~3.3]{Lehmann2005} and \eqref{eq: def of Salpha}, the p-value of $\TestFamily_{\Smat}$ at some $\xvec \in \Rset^{\mydim}$ is given by:
\begin{equation}
	\label{eq: def of pval}
	\pval{\Smat}{\xvec} 
	= \inf \mybig \{ \level \in \Open{0}{1}: \xvec \in \Scal_\level \mybig \} 
	= \inf \mybig \{ \level \in \Open{0}{1}: \mu(\level) < \thenorm( \xvec ) \mybig \} \myfstop
\end{equation}
Set $\level_0(\xvec) = 1 - \Fchi(\nu_{\Smat}^2(\xvec))$. \eqref{Eq: Wald threshold} implies that $\level_0 (\xvec) = 1 - \Fchi(\mu(\level_0 (\xvec))^2)$. Thereby, we have $\mu(\level_0 (\xvec)) = \thenorm( \xvec )$. Since $\mu$ is decreasing and bijective, the foregoing implies that 
$$\mybig \{ \level \in \Open{0}{1}: \mu(\level) < \thenorm( \xvec ) \mybig \} = \mybig \{ \level \in \Open{0}{1}: \mu(\level) < \mu(\level_0 (\xvec)) \mybig \} = \Open{\level_0(\xvec)}{1}$$
We thus conclude from \eqref{eq: def of pval} that $\pval{\Smat}{\xvec} = \alpha_0(\xvec) = 1 - \Fbb_{\chi^2_{\mydim}}(\nu_{\Smat}^2( \xvec ))$. 

%%%%%%%%%%%%%%%%%%%%%%%%%%%%%%%%%%%%%%%%%%%%%%%%%%%%%%
%%%%%%%%%%%%%%%%%%%%%%%%%%%%%%%%%%%%%%%%%%%%%%%%%%%%%%
%%%%%%%%%%%%%%%%%%%%%%%%%%%%%%%%%%%%%%%%%%%%%%%%%%%%%%
%%%%%%%%%%%%%%%%%%%%%%%%%%%%%%%%%%%%%%%%%%%%%%%%%%%%%%

\bibliography{Refs}

\end{document}